\documentclass[11pt,letterpaper]{article} 

\usepackage[OT1]{fontenc} 

\usepackage{fullpage,times}
\usepackage{amsthm,amsmath,amsfonts} 
\usepackage{epsf,epsfig,caption,subcaption,graphicx} 
\usepackage[ruled,boxed]{algorithm2e}
\usepackage{tikz}

\newtheorem{theorem}{Theorem}[section]
\newtheorem{lemma}[theorem]{Lemma} 
\newtheorem{proposition}[theorem]{Proposition} 
\newtheorem{corollary}[theorem]{Corollary}
\theoremstyle{definition}

\newtheorem{assumption}[theorem]{Assumption}

\theoremstyle{remark}

\def\var{\mbox{Var}} 

\def\E{\mathbb{E}} 

\def\exp{\mbox{exp}}
\def\pa{\mbox{pa}}

\def\nd{\mbox{nd}}

\def\S{S}
\def\T{S}

\def\M{M}
\def\d{d}

\DeclareMathAlphabet\mathbfcal{OMS}{cmsy}{b}{n}

\begin{document}

\begin{center}
	{\bf{\LARGE{Learning Quadratic Variance Function (QVF) DAG models via OverDispersion Scoring (ODS)}}}
	
	\vspace*{.1in}
	\begin{tabular}{cc}
		Gunwoong Park$^1$\;\;Garvesh Raskutti$^{2,3,4}$\\
	\end{tabular}
	
	\vspace*{.1in}
	
	\begin{tabular}{c}
		$^1$ Department of Statistics, University of Michigan-Ann Arbor \\
		$^2$ Department of Statistics, University of Wisconsin-Madison \\
		$^3$ Department of Computer Science\\
		$^4$ Wisconsin Institute for Discovery, Optimization Group \\
	\end{tabular}
	
	\vspace*{.1in}
	
	
\end{center}
\begin{abstract}
Learning DAG or Bayesian network models is an important problem in multi-variate causal inference. However, a number of challenges arises in learning large-scale DAG models including model identifiability and computational complexity since the space of directed graphs is huge. In this paper, we address these issues in a number of steps for a broad class of DAG models where the noise or variance is signal-dependent. Firstly we introduce a new class of identifiable DAG models, where each node has a distribution where the variance is a quadratic function of the mean (QVF DAG models). Our QVF DAG models include many interesting classes of distributions such as Poisson, Binomial, Geometric, Exponential, Gamma and many other distributions in which the noise variance depends on the mean. We prove that this class of QVF DAG models is identifiable, and introduce a new algorithm, the OverDispersion Scoring (ODS) algorithm, for learning large-scale QVF DAG models. Our algorithm is based on firstly learning the moralized or undirected graphical model representation of the DAG to reduce the DAG search-space, and then exploiting the quadratic variance property to learn the causal ordering. We show through theoretical results and simulations that our algorithm is statistically consistent in the high-dimensional $p>n$ setting provided that the degree of the moralized graph is bounded and performs well compared to state-of-the-art DAG-learning algorithms.
\end{abstract}

\section{Introduction}

Probabilistic directed acyclic graphical (DAG) models or Bayesian networks provide a widely used framework for representing causal or directional dependence relationships amongst multiple variables. DAG models have applications in various areas including genomics, neuroimaging, statistical physics, spatial statistics and many others (see e.g.,~\cite{doya2007bayesian, friedman2000using, kephart1991directed}). One of the fundamental problems associated with DAG models or Bayesian networks is structure learning from observational data.

If the number of variables is large, a number of challenges arise that make learning large-scale DAG models extremely difficult even when variables have a natural causal or directional structure. These challenges include: (1) identifiability since inferring causal directions from only observational data is in general not possible in the absence of additional assumptions; (2) computational complexity since it is NP-hard to search over the space of DAGs~\cite{Chickering1996}; (3) providing sample size guarantee in the setting where the number of nodes $p$ is large. In this paper we develop a general framework and algorithm for learning large-scale DAG models that addresses these challenges in a number of steps: Firstly, we introduce a new class of provably identifiable DAG models where each node has a conditional distribution where the variance is a quadratic function of the mean, which we refer to as QVF (quadratic variance function) distributions; secondly, we introduce a general OverDispersion Scoring (ODS) algorithm for learning large-scale QVF DAG models; thirdly, we provide theoretical guarantees for our ODS algorithm which proves that our algorithm is consistent in the high-dimensional setting $p > n$ provided that the moralized graph of the DAG is sparse; and finally, we show through a simulation study that our ODS algorithm supports our theoretical result has favorable performance to a number of state-of-the-art algorithms for learning both low-dimensional and high-dimensional DAG models. 

Our algorithm is based on combining two ideas: \emph{overdisperson} and \emph{moralization}. Overdispersion is a property of Poisson and other random variables where the variance depends on the mean and we use overdispersion to address the identifiability issue. While overdispersion is a known phenomena used and exploited in many applications (see e.g.,~\cite{Dean1992, Zheng2006}), overdispersion has never been exploited for learning DAG models aside from our prior work~\cite{park2015learning} which focuses on Poisson DAG models. In this paper, we show that overdispersion applies much more broadly and is used to prove identifiability for a broad class of DAG models. To provide a scalable algorithm with statistical guarantees, even in the high-dimensional setting, we exploit the moralized graph, that is the undirected representation of the DAG. Learning the moralized graph allows us to exploit sparsity and considerably reduces the DAG search-space which has both computational and statistical benefits. Furthermore, moralization allows us to use existing scalable algorithms and theoretical guarantees for learning large-scale undirected graphical models (e.g.,~\cite{Friedman2009,Yang2012}). 

A number of approaches have been used to address the identifiabilty challenge by imposing additional assumptions. For example ICA-based methods for learning causal ordering requires independent noise and non-Gaussianity (see e.g.,~\cite{Shimizu2006}), structural equation models with Gaussian noise with equal or known variances~\cite{Peters2012}, and non-parametric structural equation models with independent noise (see e.g.,~\cite{Peters2013}). These approaches are summarized elegantly in an information-theoretic framework in~\cite{JanzingScholkopf}. Our approach is along similar lines in that we impose overdispersion as an additional assumption which induces asymmetry and guarantees identifiability. However by exploiting overdispersion, our approach applies when the noise distribution of each node depends on its mean whereas prior approaches apply when the additive noise variance is independent of the mean. Additionally, we exploit graph sparsity which has also been exploited in prior work by~\cite{loh2014high,Raskutti2013,geerpb12} for various DAG models with independent additive noise components. Furthermore, sparsity allows us to develop a tractable algorithm where we reduce the DAG space by learning the moralized graph, an idea which has been used in prior work in~\cite{Tsamardinos2003}.

The remainder of the paper is organized as follows: In Section~\ref{SecClass}, we define QVF DAG models and prove identifiability for this class of models. In Section~\ref{SecAlg}, we introduce our polynomial-time DAG learning algorithm which we refer to as the generalized OverDispersion Scoring (ODS). Statistical guarantees for learning QVF DAG models using our ODS algorithm are provided in Section~\ref{SecStat}, and we provide numerical experiments on both small DAGs and large-scale DAGs with node-size up to $5000$ nodes in Section~\ref{SecNum}. Our theoretical guarantees in Section~\ref{SecStat} prove that even in the setting where the number of nodes $p$ is larger than the sample size $n$, it is possible to learn the DAG structure under the assumption that the degree $d$ of the so-called moralized graph of the DAG is small. Our numerical experiments in Section~\ref{SecNum} support the theoretical results and show that our algorithm performs well compared to other state-of-the-art DAG learning methods. Our numerical experiments confirm that our algorithm is one of the few DAG-learning algorithms that performs well in terms of statistical and computational complexity in high-dimensional $p > n$ settings, provided that the degree of the moralized graph $d$ is bounded. 

\section{Quadratic Variance Function (QVF) DAG models and Identifiability}

\label{SecClass}

A DAG $G = (V, E)$ consists of a set of nodes $V$ and a set of directed edges $E \in V \times V$ with no directed cycle. We set $V = \{1, 2, \cdots, p\}$ and associate a random vector $X := (X_1, X_2, \cdots, X_p)$ with probability distribution $\mathbb{P}$ over the nodes in $G$. A directed edge from node $j$ to $k$ is denoted by $(j,k)$ or $j \rightarrow k$. The set of \emph{parents} of node $k$ denoted by $\pa(k)$ consists of all nodes $j$ such that $(j,k) \in E$. If there is a directed path $j\to \cdots \to k$, then $k$ is called a \emph{descendant} of $j$ and $j$ is an \emph{ancestor} of $k$. The set $\mbox{de}(k)$ denotes the set of all descendants of node $k$. The \emph{non-descendants} of node $k$ are $\nd(k) := V \setminus (\{k\} \cup \mbox{de}(k))$. An important property of DAGs is that there exists a (possibly non-unique) \emph{causal ordering} $\pi^*$ of a directed graph that represents directions of edges such that for every directed edge $(j, k) \in E$, $j$ comes before $k$ in the causal ordering. Without loss of generality, we assume the true causal ordering is $\pi^* = (1, 2 ,\cdots, p)$ for $G$.

Suppose that $X$ is a $p$-variate random vector with joint probability density $f_{G}(X)$. Then, a probabilistic DAG model has the following factorization~\cite{Lauritzen1996}:
\begin{equation}
\label{EqnFactorization}
	f_{G}(X) = \prod_{j=1}^{p} f_j(X_j \mid X_{\pa(j)}), 
\end{equation}
where $f_j(X_j \mid X_{\pa(j)})$ refers to the conditional distribution of a random variable $X_j$ in terms of its parents $X_{\pa(j)} :=\{X_s : s \in \pa(j) \}$. 

A core concept in this paper is identifiability for a family of probability distributions defined by the DAG factorization provided above. Let $\mathcal{G}_p$ denote the set of $p$-node DAGs and let $ \mathcal{F}_p(\mathcal{P}) := \{f_G\;:\; f_j \in \mathcal{P}\; \;G \in \mathcal{G}_p \}$ be a family of $p$-variate distributions where each $f_G$ factorizes according to $G$ through \eqref{EqnFactorization} and each conditional distribution $f_j$ lies in a family of distribution $\mathcal{P}$. A family of distributions $\mathcal{F}_p$ is \emph{identifiable} if there exist functions $F_p : \mathcal{F}_p \rightarrow \mathcal{G}_p$ where $F_p(f_G) = G$ for all $f_G \in \mathcal{F}_p$ for all $p \geq 2$.

In our setting $\mathcal{P}$ is a setting where the variance is a linear function of the mean so we deal with signal-dependent noise or variance. Prior work has considered classes of distribution $\mathcal{P}$. For example ICA-based methods make the assumption that $\mathcal{P}$ is independent error with non-Gaussian components~\cite{Shimizu2006}, non-parametric regression assumes $\mathcal{P}$ is a non-parametric model with additive independent noise~\cite{Peters2013}, and in~\cite{Peters2012}, $\mathcal{P}$ represents linear Gaussian relationships with equal or known variances. On the other hand, general Gaussian DAG models do not belong to QVF DAG models because means and covariance function for Gaussian distributions are unrelated. Hence Gaussian DAG models can only be learnt up to Markov equivalence~\cite{Heckerman1995}. We define $\mathcal{P}$ more precisely in the next section.

\subsection{Quadratic Variance Function (QVF) DAG models}

Firstly, we define quadratic variance function (QVF) DAG models. For QVF DAG models each node has a conditional distribution $\mathcal{P}$ given its parents with the property that the variance is a quadratic function of the mean. More precisely, there exist constants $\beta_{j0}, \beta_{j1} \in \mathbb{R}$ for all $j \in V$ such that:
\begin{equation}
	\label{eq:Quad}
	\var(X_j \mid X_{\pa(j)})	= \beta_{j0} \E(X_j \mid X_{\pa(j)}) + \beta_{j1} \E(X_j \mid X_{\pa(j)})^2.
\end{equation}  
 
To the best of our knowledge, quadratic variance function (QVF) probability distributions were first introduced in the context of natural parameter exponential families (NEF)~\cite{morris1982natural} which include Poisson, Binomial, Negative Binomial and Gamma distributions.

For natural exponential families with quadratic variance functions (NEF-QVF), the conditional distribution of each node given its parents takes the simple form:
\begin{equation*}
	\label{eq:NEF-QVF}
	P(X_j \mid X_{\pa(j)}) = \exp \left( \theta_{jj} X_j + \sum_{(k,j)\in E} \theta_{jk} X_k X_j - B_j( X_j ) - A_j \left(\theta_{jj} + \sum_{(k,j)\in E} \theta_{jk} X_k \right) \right)
\end{equation*}
where $A_j(\cdot)$ is the log-partition function, $B_j(\cdot)$ is determined by a chosen exponential family, and $\theta_{jk} \in \mathbb{R}$ is a parameter corresponding to a node $j$. By the factorization property~\eqref{EqnFactorization}, the joint distribution of a NEF-QVF DAG model takes the following form:
\begin{equation}
\label{DAGGLM}
P(X) = \exp \left( \sum_{j \in V} \theta_{jj} X_j + \sum_{(k,j)\in E} \theta_{jk} X_k X_j - \sum_{j \in V} B_j( X_j ) - \sum_{j \in V} A_j \left(\theta_{jj} + \sum_{(k,j)\in E} \theta_{jk} X_k \right) \right). 
\end{equation}

From Equation~\eqref{DAGGLM}, we provide examples of classes of NEF-QVF DAG models. For Poisson DAG models studied in~\cite{park2015learning} the log-partition function $A_j(\cdot) = \exp(\cdot)$, and $B_j(\cdot) = \log(\cdot!)$.  Similarly, Binomial DAG models can be derived as an example of QVF DAG models where the conditional distribution for each node is binomial with known parameter $N_j$ and the log-partition function $A_j(\cdot) = N_j \log(1+\exp(\cdot))$, and $B_j(\cdot) = -\log\binom{N_j}{\cdot}$. Another interesting instance is Exponential DAG models where each node conditional distribution given its parents is Exponential. Then, $A_j(\cdot) = -\log(-\cdot)$ and $B_j(\cdot) = 0$. Our framework also naturally extends to mixed DAG models, where the conditional distributions have different distributions which incorporates different data types. In Section~\ref{SecNum}, we will provide numerical experiments on Poisson and Binomial DAG models.

\subsection{Identifiability of QVF DAG models}

In this section we prove that QVF DAG models are identifiable. To provide intuition, we prove identifiability for the two-node Poisson DAG model. Consider all three models illustrated in Figure~\ref{figure1}: $\mathcal{M}_1: X_1 \sim \mbox{Poisson}(\lambda_1),\;\; X_2 \sim \mbox{Poisson}(\lambda_2)$, where $X_1$ and $X_2$ are independent;
$\mathcal{M}_2: X_1 \sim \mbox{Poisson}(\lambda_1)$ and $X_2 \mid X_1 \sim \mbox{Poisson}(g_2(X_1))$;
and $\mathcal{M}_3: X_2 \sim \mbox{Poisson}(\lambda_2)$ and $X_1 \mid X_2 \sim \mbox{Poisson}(g_1(X_2))$ for arbitrary positive functions $g_1, g_2: \mathbb{R} \to \mathbb{R}^{+}$. Our goal is to determine whether the underlying DAG model is $\mathcal{M}_1, \mathcal{M}_2$ or $\mathcal{M}_3$.

\begin{figure}[!t]
	\centering
	\begin {tikzpicture}[ -latex ,auto,
	state/.style={circle, draw=black, fill= white, thick, minimum size= 2mm},
	label/.style={thick, minimum size= 2mm}
	]
	\node[state] (X1)  at (0,0)   {\small{$X_1$} }; \node[state] (X2)  at (2,0)   {\small{$X_2$}}; \node[label] (X3) at (1,-.7) {$\mathcal{M}_1$};
	\node[state] (Y1)  at (5,0)   {\small{$X_1$}}; \node[state] (Y2)  at (7,0)   {\small{$X_2$}}; \node[label] (Y3) at (6,-.7) {$\mathcal{M}_2$};
	\node[state] (Z1)  at (10,0)   {\small{$X_1$}}; \node[state] (Z2)  at (12,0)   {\small{$X_2$}}; \node[label] (Z3) at (11,-.7) {$\mathcal{M}_3$};
	\path (Y1) edge [shorten <= 2pt, shorten >= 2pt] node[above]  { } (Y2); 
	\path (Z2) edge [shorten <= 2pt, shorten >= 2pt] node[above]  { } (Z1);
\end{tikzpicture}
\caption{Directed graphical models of $\mathcal{M}_1$, $\mathcal{M}_2$ and $\mathcal{M}_3$  }
\label{figure1}
\end{figure}
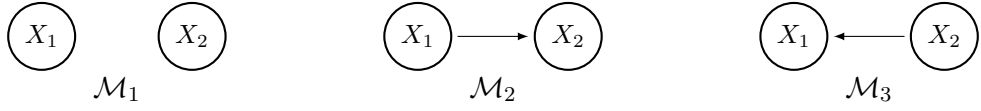

We exploit the fact that for a Poisson random variable $X$, $\var(X) = \E(X)$, while for a distribution which is conditionally Poisson, the marginal variance is overdispersed relative to the marginal expectation, $\var(X) > \E(X)$. Hence for $\mathcal{M}_1$, $\var(X_1) = \E(X_1)$ and $\var(X_2) = \E(X_2)$. For $\mathcal{M}_2$, $\var(X_1) = \E(X_1)$, while
\begin{equation*}
\var(X_2) = \E(\var(X_2 \mid X_1)) + \var(\E(X_2 \mid X_1)) = \E(\E(X_2 \mid X_1)) + \var(g_2(X_1)) > \E(X_2),
\end{equation*}
as long as $\var(g_2(X_1)) > 0$. 

Similarly under $\mathcal{M}_3$, $\mbox{Var}(X_2) = \mathbb{E}(X_2)$ and $\mbox{Var}(X_1) > \mathbb{E}(X_1)$ as long as $\mbox{Var}(g_1(X_2)) > 0$. Hence we can distinguish models $\mathcal{M}_1$, $\mathcal{M}_2$, and $\mathcal{M}_3$ by testing whether the variance is greater than or equal to the expectation. With finite samples, the quantities $\E(\cdot)$ and $\var(\cdot)$ can be estimated from data and we describe this more precisely in Sections~\ref{SecAlg} and~\ref{SecStat}. 

For general QVF DAG models, the variance for each node distribution is not necessarily equal to the mean. Hence we introduce a linear transformation $T_j(X_j) = \omega_j X_j$ such that $\var(T_j(X_j) \mid X_{\pa(j)}) = \E(T_j(X_j) \mid X_{\pa(j)})$ in Proposition~\ref{prop:a}. This transformation enables us to use the notion of \emph{overdispersion} for recovering general QVF DAG models. We present examples of distributions for QVF DAG models with the triple $(\beta_{0}, \beta_{1}, \omega)$ in the following Table~\ref{Example}.

\begin{proposition}
	\label{prop:a}
	Let $X =(X_1, X_2, \cdots, X_p)$ be a random vector associated with a QVF DAG model with quadratic variance coefficients $(\beta_{j0}, \beta_{j1})_{j = 1}^{p}$ specified in~\eqref{eq:Quad}. Then, there exists a transformation $T_j(X_j) = \omega_j X_j$ for any node $j\in V$ where $\omega_j = ( \beta_{j0} + \beta_{j1} \E(X_j \mid X_{\pa(j)} ) )^{-1}$ such that 
	$$\var(T_j(X_j) \mid X_{\pa(j)}) = \E(T_j(X_j) \mid X_{\pa(j)}).$$
\end{proposition} 

\begin{proof}
	For any node $j \in V$,
	\begin{eqnarray*}
		\var(\omega_j X_j \mid X_{\pa(j)}) & = & \omega_j^2 \var( X_j \mid X_{\pa(j)}) \\
		& \stackrel{(a)}{=} & \omega_j^2 ( \beta_{j0} \E(X_j \mid X_{\pa(j)}) + \beta_{j1} \E(X_j \mid X_{\pa(j)})^2 ) \\
		&\stackrel{(b)}{=} & \omega_j \E(X_j \mid X_{\pa(j)}) \\
		& = & \E( \omega_j X_j \mid X_{\pa(j)}).
	\end{eqnarray*} 
	(a) follows from the quadratic variance property~\eqref{eq:Quad}, and (b) follows from the definition of $\omega_j$.
\end{proof}

\begin{table}
	\centering
	\begin{tabular}{|l  l |c c c|} \hline
		Distribution & & $\beta_0$ & $\beta_1$ 	 	 & $\omega$ \\ \hline
		Binomial & Bin$(N,p)$	 & 1		 & $-\frac{1}{N}$	 & $\frac{N}{N - \E(X)}$ \\
		Poisson & Poi$(\lambda)$ 	 & 1		 & 0				 & 1 \\
		Generalized Poisson & GPoi($\lambda_1, \lambda_2)$ & $\frac{1}{(1-\lambda_2)^2}$ & 0 & $\frac{1}{(1-\lambda_2)^2}$ \\
		Geometric & Geo($p$) 	 & 1 & 1	 & $\frac{1}{1 + \E(X)}$ \\
		Negative Binomial & NB($R,p$)	 & 1 & $\frac{1}{R}$ 	 & $\frac{R}{R + \E(X)}$ \\
		Exponential & Exp($\lambda$) & 0 		 & 1				 & $\frac{1}{\E(X) }$ \\
		Gamma & Gamma($\alpha, \beta$)		 & 0 		 & $\frac{1}{\alpha}$& $\frac{\alpha}{\E(X)}$ \\ \hline
	\end{tabular}
	\caption{Examples of distributions for QVF DAG models with $\beta_0, \beta_1$ and $\omega$ where $\E(X)$ is its expectation}
	\label{Example}
\end{table}

Now we extend to general $p$-variate QVF DAG models. The key idea to extending identifiability from the bivariate to multivariate scenario involves conditioning on parents of each node, and then testing overdispersion. 

\begin{theorem}[Identifiability for p-variate QVF DAG models]
	\label{Thmidentifiability}
	Consider the class of QVF DAG models~\eqref{EqnFactorization} with quadratic variance coefficients $(\beta_{j0}, \beta_{j1})_{j = 1}^{p}$~\eqref{eq:Quad}. Suppose that $\beta_{j1} > -1$ for all $j \in V$. Furthermore, for all $j \in V$, $K_j \subset \pa(j)$, $K_j \neq \emptyset$, and $S \subset \nd(j) \setminus K_j$ where $\beta_{j0} + \beta_{j1} \E(X_j \mid X_{S} ) \neq 0$ and
	\begin{equation}
		\var( \E(X_j \mid X_{pa(j)}) \mid X_S) > 0.
	\end{equation}
	The the class of QVF DAG models is identifiable.
\end{theorem}

The proof is provided in Appendix~\ref{SecSubThmIde}. Theorem~\ref{Thmidentifiability} shows that any QVF DAG model is fully identifiable under the assumption that all parents of node $j$ contribute to its variability. The condition $\beta_{j1} > -1$ rules out DAG models with Bernoulli and multinomial distributions which are known to be non-identifiable~\cite{Heckerman1995} with $\beta_{j1} = -1$. 

\section{OverDispersion Scoring (ODS) Algorithm}

\label{SecAlg}

In this section, we present our generalized OverDispersion Scoring (ODS) algorithm. An important concept we need to introduce for the generalized ODS algorithm is the \emph{moral} graph or undirected graphical model representation of a DAG (see e.g.,~\cite{Cowell1999}). The moralized graph $G^m$ for a DAG $G= (V, E)$ is an undirected graph where $G^m = (V, E^m)$ where $E^m$ includes the edge set $E$ for the DAG $G$ with directions removed plus edges between any nodes that are parents of a common child. Figure~\ref{figure2} represents the moralized graph for a simple $3$-node example where $E = \{(1,3), (2,3)\}$ for the DAG $G$. Since nodes $1$ and $2$ are parents with a common child $3$, the additional edge $(1,2)$ arises, and therefore $E^m = \{(1,2), (1,3), (2,3)\}$. Finally, the \emph{neighborhood} for a node $j$ refers to the adjacent nodes to $j$ in the moralized graph, and is denoted by $\mathcal{N}(j) := \{k \in V \mid (j,k) \text{ or } (k,j) \in E^m \}$.

\begin{figure}[t] 
	\centering
	\begin{tikzpicture}[ -latex , auto,
	state/.style={circle, draw=black, fill= white, thick, minimum size= 5mm},
	label/.style={thick, minimum size= 5mm}
	]
	\node[state] (X1)  at (0,0)   {$1$};    \node[state] (X2)  at (3,0)   {$2$};  
	\node[state] (X3)  at (1.5,0)   {$3$}; 
	\node[label] (G1) at (1.5,-0.8) {$G$};
	\node[state] (Y1)  at (5,0)   {$1$}; \node[state] (Y2)  at (8,0)   {$2$};  
	\node[state] (Y3)  at (6.5,0)   {$3$};   
	\node[label] (G2) at (6.5,-0.8) {$G^m$};
	
	\path (X1) edge [shorten <=2pt, shorten >=2pt] node[above]  { } (X3); 
	\path (X2) edge [shorten <=2pt, shorten >=2pt] node[above]  { } (X3); 
	\path (Y1) edge [-,shorten <=2pt, shorten >=2pt] node[above] { } (Y3); 
	\path (Y2) edge [-,shorten <=2pt, shorten >=2pt] node[above] { } (Y3); 
	\path (Y1) edge [-,shorten <=1pt, shorten >=1pt,bend left = 45] node[above] { } (Y2); 
	\end{tikzpicture}
	\caption{Moralized graph $G^m$ for DAG $G$} 
	\label{figure2}
\end{figure}
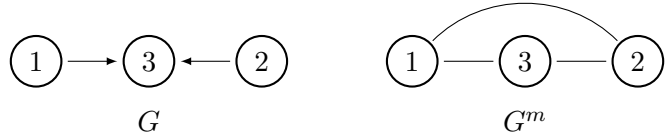

Our generalized ODS algorithm has three main steps: 1) estimate the moralized graph $G^m$ for the DAG $G$; 2) estimate the causal ordering of the DAG $G$ using overdispersion scoring based on the moralized graph from step 1); and 3) estimating the DAG structure, given the causal ordering from step 2). There are many choices of algorithms for Steps 1) and 3), for example standard neighborhood selection procedures in which we use not only regression algorithms, but also off-the-shelf graph learning algorithms (e.g.,~\cite{Yang2012, Tsamardinos2006, Aliferis2003}). Although Steps 2) and 3) are sufficient to recover DAG structures, Step 1) is performed because it reduces both computational and sample complexity by exploiting the sparsity of the moralized graph for the DAG.

\setlength{\algomargin}{0.5em}
\begin{algorithm}[!t]
	\caption{ \bf Generalized OverDispersion Scoring (ODS) \label{AlgODS} }
	\SetKwInOut{Input}{Input}
	\SetKwInOut{Output}{Output}
	\SetKwInOut{Return}{Return}
	\Input{$n$ i.i.d. samples from a QVF-DAG model}
	\Output{Estimated causal ordering $\widehat{\pi} \in \mathbb{N}^{p}$ and an edge structure, $\widehat{E} \in V \times V$ }
	\BlankLine
	Step 1: Estimate the undirected edges $\widehat{E}^m = \cup_{j \in V} \cup_{k \in \widehat{\mathcal{N}}(j)} (j,k)$ where $\widehat{\mathcal{N}}(j)$ is estimated neighborhood set of a node $j$ in the moralized graph\;
	Step 2: Estimate the causal ordering using overdispersion scores\; 
	\For{$k \in \{1,2,\cdots,p\}$}
	{Calculate overdispersion scores $\widehat{\mathcal{S}}(1,k)$ using ~\eqref{EqnTruncScore1}\; }
	The first element of the causal ordering $\widehat{\pi}_1 = \arg \min_k \widehat{\mathcal{S}}(1,k) $\;
	\For{$j = \{2,3,\cdots,p-1\}$}{
		\For{$k \in \widehat{\mathcal{N}}(\widehat{\pi}_{j-1}) \cap \{1,2,\cdots,p\} \setminus \{\widehat{\pi}_1,\cdots,\widehat{\pi}_{j-1}\}$ }{
			Find candidate parents set $\widehat{C}_{jk} = \widehat{\mathcal{N}}(k) \cap \{\widehat{\pi}_1, \widehat{\pi}_2,\cdots,\widehat{\pi}_{j-1}\}$\;
			Calculate overdispersion scores $\widehat{\mathcal{S}}(j,k)$ using ~\eqref{EqnTruncScorej};
		}
		The $j^{th}$ element of a causal ordering $\widehat{\pi}_j = \arg \min_k \widehat{\mathcal{S}}(j,k)$\;
		Step 3: Estimate the directed edges toward $\widehat{\pi}_j$, denoted by $\widehat{D}_j$\;
	}
	The $p^{th}$ element of the causal ordering $\widehat{\pi}_p = \{1,2,\cdots,p\} \setminus \{\widehat{\pi}_1, \widehat{\pi}_2,\cdots, \widehat{\pi}_{p-1}\}$\;
	The directed edges toward $\widehat{\pi}_p$, denoted by $\widehat{D}_p = \{ (z,\widehat{\pi}_p) \;|\; z \in \widehat{\mathcal{N}}(\widehat{\pi}_p) \} $\;
	\BlankLine
	\Return{$\widehat{\pi} = (\widehat{\pi}_1,\widehat{\pi}_2,\cdots,\widehat{\pi}_p)$, and $\widehat{E} = \cup_{j = \{2,3,\cdots,p\}} \widehat{D}_j$}
\end{algorithm} 

The main purpose of Step 1) is to reduce the search-space by exploiting sparsity of the moralized graph. The moralized graph provides a \emph{candidate parents} set for each node. Similar ideas of reducing search space by utilizing the moralized graph or different undirected graphs are applied in existing algorithms (e.g.,~\cite{Tsamardinos2003, Friedman1999, loh2014high}). The concept of candidate parents set exploits two properties; (i) the neighborhood of a node $j$ is a superset of its parents, and (ii) a node should appear later than its parents in the causal ordering. Hence, the candidate parents set for a given node $j$ is the intersection of its neighborhood and elements of the causal ordering which appear before that node $j$. This candidate parents set is used as a conditioning set for the overdispersion score in Step 2). In principle, the size of the conditioning set for an overdispersion score could be $p-1$ if the moralized graph is not used. Since Step 2) requires computation of a conditional mean and variance, both the computational complexity and sample complexity depend significantly on the number of variables we condition on as illustrated in Sections~\ref{SecCom} and~\ref{SecStat}. Therefore by making the conditioning set for the overdispersion score of each node as small as possible, we gain significant computational and statistical improvements. 

A number of choices are available for estimation of the moralized graph. Since the moralized graph is an undirected graph, standard undirected graph learning algorithms such as HITON~\cite{Aliferis2003} and MMPC algorithms~\cite{Tsamardinos2003} as well as  $\ell_1$-penalized likelihood regression for generalized linear models (GLM)~\cite{Friedman2009}. In addition, standard DAG learning algorithms such as PC~\cite{spirtes2000causation}, GES~\cite{Chickering2003} and MMHC algorithms~\cite{Tsamardinos2003} can be applied to estimate the Markov equivalence class and then the moralized graph is generated from the Markov equivalence class. 

Step 2) of the generalized ODS algorithm involves learning the causal ordering by comparing overdispersion scores of nodes using~\eqref{EqnTruncScorej}. The basic idea is to determine which nodes are overdispersed based on the sample conditional mean and conditional variance after the transformation in Proposition~\ref{prop:a}. The causal ordering is determined one node at a time by selecting the node with the smallest overdispersion score which is representative of a node that is least likely to be overdispersed.

Regarding the overdispersion scores, suppose that there are $n$ i.i.d.  samples $X^{1:n} := \{X^{(i)}\}_{i=1}^{n}$ where $X^{(i)} := (X_1^{(i)}, X_2^{(i)},\cdots, X_p^{(i)})$ is a $p$-variate random vector drawn from an underlying QVF DAG model with quadratic variance coefficients $(\beta_{j0}, \beta_{j1})_{j=1}^{p}$. We use the notation $\widehat{\cdot}$ to denote an estimate based on $(X^{(i)})_{i=1}^n$. In addition, we use $n(x_S) = \sum_{i=1}^n \mathbf{1}( X_S^{(i)} = x_S )$ to denote the conditional sample size, and $n_S = \sum_{ x_S } n(x_S) \mathbf{1}( n(x_S) \geq c_0 \cdot n )$ for an arbitrary $c_0 \in (0,1)$ to denote a truncated conditional sample size. We discuss the choice of $c_0$ shortly.

More precisely the overdispersion scores in Step 2) of~\ref{AlgODS} involves the following equations:
\begin{eqnarray}
\label{EqnTruncScore1}
\widehat{\mathcal{S}}(1,k) & := & 
\widehat{\omega}_j^2 \cdot  \widehat{\var}(X_j) - \widehat{\omega}_j \cdot \widehat{\E}(X_j) ~~~ \text{where}~~~ \widehat{\omega}_j := (\beta_{10} + \beta_{11} \widehat{\E}(X_j))^{-1},
\hspace{2.1cm}
\end{eqnarray}
\begin{eqnarray}
\label{EqnTruncScorej}
\widehat{\mathcal{S}}(j,k) & := & 
\sum_{x \in \mathcal{X}(\widehat{C}_{jk}) } \frac{ n(x) }{n_{\widehat{C}_{jk}}} \left[ \widehat{\omega}_{jk}(x)^2 \cdot \widehat{\var}(X_j \mid X_{\widehat{C}_{jk}} = x ) - \widehat{\omega}_{jk}(x) \cdot \widehat{\E}(X_j \mid X_{\widehat{C}_{jk}} = x) \right]
\end{eqnarray}
where $\widehat{\omega}_{jk}(x) := (\beta_{j0} + \beta_{j1} \widehat{\E}(X_j \mid X_{\widehat{C}_{jk}} = x))^{-1}$. $\widehat{C}_{jk}$ is the estimated candidate parents set of node $j$ for the $k^{th}$ element of the causal ordering, and $\mathcal{X}(\widehat{C}_{jk}) := \{x_{jk} \in \{X_{\widehat{C}_{jk}}^{(1)}, X_{\widehat{C}_{jk}}^{(2)},\cdots,X_{\widehat{C}_{jk}}^{(n)} \} : n(x_{jk}) \geq c_0 \cdot n \}$ to ensure we have enough samples for each element of an overdispersion score. $c_0$ is a tuning parameter of our algorithm that we specify in Theorem~\ref{ThmCausalOrdering} and our numerical experiments. 
The term $\widehat{\omega}_{jk}(x)$ is an empirical version of the transformation in Proposition~\ref{prop:a} assuming  $\widehat{C}_{jk}$ is the parents of a node $j$. 

Finding the set of parents of node $j$ boils down to selecting the parents out of all elements before node $j$ in the causal ordering. Hence given the estimated causal ordering from Step 2), Step 3) can be reduced to $p$ neighborhood selection problems which can be performed using $\ell_1$-penalized likelihood regression for GLMs~\cite{Friedman2009} as well as standard DAG learning algorithms such as the PC~\cite{spirtes2000causation}, GES~\cite{Chickering2003}, and MMHC algorithms~\cite{Tsamardinos2003}.

\subsection{Computational Complexity}

\label{SecCom}

For steps 1) and 3) of the generalized ODS algorithm, we use off-the-shelf algorithms and the computational complexity depends on the choice of algorithm. For example, if we use the neighborhood selection $\ell_1$-penalized likelihood regression for GLMs~\cite{Friedman2009} as is used in Yang et al.~\cite{Yang2012}, the worst-case complexity is $O(\min(n,p)np)$ for a single $\ell_1$-penalized likelihood regression, but since there are $p$ nodes, the total worst-case complexity is $O(\min(n,p)np^2)$. Similarly, if we use $\ell_1$-penalized likelihood regression for Step 3) the worst-case  complexity is also $O(\min(n,p)np^2)$ but maybe less if the degree $\d$ of the moralized graph is small. 

For Step 2) where we estimate the causal ordering, there are $(p-1)$ iterations and each iteration has a number of overdispersion scores $\widehat{S}(j,k)$ to be computed which is bounded by $O(\d)$ where $\d$ is the maximum degree of the moralized graph. Hence the total number of overdispersion scores that need to be computed is $O(p \d)$. Since the time for calculating each overdispersion score is proportional to the sample size $n$, the complexity is $O( n p \d )$. 

In the worst-case where the degree of the moralized graph is $p$, the computational complexity of Step 2) is $O( n p^2 )$. As we discussed earlier, there is a significant computational saving by exploiting the sparsity of the moralized graph which is why we perform Step 1) of the generalized ODS algorithm. Hence, Step 1) is the main computational bottleneck of the generalized ODS algorithm. The addition of Step 2) which estimates the causal ordering does not significantly add to the computational bottleneck. Consequently, the generalized ODS algorithm, which is designed for learning DAGs is almost as computationally efficient as standard methods for learning undirected graphical models. As we show in numerical experiments, the ODS algorithm using $\ell_1$-penalized likelihood regression for GLMs in both Steps 1) and 3) is faster than the state-of-the-art GES algorithm.

\subsection{Statistical Guarantees}

\label{SecStat}

In this section, we provide theoretical guarantees for our generalized ODS algorithm. We provide sample complexity guarantees for the ODS algorithm in the high-dimensional setting in three steps, by proving consistency of Steps 1), 2) and 3) in Sections~\ref{SecStep1},~\ref{SecStep2} and ~\ref{SecStep3}, respectively. All three main results are expressed in terms of the triple $(n,p,d)$.


Although any off-the-shelf algorithms can be used in Steps 1) and 3), our theoretical guarantees focus on the case when we use the R package glmnet~\cite{Friedman2009} for neighborhood selection. We focus on glmnet since there exist provable theoretical guarantees for neighborhood selection for graphical model learning in the high-dimensional setting (see e.g.,~\cite{Yang2012, Ravikumar2010}) and performs well in our simulation study. The glmnet package involves minimizing the $\ell_1$-penalized generalized linear model loss. 

Without loss of generality, assume that $(1,2,\cdots,p)$ is the true causal ordering and for ease of notation let $[\cdot]_{k}$ and $[\cdot]_{S}$ denotes parameter(s) corresponding to the variable $X_k$ and random vector $X_S$, respectively. Suppose that $\theta_{D_j}^* \in \Theta_{D_j}$ denotes the solution of the following GLM problem where $\Theta_{D_j} := \{ \theta \in \mathbb{R}^{p} : [\theta]_{k} = 0 \textrm{ for } k \notin \pa(j) \}$.
\begin{equation}
	\label{ThetaD}
	\theta_{D_j}^* := \arg \min_{\theta \in \Theta_{D_j}} \E \left( - X_j([\theta]_j + \langle [\theta]_{\pa(j)}, X_{\pa(j)} \rangle) + A_j( [\theta]_j + \langle [\theta]_{\pa(j)}, X_{\pa(j)} \rangle ) \right),
\end{equation}
where $A_j(\cdot)$ is the log-partition function determined by the GLM family~\eqref{DAGGLM}, and $\langle \cdot, \cdot \rangle$ represents the inner product. In the special case where $X_j$ has an NEF-QVF distribution with log-partition function $A_j(.)$, $\theta_{D_j}^*$ corresponds exactly to the set of true parameters, that is $\theta^*_{jk}$ is the co-efficient $k \in \pa(j)$ which represents the influence of of node $k$ on node $j$. However our results apply more generally and we do not require that $X_j$ belongs to an NEF-QVF DAG model.

Similar definitions are required for parameters associated with the moralized graph $G^m$. Define $\theta_{M_j}^* \in \Theta_{M_j}$ as the solution of the following GLM problem for a node $j$ over its neighbors where $\Theta_{M_j} := \{ \theta \in \mathbb{R}^{p} : [\theta]_{k} = 0 \textrm{ for } k \notin \mathcal{N}(j) \}$. 
\begin{equation}
	\label{ThetaM}
	\theta_{M_j}^* := \arg \min_{\theta \in \Theta_{M_j}} \E \left( - X_j([\theta]_j + \langle [\theta]_{\mathcal{N}(j)}, X_{\mathcal{N}(j)} \rangle) + A_j([\theta]_j + \langle [\theta]_{\mathcal{N}(j)}, X_{\mathcal{N}(j)} \rangle) \right).
\end{equation}

We impose the following identifiability assumptions on $\theta_{D_j}^*$ and $\theta_{M_j}^*$.

\begin{assumption}
\label{Ass:RestFaithfulness}
\begin{enumerate}
\item[(a)] For any node $j \in V$ and $k \in \pa(j)$,
	$$\mbox{Cov} (X_j, X_k) \neq \mbox{Cov}(X_k, \bigtriangledown A_j( [\theta_{D_j}^*]_j + \langle [\theta_{D_j}^*]_{\pa(j) \setminus k}, X_{\pa(j) \setminus j} \rangle ) ).
	$$
\item[(b)] For any node $j \in V$ and $k \in \mathcal{N}(j)$,
	$$\mbox{Cov} (X_j, X_k) \neq \mbox{Cov}(X_k, \bigtriangledown A_j( [\theta_{M_j}^*]_j+ \langle [\theta_{M_j}^*]_{\mathcal{N}(j) \setminus k}, X_{\mathcal{N}(j) \setminus j} \rangle ) ).
	$$
\end{enumerate}
\end{assumption}
Assumption~\ref{Ass:RestFaithfulness} can be understood as a notion of restricted faithfulness only for neighbors and parents for each node. To provide intuition consider the special case of Gaussian DAG models. The log-partition function is $A_j(\eta) = \frac{\eta^2}{2}$, so that $\bigtriangledown A_j(\eta) = \eta$. Then, the condition boils down to $\mbox{Cov} (X_j, X_k) \neq \sum_{m \in \pa(j) \setminus k} [\theta_{D_j}^*]_{m} \mbox{Cov}(X_k, X_m)$, meaning the directed path from $X_k$ to $X_j$ does not exactly cancel the sum of paths from other parents of $X_k$. For general exponential families, the right-hand side involves non-linear functions of the variables of $X$ corresponding to sets of measure $0$. Under Assumption~\ref{Ass:RestFaithfulness}, the following result holds.

\begin{lemma}
\label{Lem:RestFaithfulness}
\begin{enumerate}
\item[(a)] Under Assumption~\ref{Ass:RestFaithfulness}(a), for all $1 \leq j \leq p$, $\mbox{supp}(\theta_{D_j}^*) = \pa(j)$. 
\item[(b)] Under Assumption~\ref{Ass:RestFaithfulness}(b), for all $1 \leq j \leq p$, $\mbox{supp}(\theta_{M_j}^*) = \mathcal{N}(j)$. 
\end{enumerate}
\end{lemma}
Using the parameters $(\theta_{M_j}^*)_{j=1}^{p}$ and $(\theta_{D_j}^*)_{j=1}^p$ and their relationships to $\pa(j)$ and $\mathcal{N}(j)$ respectively, we provide consistency guarantees for Steps 1) and 3) respectively.

\subsubsection{Step 1): Recovery of the Moralized Graph via $\ell_1$-penalized likelihood regression for GLMs}

\label{SecStep1}

We first focus on the theoretical guarantee for recovering the moralized graph $G^m$. As we mentioned earlier, we approach this problem by solving an empirical version of the $\ell_1$-penalized likelihood regression. Given $n$ i.i.d. samples $X^{1:n} = (X^{(i)})_{i=1}^{n}$ where $X^{(i)} = (X_1^{(i)}, X_2^{(i)}\cdots, X_p^{(i)})$ is a $p$-variate random vector drawn from the underlying DAG model, we define the conditional negative log-likelihood for a variable $X_j$:
\begin{equation}
\label{Eq1}
\ell_j( \theta; X^{1:n}) := \frac{1}{n} \sum_{i = 1}^{n} \left( -X_j^{(i)}([\theta]_j + \langle [\theta]_{V \setminus j}, X_{V \setminus j}^{(i)} \rangle) + A_j([\theta]_j + \langle [\theta]_{V \setminus j}, X_{V \setminus j}^{(i)} \rangle)  \right)
\end{equation}
where $\theta \in \mathbb{R}^{p}$ and $A_j(\cdot)$ is the log-partition function determined based on the chosen GLM family. 

We analyze the $\ell_1$-penalized log-likelihood for each node $j \in V$: 
\begin{equation}
\label{P1}
\hat{\theta}_{M_j} := \arg \min_{ \theta \in \mathbb{R}^{p} } \ell_j( \theta; X^{1:n}) + \lambda_n \| [\theta]_{V \setminus j} \|_1 
\end{equation}
where $\lambda_n > 0$ is the regularization parameter. Based on $\hat{\theta}_{M_j}$, the estimated neighborhood of node $j$ is $\widehat{\mathcal{N}}(j) := \{ k \in V \setminus j: [\hat{\theta}_M]_{k} \neq 0 \}$. 
Based on Lemma~\ref{Lem:RestFaithfulness}, $\mbox{supp}(\theta_{M_j}^*) = \mathcal{N}(j)$
where $\theta_{M_j}^*$ is defined by~\eqref{ThetaM}. Hence if for each $j$, $\hat{\theta}_{M_j}$ in~\eqref{P1} is sufficiently close to $\theta_{M_j}^*$, we conclude that $\widehat{\mathcal{N}}(j) = \mathcal{N}(j)$. 

We begin by discussing the assumptions we impose on the DAG $G$. Since we apply the neighborhood selection strategy in Steps 1) and 3), we will present assumptions for both steps here. Most of the assumptions are similar to those imposed in ~\cite{Yang2012} where neighborhood selection is used for graphical model learning. Important quantities are the Hessian matrices of the negative conditional log-likelihood of a variable $X_j$ given either the rest of the nodes $Q^{M_j}= \bigtriangledown^2 \ell_j(\theta_{M_j}^*; X^{1:n})$, and the nodes before $j$ in the causal ordering $Q^{D_j} = \bigtriangledown^2 \ell_j^D(\theta_{D_j}^*; X^{1:n})$ which we discuss in Section~\ref{SecStep3}. Let $A_{\S \S}$ be the $|S| \times |S|$ sub-matrix of the matrix $A_j$ corresponding to variables $X_{\S}$. 

\begin{assumption}[Dependence assumption]
	\label{A1Dep}
	There exists a constant $\rho_{\min} > 0 $ such that 
	$$\min_{j \in V} \min( \lambda_{\min}(Q_{\mathcal{N}(j) \mathcal{N}(j)}^{M_j}), \lambda_{\min}(Q^{D_j}_{\pa(j) \pa(j)}) ) \geq \rho_{\min}.$$ 
	Moreover, there exists a constant $\rho_{\max} < \infty$ such that 
	$$\max_{j \in V } \left( \lambda_{\max} \left( \frac{1}{n}\sum_{i=1}^{n} X_{\mathcal{N}(j)}^{(i)} (X_{\mathcal{N}(j)}^{(i)})^T \right) \right) \leq \rho_{\max}$$ 
	where $\lambda_{\min}(A)$ and $\lambda_{\max}(A)$ are the smallest and largest eigenvalues of the matrix $A$, respectively. 
\end{assumption}

\begin{assumption}[Incoherence assumption]
	\label{A2Inc}
	There exists a constant $\alpha \in (0,1]$ such that 
	$$\max_{j \in V} \max \left( \max_{t \in \mathcal{N}(j)^c} \| Q_{t \mathcal{N}(j)}^{M_j} (Q_{\mathcal{N}(j) \mathcal{N}(j)}^{M_j})^{-1}\|_1, \max_{t' \in \pa(j)^c} \| Q^{D_j}_{t' \pa(j)} (Q^{D_j}_{\pa(j) \pa(j)})^{-1}\|_1 \right) \leq 1 - \alpha.$$
\end{assumption}

The dependence assumption~\ref{A1Dep} can be interpreted as ensuring that the variables in both $\mathcal{N}(j)$ and $\pa(j)$ are not too dependent. In addition, the incoherence assumption~\ref{A2Inc} ensures that variables that are not in the set of true variables are not highly correlated with variables in the true variable set. These two assumptions are standard in all neighborhood regression approaches for variable selection involving $\ell_1$-based methods and these conditions have imposed in proper work both for high-dimensional regression and graphical model learning~\cite{Yang2012, meinshausen2006high, Wainwright2006, Ravikumar2011}.

To ensure suitable concentration bounds hold, we impose two further technical assumptions. Firstly we require a boundedness assumption on the moment generating function to control the tail behavior.

\begin{assumption}[Concentration bound assumption]
	\label{A3Con}
	There exists a constant $M > 0$ such that 
	$$\max_{j \in V} \E( \exp(|X_j|) ) < M$$.
\end{assumption}

We also require conditions on the first and third derivatives on the log-partition functions $A_j(.)$ for $1 \leq j \leq p$ in~\eqref{Eq1} and~\eqref{Eq1D}. Let $A_j'(.)$ and $A_j'''(.)$ are the first and third derivatives of $A_j(.)$ respectively.

\begin{assumption}[Log-partition assumption]
	\label{A4}
	For the log-partition functions $A_j(\cdot)$ in~\eqref{Eq1} or~\eqref{Eq1D}, there exist constants $\kappa_1$ and $\kappa_2$ such that $\max_{j \in V}\{|A_j'(a)|, |A_j'''(a)|\} \leq n^{\kappa_2}$ for $a \in [0, \kappa_1  \max\{\log(n),\log(p)\} )$, $\kappa_1 \geq 6 \max( \| \theta_{M_j}^* \|_1, \| \theta_{D_j}^* \|_1)$ and $\kappa_2 \in [0,1/4]$. 
\end{assumption}

Prior work in~\cite{Yang2012, Ravikumar2011, jalali2011learning} impose similar technical conditions that control the tail behavior of $(X_j)_{j=1}^p$. It is important to note that there exist many distributions and associated parameters that satisfy these assumptions. For example the Binomial, Multinomial or Exponential distributions, the log-partition assumption~\ref{A4} is satisfied with $\kappa_2 = 0$ because the log-partition function $A_j(\cdot)$ is bounded. For the Poisson distribution which has one of the steepest log-partition function, $A_j(\cdot) = \exp(\cdot)$. Hence, in order to satisfy Assumption~\ref{A4}, we require $\|\theta_{M_j}^*\|_1 \leq \frac{ \log n }{ 48 \log p } $ with $\kappa_2 = \frac{1}{8}$. 

Putting together Assumptions~\ref{A1Dep}~\ref{A2Inc},~\ref{A3Con}, and~\ref{A4}, we have the following main result that  the moralized graph can be recovered via $\ell_1$-penalized likelihood regression for GLMs in high-dimensional settings.

\begin{theorem}[Learning the moralized graph]
	\label{ThmMoralGraph}
Consider the DAG model~\eqref{EqnFactorization} satisfying the QVF property~\eqref{eq:Quad} and $d$ is the maximum degree of the moralized graph. Suppose that Assumptions~\ref{Ass:RestFaithfulness}(b),~\ref{A1Dep},~\ref{A2Inc},~\ref{A3Con} and~\ref{A4} are satisfied. Assume $\hat{\theta}_{M_j}$ is any solution to the optimization problem~\eqref{P1} and $\frac{ 9 \log^2( \max\{n,p\} )}{n^{a}} \leq \lambda_n \leq \frac{ \rho_{\min}^2 }{ 30 n^{\kappa_2} \log(\max\{n,p\}) \d \rho_{\max} }$ for some $a \in (2\kappa_2, 1/2)$, and $\min_{j \in V} \min_{t \in \mathcal{N}(j)} |[\theta_{M}^*]_t| \geq \frac{10}{\rho_{\min}} \sqrt{\d} \lambda_n$. Then for any constant $\epsilon > 0$, there exists a positive constant $C_{\epsilon}$ such that if $n \geq C_{\epsilon} ( \d \log^3 \max\{n,p\} )^{ \frac{1}{a - \kappa_2} }$,
$$
\mathbb{P}(\mbox{supp}(\hat{\theta}_{M_j}) = \mathcal{N}(j)) \geq 1 - \epsilon,
$$
for all $j \in V$.
\end{theorem}

We defer the proof to Appendix~\ref{SecThmStep1Proof}. The key technique for the proof is that standard \emph{primal-dual witness} method used in Wainwright~\cite{Wainwright2006}; Ravikumar et al.~\cite{Ravikumar2011}; Jalali et al.~\cite{jalali2011learning}; and Yang et al.~\cite{Yang2012}. Theorem~\ref{ThmMoralGraph} shows that the moralized graph $G^m$ can be recovered via $\ell_1$-penalized likelihood regression if sample size $n = \Omega(( \d \log^3( \max\{n,p\} ) )^{ \frac{1}{a - \kappa_2} })$ with high probability.

\subsubsection{Step 2): Recovering the Causal Ordering using OverDispersion Scores}

\label{SecStep2}

In this section, we provide theoretical guarantees for recovering the causal ordering for the DAG $G$ via our generalized ODS algorithm. The first required condition is a stronger version of the identifiability assumption required for Theorem~\ref{Thmidentifiability} since we move from the population distribution to the finite sample setting. 
 
\begin{assumption}
	\label{A1}
	For all $j \in V$ and any $K_j \subset \pa(j)$ where $K_j \neq \emptyset$ and $S \subset  \nd(j) \setminus K_j$:
	\begin{itemize}
		\item[(a)] There exists an $M_{\min} > 0$ such that 
		$\var( \E(X_j \mid X_{pa(j)}) \mid X_S) > M_{\min}.$
		\item[(b)] There exists an $\omega_{\min}>0$ such that  
		$| \beta_{j0} + \beta_{j1} \E(X_j \mid X_{S} ) | > \omega_{\min}.$
	\end{itemize}
\end{assumption}

Assumption~\ref{A3Con} is required since the overdispersion score is sensitive to the accuracy of the sample conditional mean and conditional variance. Since the true causal ordering $\pi^*$ may not be unique, we use $\mathcal{E}(\pi^*)$ to denote the set of all the causal orderings that are consistent with the true DAG $G$. 

\begin{theorem}[Recovery of the causal ordering]
	\label{ThmCausalOrdering}
	Consider the DAG model~\eqref{EqnFactorization} satisfying the QVF property~\eqref{eq:Quad} with co-efficients $(\beta_{j0}, \beta_{j1})_{j=1}^{p}$ and $d$ is the maximum degree of the moralized graph. Suppose that $\beta_{j1} > - 1$ for all $j \in V$, and the structure of the moralized graph $G^m$ is known. Suppose also that Assumptions~\ref{A3Con} and~\ref{A1} are satisfied. Then for any $\epsilon > 0$ and $c_0 \geq  \log^d \max \{n,p\}$, there exists a positive constant $K_{\epsilon}$ such that for $n \geq K_{\epsilon}\log^{5+d}(\max\{n,p\})$, 
	\begin{equation*}
	P( \widehat{\pi} \in \mathcal{E}(\pi^*) ) \geq 1 - \epsilon.
	\end{equation*}
\end{theorem}

The detail of the proof is provided in Appendix~\ref{SecThmCausalOrderingProof}. The proof is novel and involves the combination of the transformation and overdispersion property exploited in Theorem~\ref{Thmidentifiability}. Intuitively, the estimated overdispersion scores $\widehat{\mathcal{S}}(j,k)$ converge to the true overdispersion scores $\mathcal{S}(j,k)$ as the sample size $n$ increases which is where we exploit Assumption~\ref{A3Con}. This allows us to recover a true causal ordering for the DAG $G$. Assuming the moralized graph $G^m$ is known is essential to exploiting the degree condition on the moralized graph and emphasizes the importance of Step 1) and Theorem~\ref{ThmMoralGraph}.
 
Theorem~\ref{ThmCausalOrdering} claims that if the triple $(n,d,p)$ satisfies $n = \Omega(\log^{5+d} p )$, our generalized ODS algorithm correctly estimates the true causal ordering. Therefore if the moralized graph is sparse (i.e.,$d = \Omega( \log p)$), our generalized ODS algorithm recovers the true casual ordering in the high-dimensional settings. Note that if the moralized graph is not sparse and $d = \Omega(p)$, the generalized ODS algorithm requires an extremely large sample size. Prior work on DAG learning algorithms in the high-dimensional setting has been based on learning the Markov equialence class in settings with additive independent noise (see e.g.,~\cite{loh2014high,geerpb12}). 

\subsubsection{Step 3): Recovery of the DAG via $\ell_1$-penalized likelihood regression}

\label{SecStep3}

Similar to Step 1), we provide a theoretical guarantee for Step 3) using $\ell_1$-penalized likelihood regression where we estimate the parents of each node $\pa(j)$. Importantly, we assume that Step 2) of the ODS algorithm has occurred and using Theorem~\ref{ThmCausalOrdering}, a true causal ordering has been learned. Recall that we impose the assumption that the true causal ordering is $\pi^* = (1,2,\cdots,p)$. Then, we estimate the parents of a node $j$ over the possible parents $\{1, 2,\cdots,j-1\}$. 

For notational convenience, we use $X_{1:j} = (X_1, X_2,\cdots,X_j)$. Then for any variable $X_j$, the conditional negative log-likelihood for a given GLM is as follows:
\begin{equation}
	\label{Eq1D}
	\ell_{j}^D(\theta; X^{1:n}) := \frac{1}{n} \sum_{i = 1}^{n} \left( -X_j^{(i)}( [\theta]_j + \langle [\theta]_{1:j-1}, X_{1:j-1}^{(i)} \rangle) + A_j([\theta]_j + \langle [\theta]_{1:j-1}, X_{1:j-1}^{(i)} \rangle )  \right)
\end{equation}
where $\theta \in \mathbb{R}^{j}$, and $A_j(\cdot)$ is the log-partition function determined by a chosen GLM family.

We solve the negative conditional log-likelihood with $\ell_1$ norm penalty for each variable $X_j$: 
\begin{equation}
	\label{P1D}
	\hat{\theta}_{D_j} := \arg \min_{ \theta \in \mathbb{R}^{j} } \ell_j^{D}( \theta; x ) + \lambda_n^D \| [\theta]_{1:j-1} \|_1.
\end{equation}

Recall that under Assumption~\ref{Ass:RestFaithfulness}(a), Lemma~\ref{Lem:RestFaithfulness}(a) shows that $\mbox{supp}(\theta_{D_j}^*) = \pa(j)$. Hence if the solution of~\eqref{P1D} for each node $j \in V$ is close to $\theta_{D_j}^*$ in~\eqref{ThetaD}, $\ell_1$-penalized likelihood regression successfully recovers the parents of node $j$. 

\begin{theorem}[Learning DAG structure]
	\label{ThmDirectGraph}
Consider the DAG model~\eqref{EqnFactorization} satisfying the QVF property~\eqref{eq:Quad} and $d$ is the maximum degree of the moralized graph. Suppose that Assumptions~\ref{Ass:RestFaithfulness}(a),~\ref{A1Dep},~\ref{A2Inc},~\ref{A3Con} and~\ref{A4} are satisfied. Assume $\hat{\theta}_{D_j}$ is any solution to the optimization problem~\eqref{P1D} and $\frac{ 9 \log^2 ( \max\{n,p\} )}{n^{a}} \leq \lambda_n^D \leq \frac{ \rho_{\min}^2 }{ 30 n^{\kappa_2} \log(\max\{n,p\}) \d \rho_{\max} }$ for some $a \in (2\kappa_2, 1/2)$, and $\min_{j \in V} \min_{t \in \mathcal{N}(j)} |[\theta_{D}^*]_t| \geq \frac{10}{\rho_{\min}} \sqrt{\d} \lambda_n$. Then for any $\epsilon > 0$, there exists a positive constant $C_{\epsilon}$ such that if $n \geq C_{\epsilon} ( \d \log^3 (\max\{n,p\}) )^{ \frac{1}{a - \kappa_2} }$,
$$
\mathbb{P}(\mbox{supp}(\hat{\theta}_{D_j}) = \mbox{pa}(j)) \geq 1 - \epsilon,
$$
for all $j \in V$.
\end{theorem}

The details of the proof are provided in Appendix~\ref{SecThmStep3Proof}. The proof technique is again based on the primal-dual technique as is used for the proof of Theorem~\ref{ThmMoralGraph}. Theorem~\ref{ThmDirectGraph} shows that $\ell_1$-penalized likelihood regression successfully recovers the structure of $G$ if the sample size is $n = \Omega( (\d \log^3( \max \{n,p\}) )^{ \frac{1}{a - \kappa_2} })$ given the true causal ordering. Note once again that we exploit the sparsity $\d$ of the moralized graph. 

So far, we have provided sample complexity guarantees for all three steps of the generalized ODS algorithm. Combining Theorems~\ref{ThmMoralGraph},~\ref{ThmCausalOrdering}, and~\ref{ThmDirectGraph}, we reach our final main result that the generalized ODS algorithm successfully recovers the true structure of a QVF DAG with high probability. Furthermore if $G$ is sparse (i.e., $\d = \Omega(\log p )$), the generalized ODS algorithm recovers the structure of QVF DAG models in the high-dimensional setting.

\begin{corollary}[Learning QVF DAG models]
	\label{CorMain}
	Consider the DAG model~\eqref{EqnFactorization} satisfying the QVF property~\eqref{eq:Quad} and $d$ is the maximum degree of the moralized graph. Suppose that Assumptions~\ref{Ass:RestFaithfulness},~\ref{A1Dep},~\ref{A2Inc},~\ref{A3Con} and~\ref{A4} are satisfied and all other conditions of Theorems~\ref{ThmMoralGraph},~\ref{ThmCausalOrdering}, and~\ref{ThmDirectGraph} are satisfied and $\widehat{G}$ is the output of the ODS algorithm. Then for any $\epsilon > 0$, there exists a positive constant $C_{\epsilon}$ such that if $n \geq C_{\epsilon} \max( \d \log^3(\max\{n,p\}) )^{ \frac{1}{a - \kappa_2} }, \log^{5+d} p )$,
$$
\mathbb{P}(\widehat{G} = G) \geq 1 - \epsilon.
$$
\end{corollary}

Concretely, we apply Corollary~\ref{CorMain} to popular examples for our class of QVF DAG models.
As we discussed earlier, Poisson DAG models have $(\beta_{j0}, \beta_{j1}) = (1, 0)$, the steepest log-partition function $A_j(\cdot) = \exp(\cdot)$, and $\kappa_2 = \frac{1}{8}$ if $\|\theta_{M_j}^*\|_1 \leq \frac{ \log n }{ 48 \log(\max\{n,p\}) }$. Then, our generalized ODS algorithm recovers Poisson DAG models with high probability if $n = \Omega(\max\{(\d \log^3 p)^4, \log^{5+d} p\})$ and $a = \frac{3}{8}$. Binomial DAG models have $(\beta_{0j}, \beta_{1j}) = (0, -\frac{1}{N})$ where $N$ is a binomial distribution parameter, the log-partition function $A_j(\cdot) = N \log(1+\exp(\cdot) )$, $\kappa_2 = 0$.  Then, the generalized ODS algorithm recover Binomial DAG models with high probability if $n = \Omega(\max\{(\d \log^3 p)^3, \log^{5+d} p\})$ and $a = \frac{1}{3}$.

\section{Numerical Experiments}

\label{SecNum}

In this section, we support our theoretical guarantees with numerical experiments and show that our generalized ODS algorithm performs favorably compared to state-of-the-art DAG learning algorithms when applied to QVF DAG models. In order to validate Theorems~\ref{ThmMoralGraph},~\ref{ThmCausalOrdering}, and~\ref{ThmDirectGraph}, we conduct a simulation study using $50$ realizations of $p$-node Poisson and Binomial DAG models~\eqref{DAGGLM}. That is, the conditional distribution for each node given its parents is either Poisson and Binomial. For all our simulation results, we generate DAG models (see Figure~\ref{fig:structure}) that ensure a unique causal ordering $\pi^* =(1,2,\cdots,p)$ with edges randomly generated while respecting the desired maximum number of parents constraints for the DAG. In our experiments, we always set the number of parents to two (the number of neighbors of each node is at least three, and therefore $\d \in [3, p-1]$). 

The set of parameters $(\theta_{jk})$ for our GLM DAG models~\eqref{DAGGLM} encodes the DAG structure as follows: if there is no directed edge from node $k$ to $j$, $\theta_{jk} = 0$, otherwise $\theta_{jk} \neq 0$. Non-zero parameters $\theta_{jk}\in E$ were generated uniformly at random in the range $\theta_{jk} \in [-1, -0.5]$ for Poisson DAG models and $\theta_{jk} \in [0.5, 1]$ for Binomial DAG models. In addition, we fixed parameters $N_1, N_2,\cdots,N_p = 4$ for Binomial DAG models. These parameter values were chosen to ensure Assumptions~\ref{A3Con} and~\ref{A4} are satisfied and most importantly, the count values do not blow up. Lastly, we set the thresholding constant for computing the ODS score to $c_0 = 0.005$ although any value below $0.01$ seems to work well in practice. We consider more general parameter choices but for brevity, focus on these parameter settings.

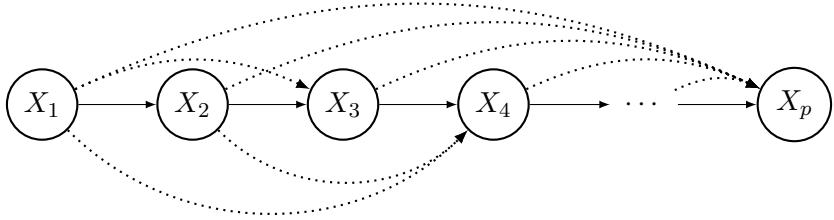
\begin{figure}[t]
	\centering
	\begin {tikzpicture}[ -latex ,auto,
	state/.style={circle, draw=black, fill= white, thick, minimum size= 2mm},
	state2/.style={circle, draw=white, fill= white, thick, minimum size= 5mm},
	label/.style={thick, minimum size= 2mm}
	]
	\node[state] (X1)  at (0,0)   {$X_1$}; 
	\node[state] (X2)  at (2,0)   {$X_2$}; 
	\node[state] (X3)  at (4,0)   {$X_3$}; 
	\node[state] (X4)  at (6,0)   {$X_4$}; 
	\node[state2](X5)  at (8,0)   {$\cdots$}; 
	\node[state] (X6)  at (10,0)   {$X_p$}; 

	\path (X1) edge[] node[] {} (X2);
	\path (X2) edge[] node[] {} (X3);
	\path (X3) edge[] node[] {} (X4);
	\path (X4) edge[] node[] {} (X5);	
	\path (X5) edge[] node[] {} (X6);	
	\path (X1) edge[dotted, thick, bend left = 25] node[] {} (X3);
	\path (X1) edge[dotted, thick, bend right = 45] node[] {} (X4);
	\path (X2) edge[dotted, thick, bend right = 45] node[] {} (X4);
	\path (X1) edge[dotted, thick, bend right = -25] node[] {} (X6);
	\path (X2) edge[dotted, thick, bend right = -25] node[] {} (X6);
	\path (X3) edge[dotted, thick, bend right = -25] node[] {} (X6);
	\path (X4) edge[dotted, thick, bend right = -25] node[] {} (X6);
	\path (X5) edge[dotted, thick, bend right = -25] node[] {} (X6);
\end{tikzpicture}
\caption{Structure of the DAG we used in numerical experiments. Solid directed edges are always present and dotted directed edges are randomly chosen based on the given number of parents of each node constraints}
\label{fig:structure}
\end{figure}

To validate Theorems~\ref{ThmMoralGraph} and~\ref{ThmCausalOrdering}, we plot the proportion (out of $50$) of simulations in which our generalized ODS algorithm recovers the correct causal ordering to validate $\pi^*$ in Fig.~\ref{fig:Num1}. We plot the accuracy rates in recovering the true causal ordering $\mathbf{1}(\hat{\pi} = \pi^*)$ as a function of the sample size ($n \in \{100, 500, 1000, 2500, 5000, 10000\}$) for different node sizes ($p=10$ for (a) and (c), and $p=100$ for (b) and (d)) and different distributions (Poisson for (a) and (b) and Binomial for (c) and (d)). In each sub-figure, two different choices for off-the-shelf algorithms for Step 1) are used; (i) $\ell_1$ penalized likelihood regression~\cite{Friedman2009} where we chose the regularization parameter $\lambda = \frac{0.75}{ \log( \max\{n,p\} ) }$ for Poisson DAG models and $\lambda = \frac{.10}{ \log( \max \{n,p\} ) }$ for Binomial DAG models; and (ii) the GES algorithm~\cite{Chickering2003} is applied for Step 1) where we used the mBDe~\cite{Heckerman1995} (modified Bayesian Dirichlet equivalent) score and then the moralized graph is generated by moralizing the estimated DAG. 

Figure~\ref{fig:Num1} shows that our generalized ODS algorithm recovers the true causal ordering $\pi^*$ well if the sample size is large, which supports our theoretical results. In addition, we can see that the $\ell_1$-penalized based generalized ODS algorithm seems to perform substantially better than the GES-based ODS algorithm. Furthermore, since $\ell_1$-penalized likelihood regression is the only algorithm that scales to the high-dimensional setting ($p \geq 1000$), we used $\ell_1$-penalized likelihood regression in Steps 1) and 3) of the generalized ODS algorithm for large-scale DAG models.

\begin{figure}[t]
	\centering  \hspace{-8mm}
	\begin{subfigure}[!htb]{.23\textwidth}
		\includegraphics[width=\textwidth,height= 30mm]{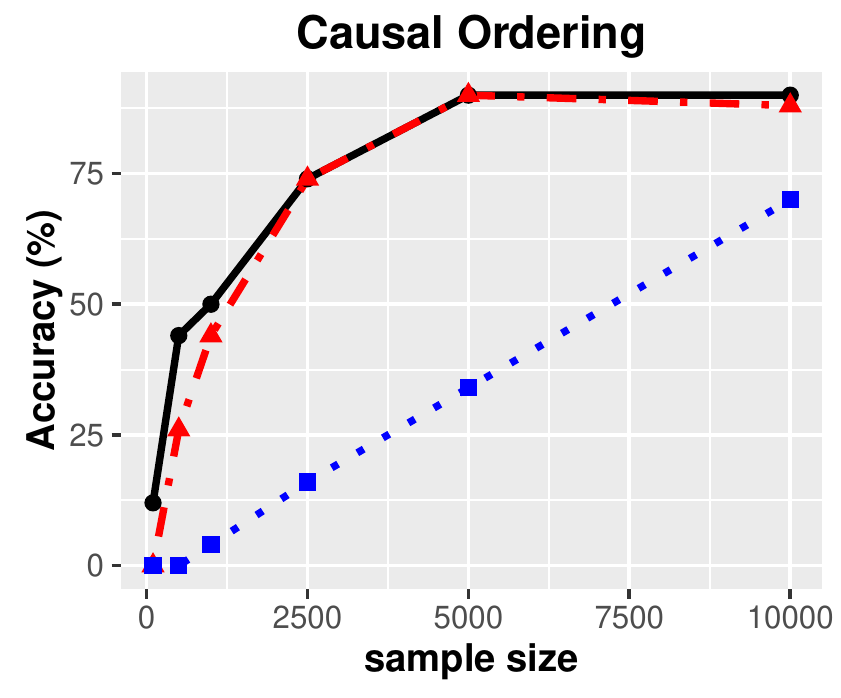}
		\caption{Poisson: $p=10$}
	\end{subfigure}
	\hspace{-2mm}
	\begin{subfigure}[!htb]{.23\textwidth}
		\includegraphics[width=\textwidth,height= 30mm]{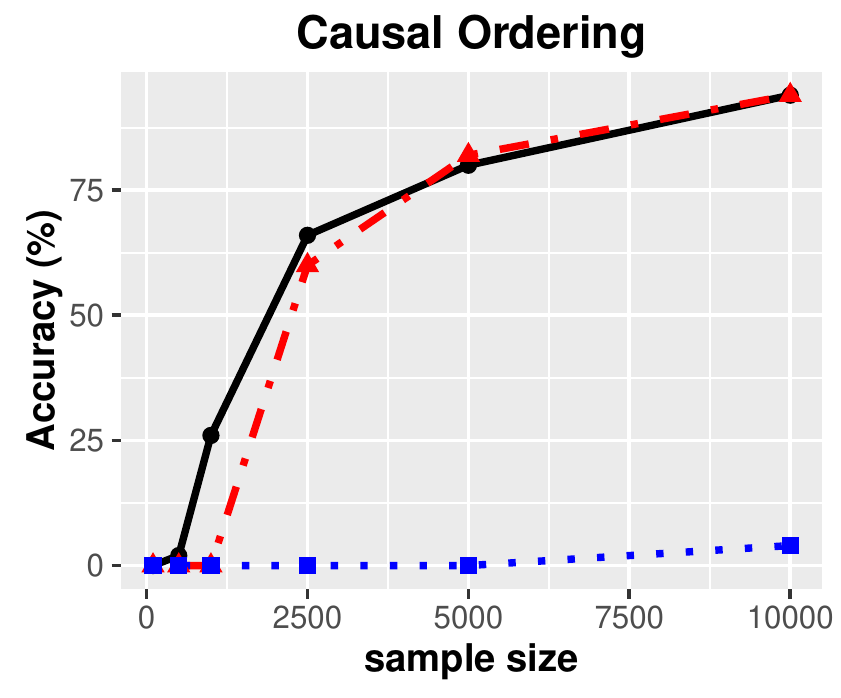}
		\caption{Poisson:  $p=100$}
	\end{subfigure} 
	\hspace{-2mm}
	\begin{subfigure}[!htb]{.23\textwidth}
		\includegraphics[width=\textwidth,height= 30mm]{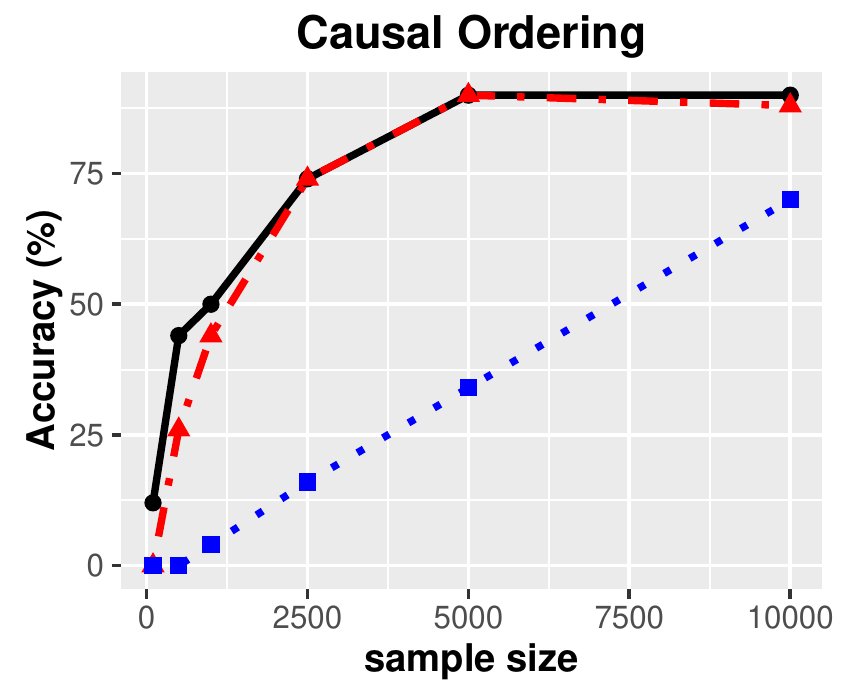}
		\caption{Binomial: $p=10$}
	\end{subfigure} \hspace{-2mm}
	\begin{subfigure}[!htb]{.23\textwidth}
		\includegraphics[width=\textwidth,height= 30mm]{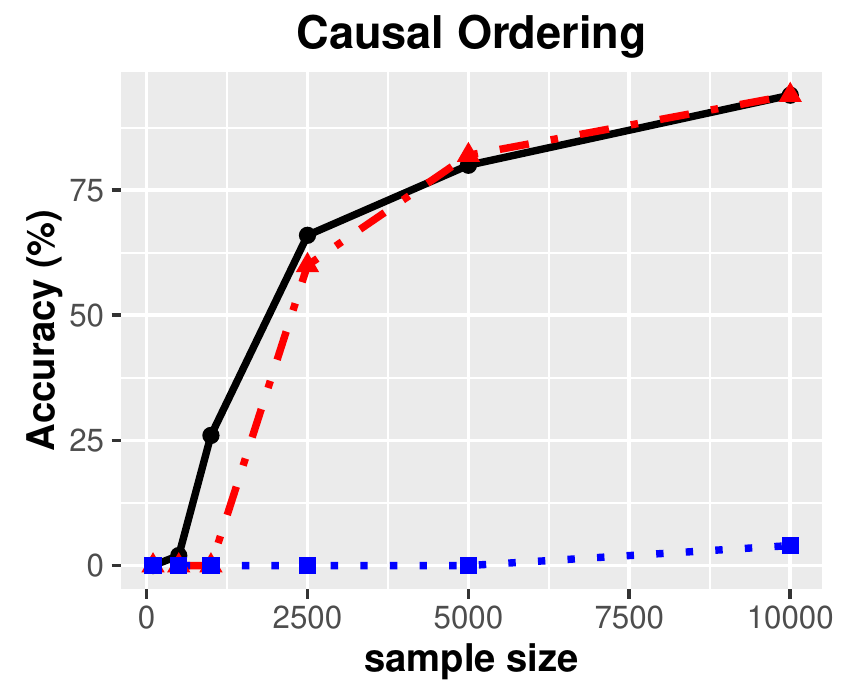}
		\caption{Binomial: $p=100$}
	\end{subfigure} \hspace{-2mm}
	\begin{subfigure}[!htb]{.08\textwidth}
		\includegraphics[width = 1.0\textwidth, height = 30mm, trim = 32mm -27mm 5mm 5mm, clip]{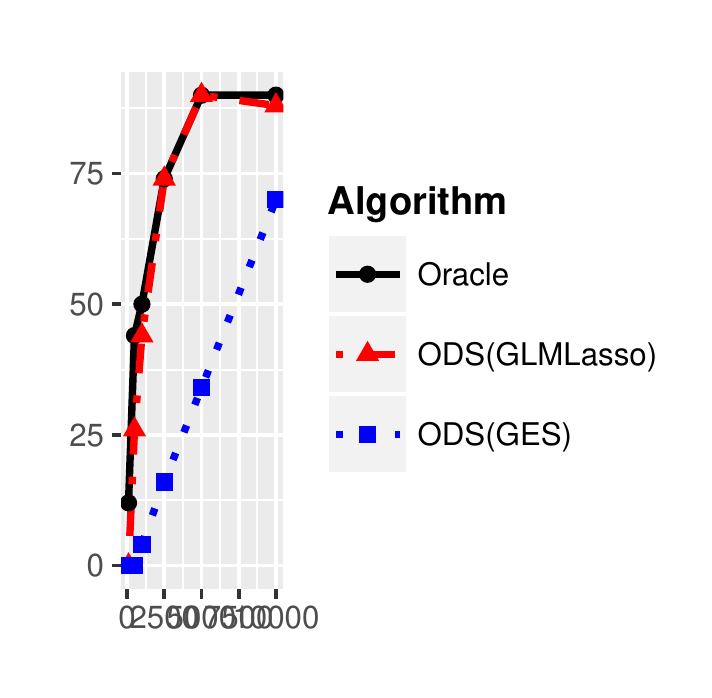}
	\end{subfigure}
	\caption{Probability of recovering the causal ordering of a DAG via our generalized ODS algorithm using two different algorithms ($\ell_1$-penalized likelihood regression and GES algorithm) in Step 1)}
	\label{fig:Num1}
\end{figure}

Figure~\ref{fig:Num2} provides a comparison of how accurately our generalized ODS algorithm performs in terms of recovering the full DAG model. We use two comparison metrics related to how many edges and directions are incorrect. First, we measured the Hamming distance between the skeleton (edges without directions) of the true DAG and the estimated DAG in (a), (c), (e) and (g). In addition, we measured the Hamming distance between the estimated and true DAG models (with directions) in (b), (d), (f), and (h).  We normalized the Hamming distances by dividing by the maximum number of errors $\binom{p}{2}$ for the skeleton and $p(p-1)$ for the full DAG respectively meaning the maximum normalized distance is $1$. We compare to two state-of-the-art directed graphical model learning algorithms, the MMHC and GES algorithms for both Poisson and Binomial DAG models. Similar to learning the causal ordering, we used two generalized ODS algorithms exploiting $\ell_1$-penalization in both Steps 1) and 3) and the GES algorithm in both Steps 1) and 3). We considered small-scale DAG models with $p = 10$ in (a), (b), (e) and (f), and $p = 100$ in (c), (d), (g) and (h). 

As we see in Figure~\ref{fig:Num2}, the ODS algorithms significantly out-perform state-of-the-art MMHC and GES algorithms in terms of directed edges and skeleton. For small sample sizes, the generalized ODS algorithms have poor performance because they fail to recover the causal ordering, however we can see that the GES-based generalized ODS algorithm always performs better than the GES algorithm. This is because the generalized ODS algorithm adds directional information to the estimated skeleton via the GES algorithm, and hence the GES-based generalized ODS algorithm cannot be worse than the GES algorithm in terms of recovering both directed edges and skeleton. Furthermore Figure~\ref{fig:Num2} shows that as sample size increases, our generalized ODS algorithms recovers the true directed edges and the skeleton for the DAG more accurately than state-of-the-art methods, which is consistent with our theoretical results. 

\begin{figure}[!t]
	\centering \hspace{-8mm}
	\begin{subfigure}[!htb]{.23\textwidth} 
		\includegraphics[width=\textwidth,height= 32mm]{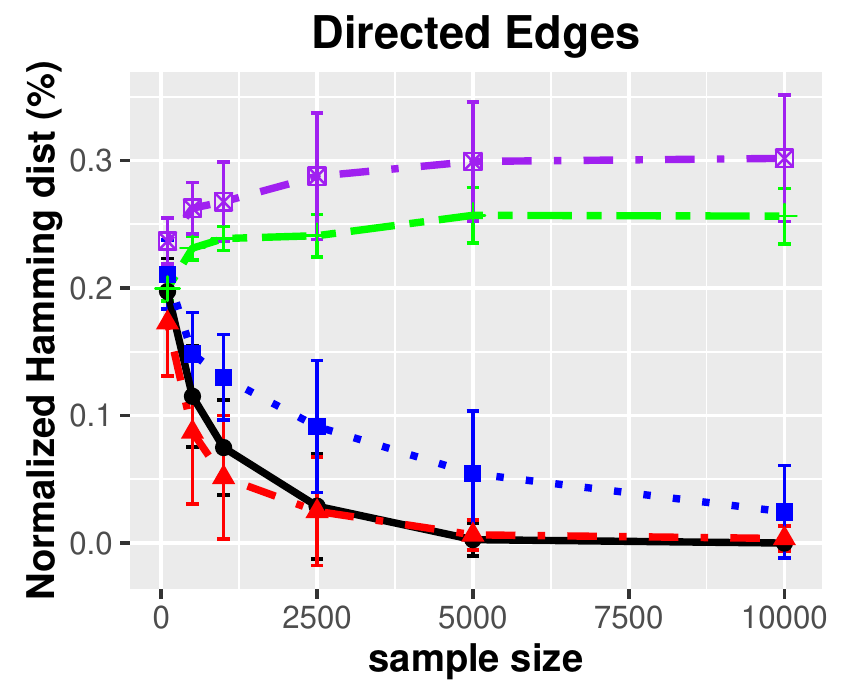}
		\caption{Poisson: $p=10$  }
	\end{subfigure} \hspace{-2mm}
	\begin{subfigure}[!htb]{.23\textwidth}
		\includegraphics[width=\textwidth,height= 32mm]{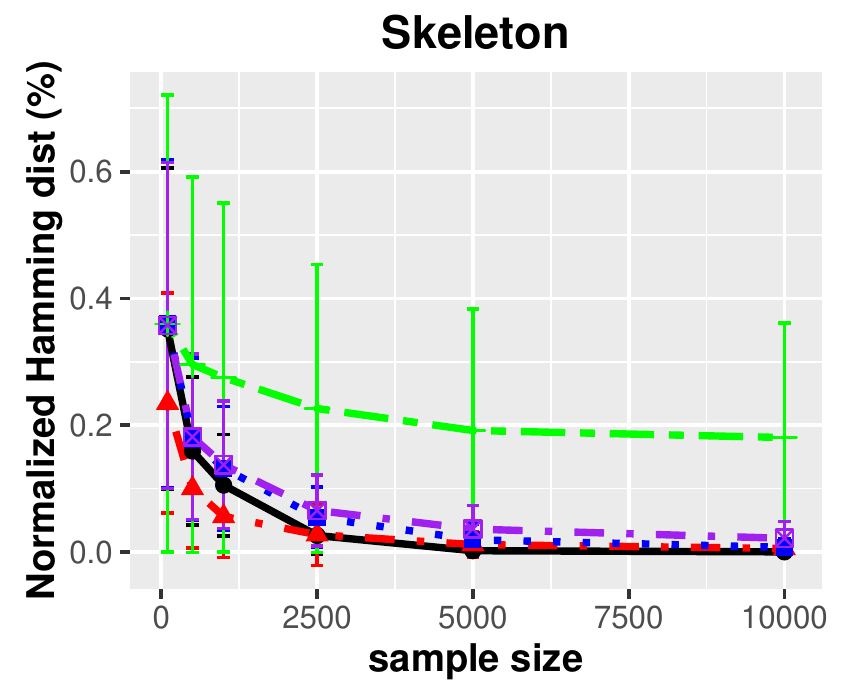}
		\caption{Poisson: $p=10$}
	\end{subfigure} \hspace{-2mm}
	\begin{subfigure}[!htb]{.23\textwidth}
		\includegraphics[width=\textwidth,height= 32mm]{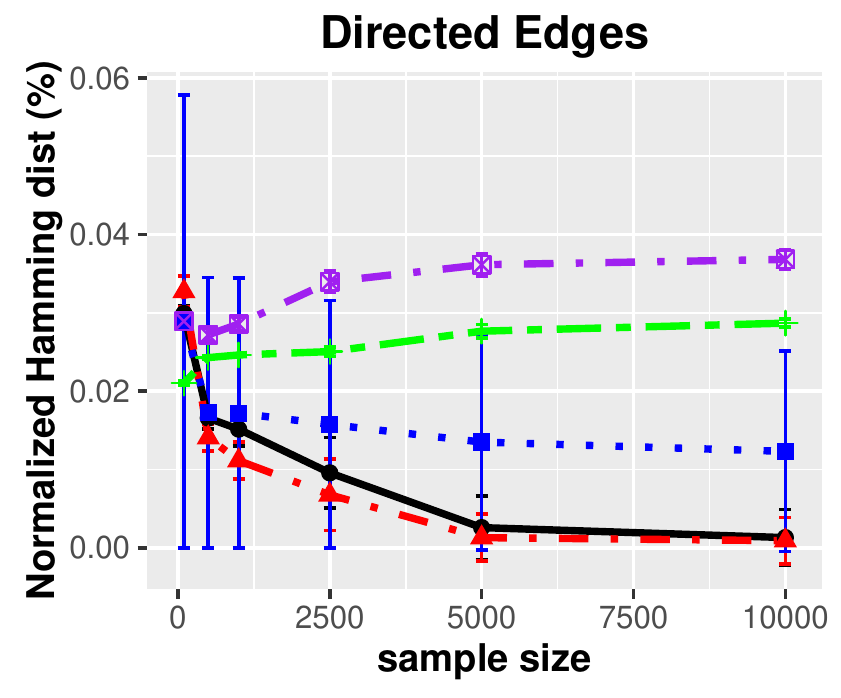}
		\caption{Poisson: $p=100$}
	\end{subfigure} \hspace{-2mm}
	\begin{subfigure}[!htb]{.23\textwidth}
		\includegraphics[width=\textwidth,height= 32mm]{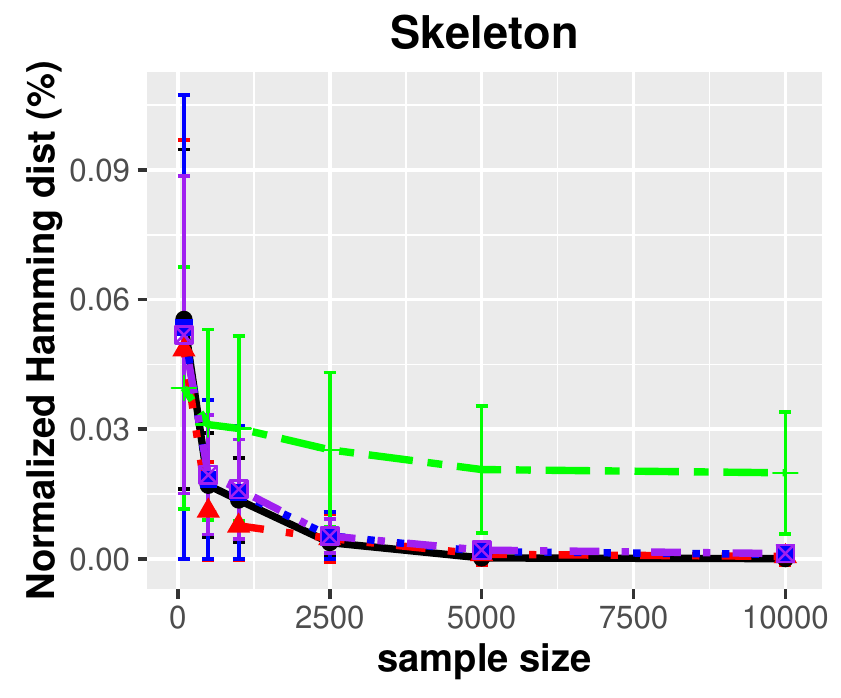}
		\caption{Poisson: $p=100$}
	\end{subfigure} \hspace{-2mm}
	\begin{subfigure}[!htb]{.08\textwidth}
		\includegraphics[width = 1.0\textwidth, height = 32mm, trim = 32mm -27mm 5mm 5mm, clip]{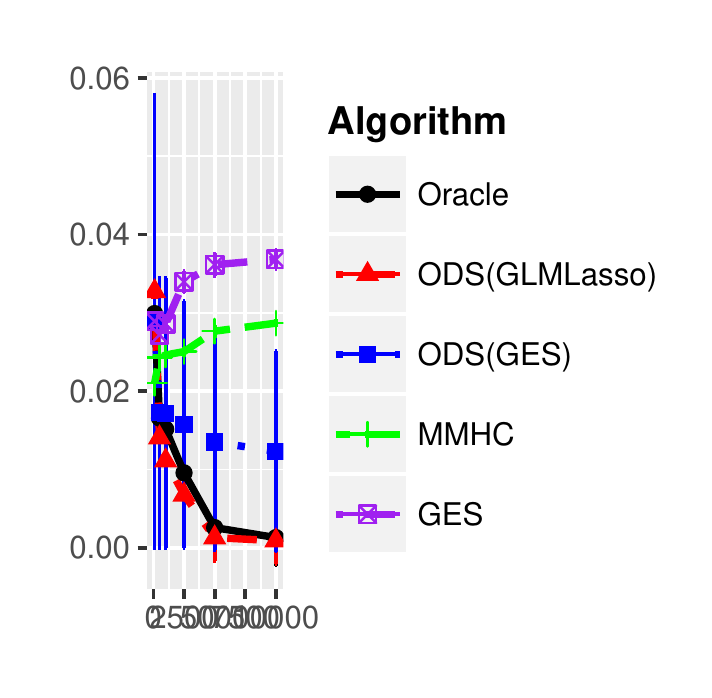}
	\end{subfigure}  
	
	\hspace{-8mm}
	\begin{subfigure}[!htb]{.23\textwidth} 
		\includegraphics[width=\textwidth,height= 32mm]{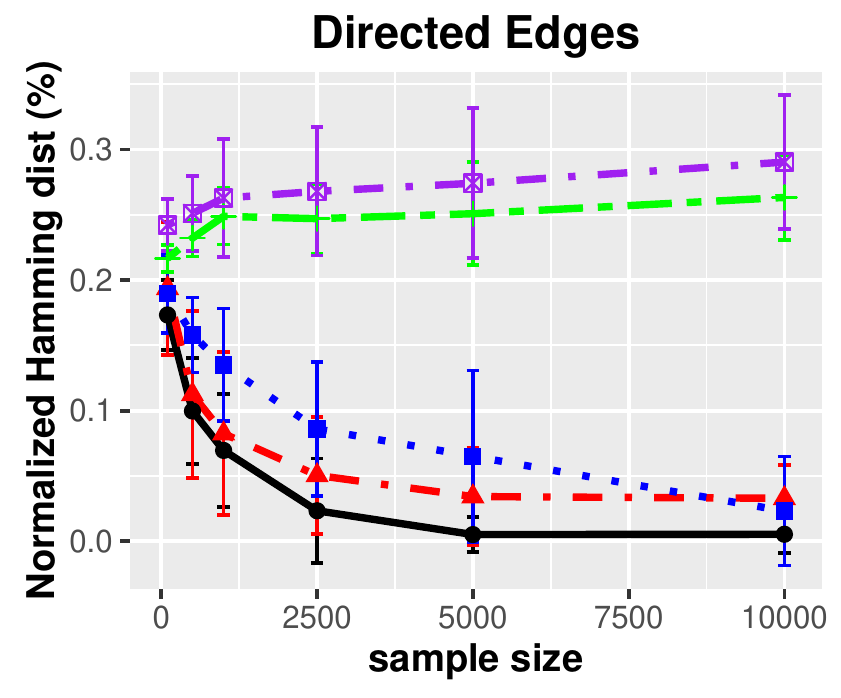}
		\caption{Binomial: $p=10$  }
	\end{subfigure} \hspace{-2mm}
	\begin{subfigure}[!htb]{.23\textwidth}
		\includegraphics[width=\textwidth,height= 32mm]{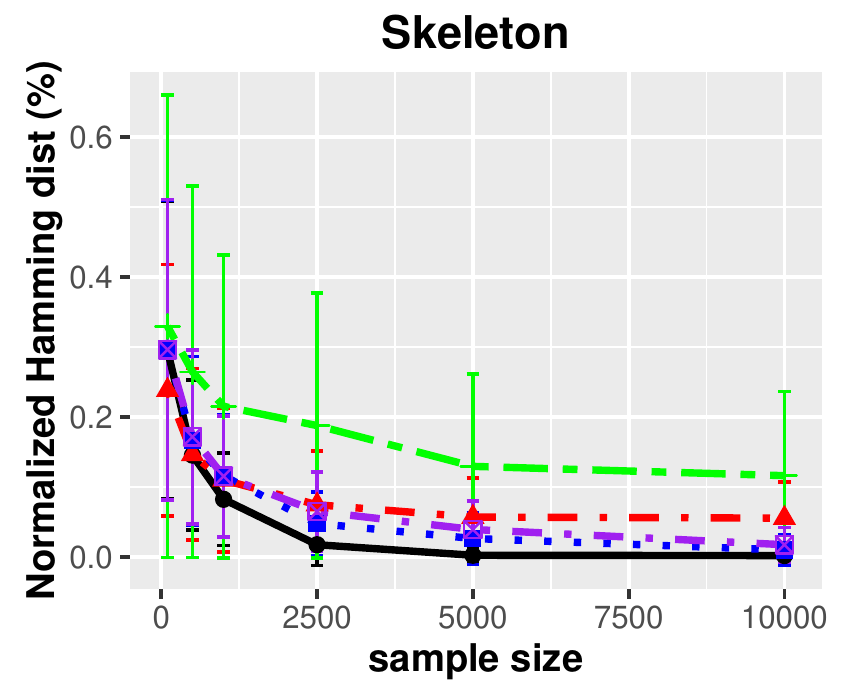}
		\caption{Binomial: $p=10$}
	\end{subfigure} \hspace{-2mm}
	\begin{subfigure}[!htb]{.23\textwidth}
		\includegraphics[width=\textwidth,height= 32mm]{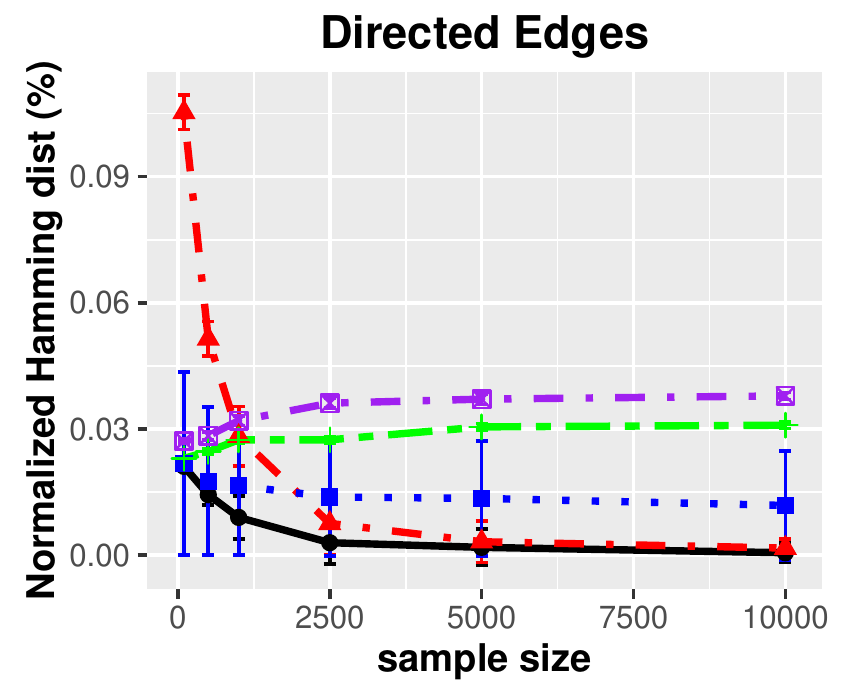}
		\caption{Binomial: $p=100$}
	\end{subfigure} \hspace{-2mm}
	\begin{subfigure}[!htb]{.23\textwidth}
		\includegraphics[width=\textwidth,height= 32mm]{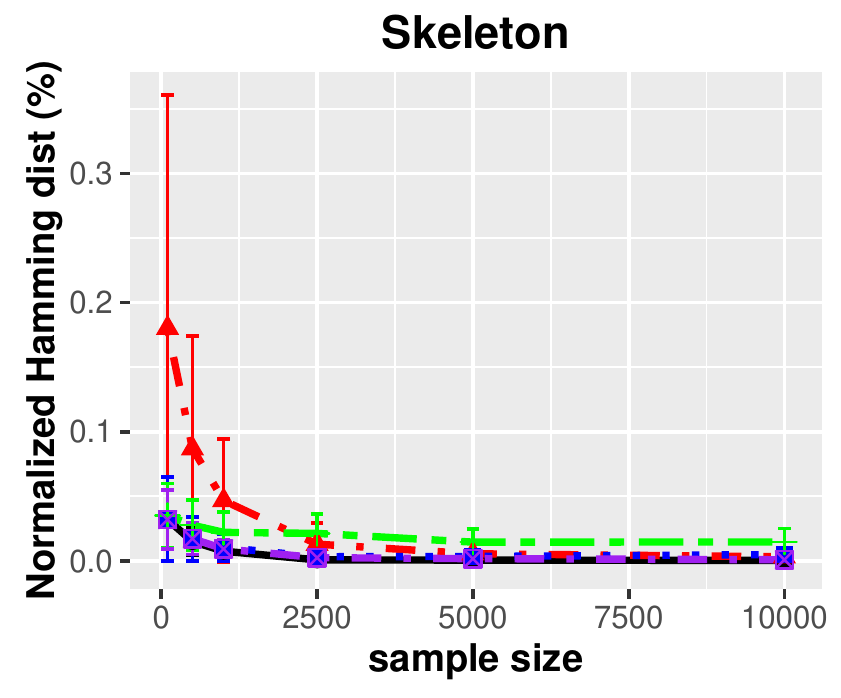}
		\caption{Binomial: $p=100$}
	\end{subfigure} \hspace{-2mm}
	\begin{subfigure}[!htb]{.08\textwidth}
		\includegraphics[width = 1.0\textwidth, height = 32mm, trim = 32mm -27mm 5mm 5mm, clip]{Rplot22}
	\end{subfigure} 
	\caption{Comparison of the generalized ODS algorithms using $\ell_1$-penalized likelihood regression (in Steps 1) and 3)) and the GES algorithm (in Steps 1) and 3)) to two state-of-the-art DAG learning algorithms (the MMHC and the GES algorithms) in terms of Hamming distance to skeletons and directed edges of Poisson and Binomial DAG models.} 
	\label{fig:Num2} \vspace{-2mm} 
\end{figure}

Next we consider the performance for large-scale DAG models to show that the ODS algorithm works in the high-dimensional setting. In all experiments, we used the $\ell_1$-penalized likelihood regression for GLMs in Steps 1) and 3) for the generalized ODS algorithm since it is the only graph-learning algorithm that scales. Figure~\ref{fig:a5} plots the statistical performance of the generalized ODS algorithm for large-scale Poisson DAGs in (a), (b), and (c) and Binomial DAGs in (d), (e), and (f). Furthermore, (a) and (d) represent the accuracy rates of the recovering the causal ordering, (b) and (e) show the normalized Hamming distance to the true skeleton, and (c) and (f) show the normalized Hamming distance for the true edge set of the DAG. Accuracies vary as a function of sample size ($n \in \{500, 1000, 2500, 5000, 10000\}$) for each node size ($p = \{1000, 2500, 5000\}$). Similar to small-scale DAG models, Figure~\ref{fig:a5} shows that the generalized ODS algorithm recovers the causal ordering and the skeleton of the DAG in the high-dimensional settings.

\begin{figure}[!t]
	\centering 
	\begin{subfigure}[!htb]{.30\textwidth} \hspace{-5mm}
		\includegraphics[width=\textwidth,height= 35mm]{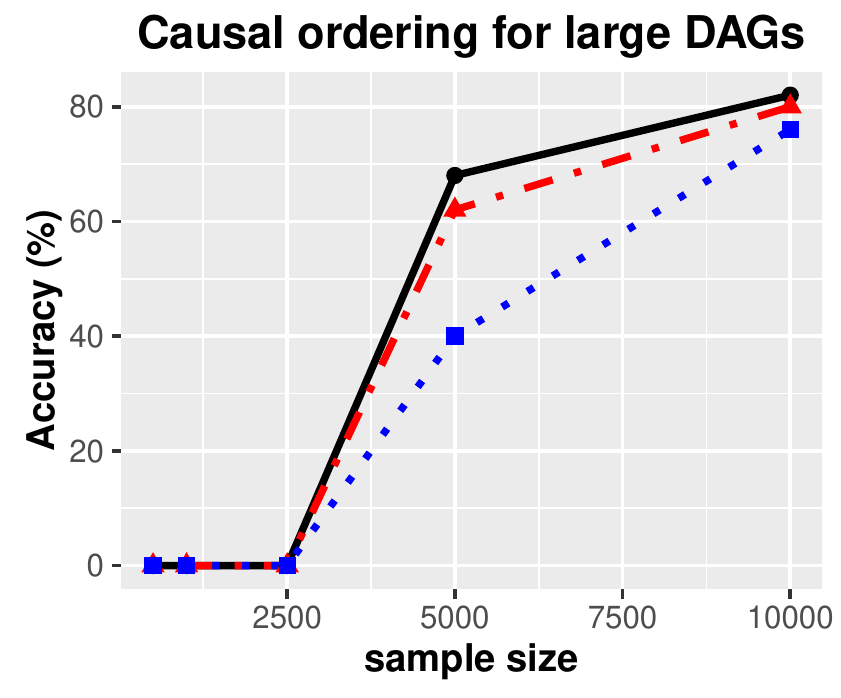}
		\caption{Poisson}
	\end{subfigure}
	\begin{subfigure}[!htb]{.30\textwidth}
		\includegraphics[width=\textwidth,height= 35mm]{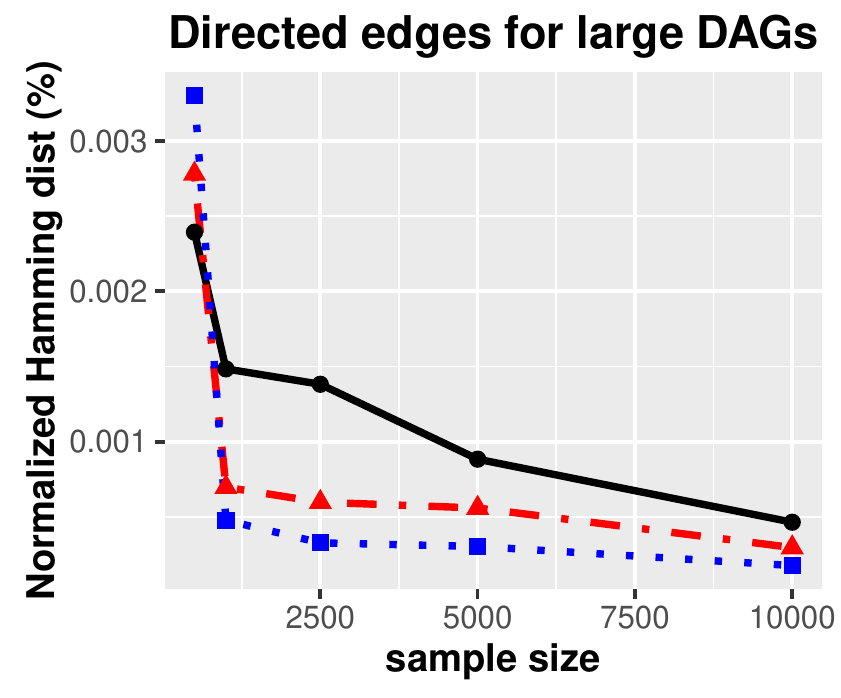}
		\caption{Poisson}
	\end{subfigure} \hspace{-2mm}
	\begin{subfigure}[!htb]{.30\textwidth}
		\includegraphics[width=\textwidth,height= 35mm]{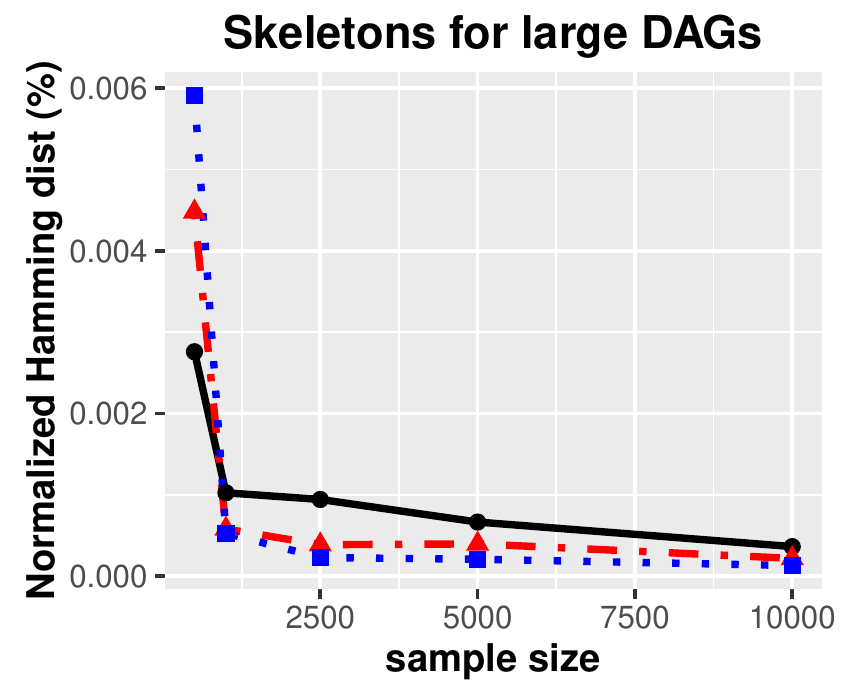}
		\caption{Poisson}
	\end{subfigure} \hspace{-2mm}
	\begin{subfigure}[!htb]{.08\textwidth}
		\includegraphics[width = 1.0\textwidth, height = 28mm]{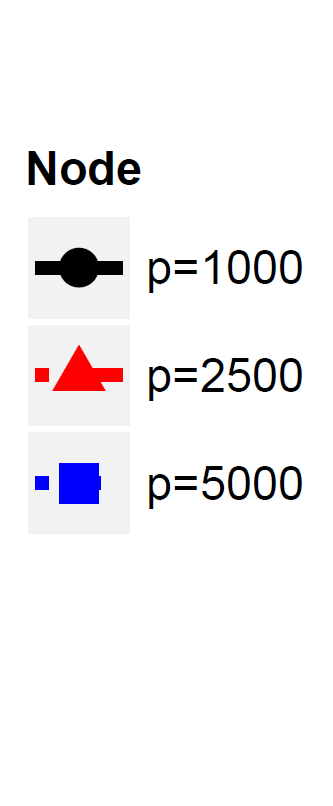}
	\end{subfigure} \\
	\begin{subfigure}[!htb]{.30\textwidth}  \hspace{-5mm}
		\includegraphics[width=\textwidth,height= 35mm]{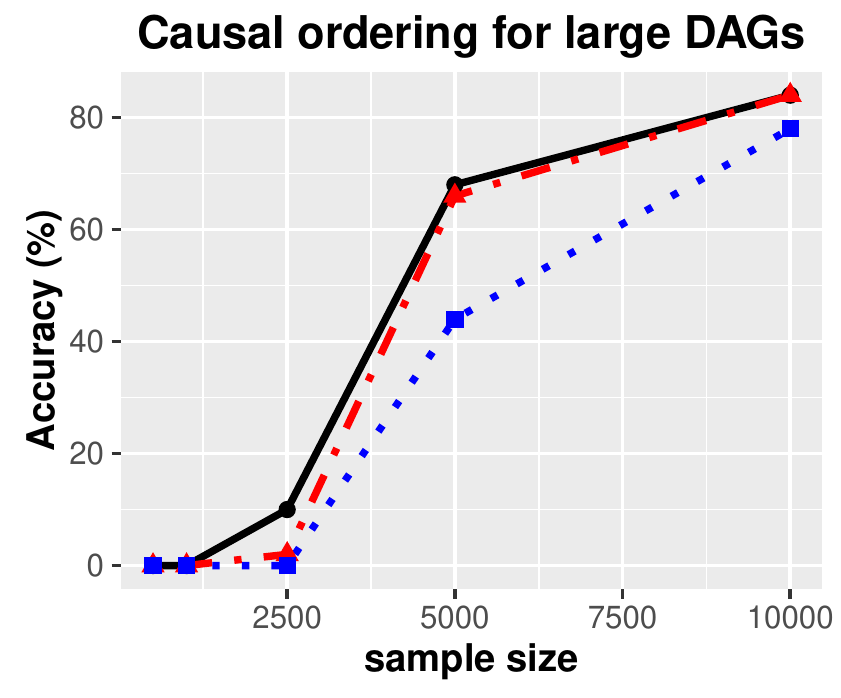}
		\caption{Binomial}
	\end{subfigure}
	\begin{subfigure}[!htb]{.30\textwidth}  \hspace{-2mm}
		\includegraphics[width=\textwidth,height= 35mm]{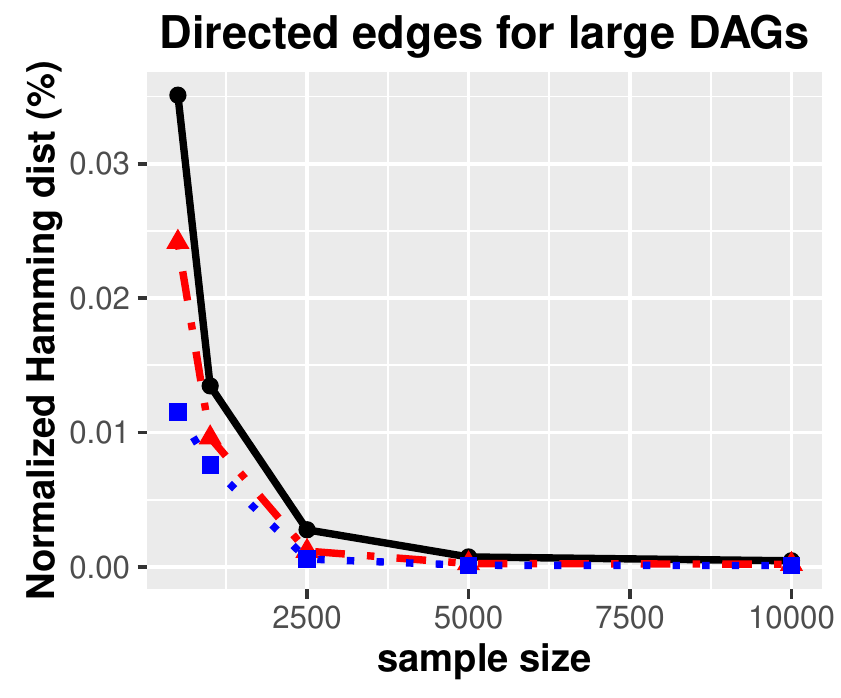}
		\caption{Binomial}
	\end{subfigure} \hspace{-2mm}
	\begin{subfigure}[!htb]{.30\textwidth}
		\includegraphics[width=\textwidth,height= 35mm]{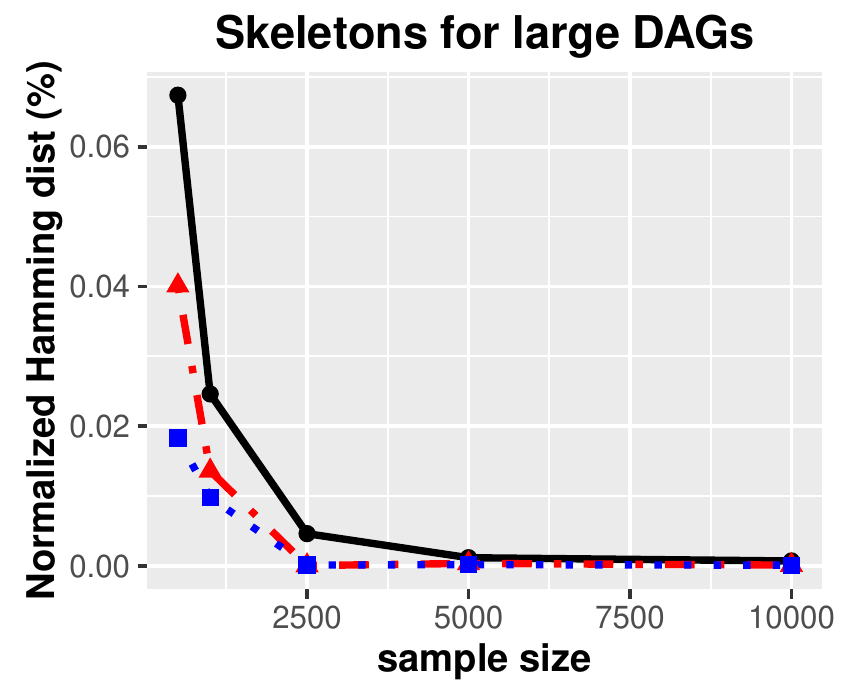}
		\caption{Binomial}
	\end{subfigure} \hspace{-2mm}
	\begin{subfigure}[!htb]{.08\textwidth}
		\includegraphics[width = 1.0\textwidth, height = 28mm]{Legend2.png}
	\end{subfigure} \hspace{-2mm}
	\caption{Performance of the generalized ODS algorithm using $\ell_1$-penalized likelihood regression in both Steps 1) and 3) for large-scale DAG models with the node size $p = \{1000, 2500, 5000\}$}
	\label{fig:a5} \vspace{-2mm} 
\end{figure}

In Figure~\ref{fig:a6}, we compared the run-time of the generalized ODS algorithms using $\ell_1$-penalized likelihood regression for GLMs in Steps 1) and 3) to the run-time of the MMHC and the GES algorithms. We measured the run-time for Poisson DAG models by varying (a) node size $p \in \{10, 20, 40, 60, 80, 100\}$ with fixed sample size $n = 10000$ and exactly two parents of each node, (b) sample size $n \in \{100, 500, 1000, 2500, 5000, 10000\}$ with the fixed node size $p= 100$ and two parents of each node, and (c) the number of parents of each node $ \in \{1,2,3,4,5,6\}$ with the fixed sample size $n = 10000$ and node size $p = 20$. The results of (a) and (b) show that the generalized ODS algorithm is not always slower than the GES algorithm. In addition, (c) also shows that the run-time of the generalized ODS algorithm depends significantly on the number of parents for each node. Figure~\ref{fig:a6} shows that the generalized ODS algorithm is significantly slower than the MMHC algorithm, however this is because the MMHC algorithm often stops earlier before they reach the true DAG (see Figure~\ref{fig:Num2}).

\begin{figure}[!t]
	\centering \hspace{-5mm}
	\begin{subfigure}[!htb]{.28\textwidth}
		\includegraphics[width=\textwidth,height= 35mm]{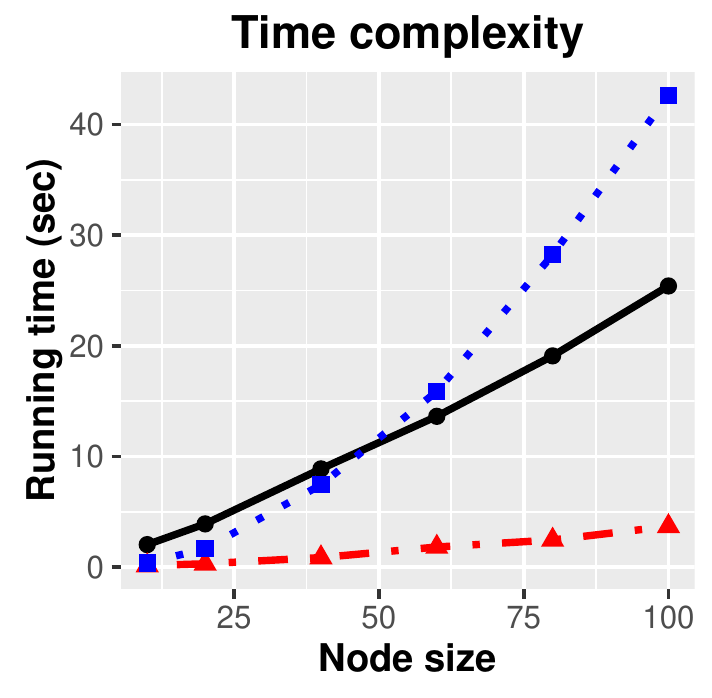}
		\caption{Poisson: $n = 10000, \d \geq 3$}
	\end{subfigure} \hspace{-2mm}
	\begin{subfigure}[!htb]{.28\textwidth}
		\includegraphics[width=\textwidth,height= 35mm]{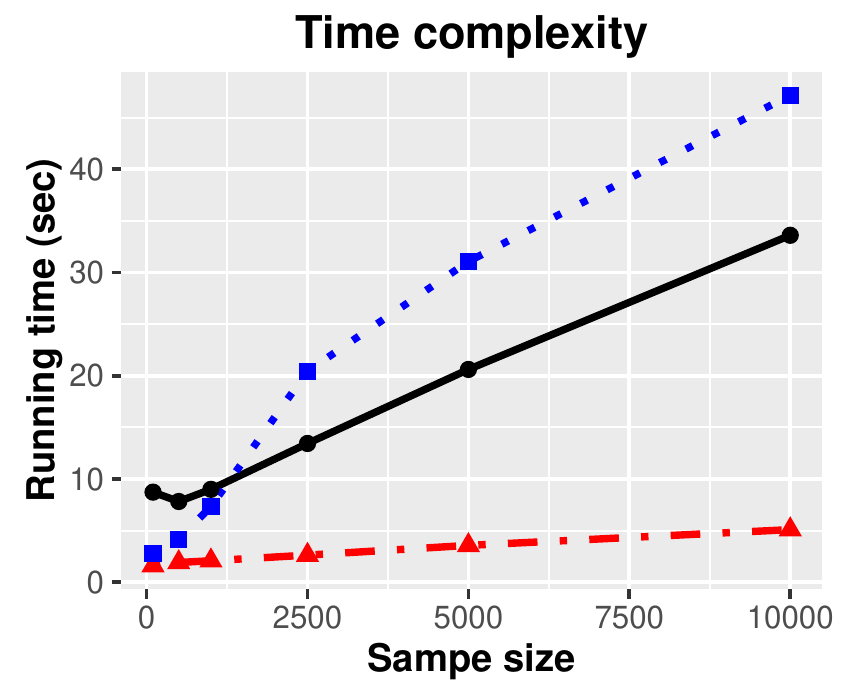}
		\caption{Poisson: $p = 100, \d \geq 3$}
	\end{subfigure} \hspace{-2mm}
	\begin{subfigure}[!htb]{.28\textwidth}
		\includegraphics[width=\textwidth,height= 35mm]{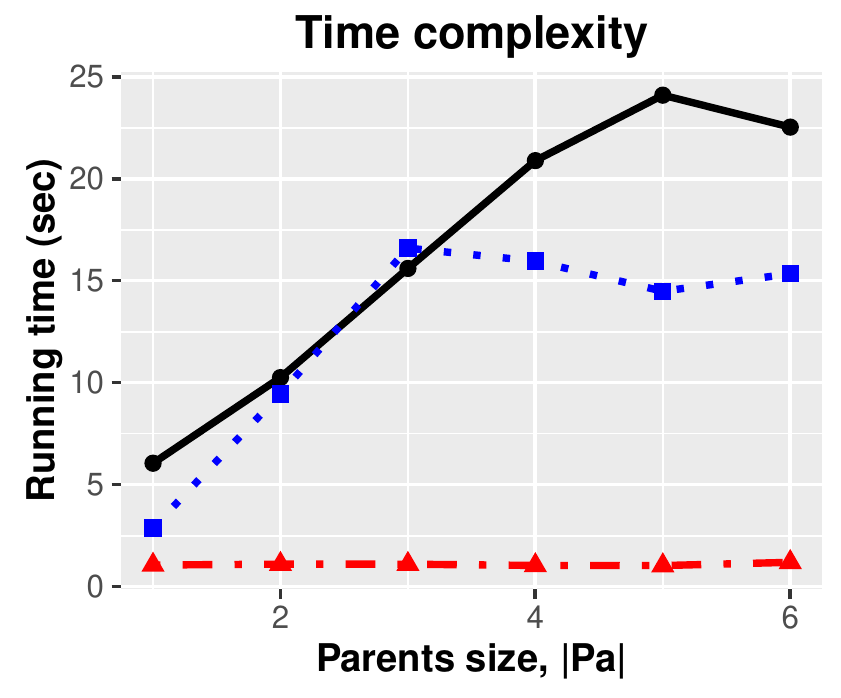}
		\caption{Poisson: $n = 10000, p = 20$}
	\end{subfigure} \hspace{-2mm}
	\begin{subfigure}[!htb]{.12\textwidth}
		\includegraphics[width = 1.0\textwidth, height = 42mm, trim = 32mm -15mm 5mm 5mm, clip]{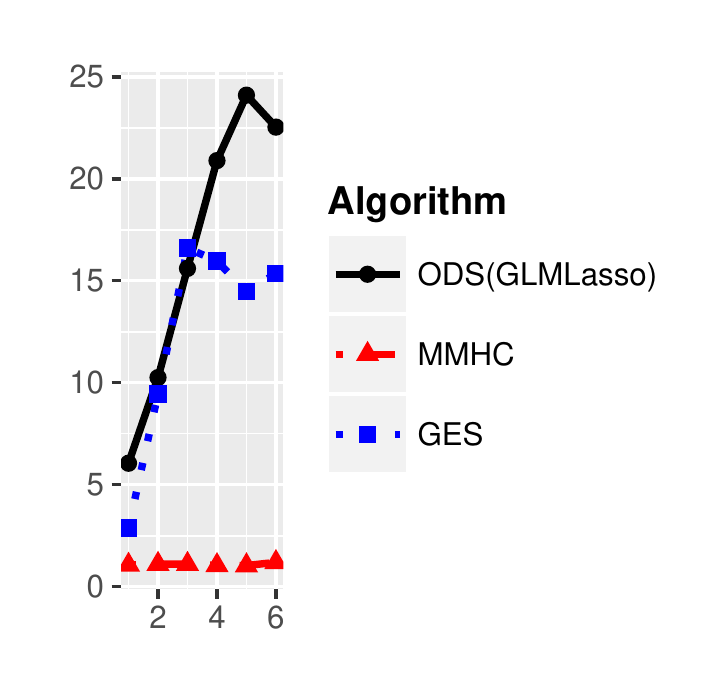}
	\end{subfigure}
	\caption{Comparison of the generalized ODS algorithms using $\ell_1$-penalized likelihood regression in Steps 1) and 3) to two standard DAG learning algorithms (the MMHC and the GES algorithms) in terms of running time with respect to (a) node size $p$, (b) sample size $n$, and (c) number of parents of each node}
	\label{fig:a6}
\end{figure}

\section*{Acknowledgement}
GP and GR were both supported by NSF DMS-1407028 over the duration of this project.

\clearpage

\bibliographystyle{IEEEtran}

\bibliography{reference}

\clearpage

\section{Appendix}

\subsection{Proof for Theorem~\ref{Thmidentifiability}}

\label{SecSubThmIde}

\begin{proof}
	Without loss of generality, we assume the causal ordering is $\pi^* = (1,2,\cdots,p)$. For notational convenience, we define $X_{1:j} = \{X_1,X_2,\cdots,X_j\}$ and $X_{1:0} = \emptyset$. For $m \in V$ and $j \in \{m,m+1,\cdots,p\}$, let $\omega_{jm} = (\beta_0 + \beta_1 \E(X_j \mid X_{1:m-1}) )^{-1}$ and $\omega_{j1} = ( \beta_0 + \beta_1 \E(X_j) )^{-1}$. Recall that the overdispersion score of node $j$ for $m^{th}$ element of the causal ordering is~\eqref{EqnTruncScorej}: 
	\begin{align*}
		\mathcal{S}(j,m) = \omega_{jm}^2 \var(X_j \mid X_{1:m-1} ) -\omega_{jm} \E(X_j \mid X_{1:m-1} ).
	\end{align*}
	
	We now prove identifiability of our class of DAG models by induction. For the first element of the causal ordering ($m = 1$), 
	\begin{align*}
	& \mathcal{S}(j,1) = \omega_{j1}^2 \var(X_j) -\omega_{j1} \E(X_j) \\
	& \stackrel{(a)}{=} \omega_{j1}^2 \big\{ \var( \E(X_j \mid X_{\pa(j)}) ) + \E( \var(X_j \mid X_{\pa(j)}) ) - \omega_{j1}^{-1} \E(X_j) \big\} \\
	& \stackrel{(b)}{=} \omega_{j1}^2 \big\{ \var( \E(X_j \mid X_{pa(j)}) ) + \E( \beta_0 \E(X_j \mid X_{\pa(j)}) + \beta_1 \E(X_j \mid X_{\pa(j)} )^2) - ( \beta_0 + \beta_1 \E(X_j) ) \E(X_j) \big\} \\
	& = \omega_{j1}^2 \big\{ \var( \E(X_j \mid X_{pa(j)}) ) + \beta_1 \E( \E(X_j \mid X_{\pa(j)} )^2 ) - \beta_1 \E(X_j)^2 \big\} \\		
	& = \omega_{j1}^2 (1 + \beta_1) \var( \E(X_j \mid X_{pa(j)}) ). 	
	\end{align*}
	$(a)$ follows from the variance decomposition formula $\var(Y)=\E(\var(Y \mid X))+\var(\E(Y \mid X))$ for some random variables $X$ and $Y$. In addition $(b)$ follows from the quadratic variance property~\eqref{eq:Quad} of our class of distributions and the definition of $\omega_{j1}$. Note that the score of the first element of the causal ordering is $\mathcal{S}(1,1) = 0$ because $\var(E(X_1)) = 0$, and other scores are strictly positive $\mathcal{S}(j,1) > 0$ by the assumption $\beta_{1} > -1$ . Therefore $1$ is the first element of the causal ordering. 
	
	For the $(m-1)^{st}$ element of the causal ordering, assume that the first $m-1$ elements of the causal ordering are correctly estimated. Now, we consider the $m^{th}$ element of the causal ordering. Then, for $j \in \{ m, m+1,\cdots,p\}$,
	\begin{align*}
	& \mathcal{S}(j,m)
	= \omega_{jm}^2 \var(X_j \mid X_{1:m-1} ) -\omega_{jm} \E(X_j \mid X_{1:m-1} ) \\
	& \stackrel{(a)}{=} \omega_{jm}^2 \big\{ \var( \E(X_j \mid X_{\pa(j)}) \mid X_{1:m-1}  ) + \E( \var(X_j \mid X_{\pa(j)}) \mid X_{1:m-1} ) - \omega_{jm}^{-1} \E(X_j \mid X_{1:m-1}) \big\} \\
	& \stackrel{(b)}{=} \omega_{jm}^2 \big\{ \var( \E(X_j \mid X_{pa(j)}) \mid X_{1:m-1} ) + \E( \beta_0 \E(X_j \mid X_{\pa(j)} \mid X_{1:m-1} ) + \beta_1 \E(X_j \mid X_{\pa(j)} \mid X_{1:m-1} )^2) \\
	& \hspace{2cm} - ( \beta_0 + \beta_1 \E(X_j \mid X_{1:m-1} ) ) \E(X_j \mid X_{1:m-1} ) \big\} \\
	& = \omega_{jm}^2 \big\{ \var( \E(X_j \mid X_{pa(j)}) \mid X_{1:m-1} ) + \beta_1 \E( \E(X_j \mid X_{\pa(j)} )^2 \mid X_{1:m-1} ) - \beta_1 \E(X_j \mid X_{1:m-1} )^2 \big\} \\		
	& = \omega_{jm}^2 (1 + \beta_1) \var( \E(X_j \mid X_{pa(j)}) \mid X_{1:m-1} ). 	
	\end{align*}	
	Again $(a)$ follows from the variance decomposition formula, and $(b)$ follows from the quadratic variance property~\eqref{eq:Quad} of our class of distributions and the definition of $\omega_{jm}$. If $\pa(j) \setminus \{1,2,\cdots,m-1\}$ is empty, $\var( \E(X_j \mid X_{pa(j)} ) \mid X_{1:m-1} ) = 0$, and hence $\mathcal{S}(m,m) = 0$. On the other hand, for any node $j$ in which $\pa(j) \setminus \{1,2,\cdots,m-1\}$ is non-empty, $\mathcal{S}(j,m) > 0$ by the assumption $\beta_{j1} > -1$, which excludes it from being next in the causal ordering. Therefore, we can estimate a valid $m^{th}$ component of the causal ordering, $\widehat{\pi}_m = m$. By induction this completes the proof. 
\end{proof}

\subsection{Proof for Lemma~\ref{Lem:RestFaithfulness}}

\label{Proof for Lemma 3.1}

\begin{proof}
	We begin with part (a).  By the construction $\theta_{D_j}^*$ in~\eqref{ThetaD}, $[\theta_{D_j}^*]_k = 0$ for any node $k \notin \pa(j)$. Hence, it is sufficient to show that for any $k \in \pa(j)$, $[\theta_{D_j}^*]_k \neq 0$. Assume for the sake of contradiction that $[\theta_{D_j}^*]_k = 0$. Applying the first order optimality condition to Equation~\eqref{ThetaD}, we have 
	\begin{eqnarray}
	\label{eq:lemma}
	\E(X_j) & = & \E( A_j'( [\theta_{D_j}^*]_{j} + \langle [\theta_{D_j}^*]_{\pa(j)}, X_{\pa(j)} \rangle ) ) \\
	\E(X_j X_k) & = & \E( A_j'( [\theta_{D_j}^*]_{j} + \langle [\theta_{D_j}^*]_{\pa(j)}, X_{\pa(j)} \rangle ) X_k ). \nonumber 
	\end{eqnarray}
	
	By the definition of the covariance, we obtain
	\begin{equation*}
	\E(X_j X_k) = \mbox{Cov}(A'( [\theta_{D_j}^*]_{j} + \langle [\theta_{D_j}^*]_{\pa(j)}, X_{\pa(j)} \rangle ), X_k) +  \E(A'( [\theta_{D_j}^*]_{j} + \langle [\theta_{D_j}^*]_{\pa(j)}, X_{\pa(j)} \rangle ))\E( X_k ) . 
	\end{equation*}
	
	By Equation~\eqref{eq:lemma},
	\begin{equation*}
	\E(X_j X_k) = \mbox{Cov}(A'( [\theta_{D_j}^*]_{j} + \langle [\theta_{D_j}^*]_{\pa(j)}, X_{\pa(j)} \rangle ), X_k ) + \E(X_j) \E( X_k ). 
	\end{equation*}
	
Therefore:
	\begin{equation*}
	\mbox{Cov} (X_j, X_k) = \mbox{Cov}(A'([\theta_{D_j}^*]_{j} + \langle [\theta_{D_j}^*]_{\pa(j)}, X_{\pa(j)} \rangle ), X_k).
	\end{equation*}
	
	By Assumption~\ref{Ass:RestFaithfulness} (a), we have $[\theta_{D_j}^*]_k = 0$, and
	\begin{equation*}
	\mbox{Cov} (X_j, X_k) = \mbox{Cov}(D'([\theta_{D_j}^*]_{j} + \langle [\theta_{D_j}^*]_{\pa(j) \setminus k}, X_{\pa(j) \setminus j} \rangle ), X_k),
	\end{equation*}
	which is a contradiction by our earlier assumption. Therefore $[\theta_{D_j}^*]_k \neq 0$. Furthermore since $k \in \pa(j)$ is arbitrary, the proof is complete. The proof for part (b) follows exactly the same line of reasoning.
\end{proof}

\subsection{Proof for Theorem~\ref{ThmMoralGraph}}

\label{SecThmStep1Proof}

In this section, we provide the proof for Theorem~\ref{ThmMoralGraph} using the primal-dual witness method that also used many works (see e.g.,~\cite{Yang2012, meinshausen2006high, Wainwright2006, Ravikumar2011}). We begin by introducing propositions to control the tail behavior for the distribution of each node:

\begin{proposition}
	\label{Prop2}
	 Define 
	\begin{eqnarray*}
		\xi_1 & := & \{ \max_{j \in V} \max_{i \in \{1,\cdots,n\} } |X_j^{(i)}| < 4 \log( \eta ) \}.
	\end{eqnarray*}
 Under Assumption~\ref{A3Con}, $P (\xi_1^c ) \leq M \cdot \eta^{-2}$.
\end{proposition}

\begin{proposition}
	\label{Prop1}
	Suppose that $X$ is a random vector according to the DAG model~\eqref{EqnFactorization}, and Assumption~\ref{A3Con} is satisfied. Then, for any vector $u \in \mathbb{R}^p$ such that $\|u\|_1 \leq c'$, for any positive constant $\delta$, 
	\begin{equation}
	P( | \langle u, X \rangle | \geq \delta \log \eta) \leq M \cdot p \cdot \eta^{ -\delta/ c' }.
	\end{equation}
\end{proposition}

Using these concentration results, we show that $\ell_1$-penalized regression recovers the neighborhood for a fixed node $j \in V$ with high probability. For ease of notation, we define a new parameter $\theta \in \mathbb{R}^{p-1}$ without the node $j$ since the node $j$ is not penalized in regression problem~\eqref{P1}. Then, the conditional negative log-likelihood of the GLM~\eqref{Eq1} is:
\begin{equation*}
\ell_j( \theta; X^{1:n} ) := \frac{1}{n} \sum_{i = 1}^{n} \left( -X_j^{(i)} \langle \theta, X_{V \setminus j}^{(i)} \rangle + A_j( \langle \theta, X_{V \setminus j}^{(i)} \rangle )  \right).
\end{equation*}

The main goal of the proof is to find the unique minimizer of the following convex problem:
\begin{equation}
	\label{eq:Objective}
	\widehat{\theta}_{M_j} := \arg \min_{\theta \in \mathbb{R}^{p-1}  } \mathcal{L}_j( \theta, \lambda_n) = \arg \min_{\theta \in \mathbb{R}^{p-1}  } \{ \ell_j (\theta ; X^{1:n}) + \lambda_n \| \theta \|_1 \}.
\end{equation}

By setting the \emph{sub-differential} to $0$, $\widehat{\theta}_{M_j}$ must satsify the following condition: 
\begin{equation}
\label{eq:Contraint1}
\bigtriangledown_\theta \mathcal{L}_j( \widehat{\theta}_{M_j}, \lambda_n ) = \bigtriangledown_\theta \ell_j( \widehat{\theta}_{M_j} ; X^{1:n}) + \lambda_n \widehat{Z}  = 0
\end{equation}
where $\widehat{Z} \in \mathbb{R}^{p-1}$ and $\widehat{Z}_t = \mbox{sign}([\widehat{\theta}_{M_j}]_{t})$ if  $t \in \mathcal{N}(j)$, otherwise $|\widehat{Z}_{t}| < 1$. 

The following Lemma~\ref{lemma: uniq} directly follows from prior works in Ravikumar et al.~\cite{Ravikumar2010} and Yang et al.~\cite{Yang2012} where each node conditional distribution is in the form of a generalized linear model. For notational convenience, let $\S = \mathcal{N}(j)$.

\begin{lemma}
	\label{lemma: uniq}
Suppose that $|\widehat{Z}_{t}| < 1$ for $t \notin S$. Then, the solution $\widehat{\theta}_{M_j}$ of ~\eqref{eq:Objective} satisfies $[\widehat{\theta}_{M_j}]_{t} = 0$ for $t \notin S$. Furthermore, if the sub-matrix of the Hessian matrix $Q_{SS}^{M_j}$ is invertible, then $\widehat{\theta}_{M_j}$ is unique. 
\end{lemma}

The remainder of the proof is to show $| \widetilde{Z}_{t} | < 1$ for all $t \notin \S$. Note that the restricted solution in \eqref{eq:ResObjective} is $(\widetilde{\theta}_{M_j}, \widetilde{Z})$. Equation~\eqref{eq:Contraint1} with the dual solution can be represented by $$\bigtriangledown^2 \ell_j( \theta_{M_j}^*;X^{1:n})( \widetilde{\theta}_{M_j} - \theta_{M_j}^* ) = -\lambda_n \widetilde{Z} - W_j^n + R_j^n$$ where: 
\begin{itemize}
	\item[(a)] $W_{j}^n$ is the sample score function. 
	\begin{equation}
	\label{eq:Wn}
		W_{j}^{n} := - \bigtriangledown \ell_j(\theta_{M_j}^*;X^{1:n}).
	\end{equation}
	\item[(b)] $R_{j}^{n} = (R_{j1}^n, R_{j2}^n,\cdots,R_{j p-1}^n)$ and $R_k^n$ is the remainder term by applying the coordinate-wise mean value theorem.  
	\begin{equation}
	\label{eq:Rn}
		R_{jk}^n := [ \bigtriangledown^2 \ell_j(\theta_{M_j}^*;X^{1:n}) - \bigtriangledown^2 \ell_j(\bar{\theta}_{M_j}^{(k)};X^{1:n})]_k^T (\widetilde{\theta}_{M_j}^{(k)} - \theta_{M_j}^*).
	\end{equation}
	Here $\bar{\theta}_{M_j}^{(k)}$ is a vector on the line between $\widetilde{\theta}$ and $\theta_{M_j}^*$ and $[\cdot]_k^T$ is the $k^{th}$ row of a matrix. 
\end{itemize}

Then, the following proposition provides a sufficient condition to control $\widetilde{Z}$.
\begin{proposition}
\label{prop: block}
Suppose that $\max(\| W_j^n \|_\infty, \| R_j^n \|_\infty)  \leq \frac{\lambda_n \alpha}{4(2- \alpha)}$. Then $| \widetilde{Z}_{t} | < 1$ for all $t \notin \S$.
\end{proposition}

Next we introduce the following three lemmas to show that conditions in Proposition~\ref{prop: block} hold. For ease of notation, let $\eta = \max\{n,p\}$ and $\widetilde{\theta}_{\S} = [\widetilde{\theta}_{M_j}]_{\S}$ and $\widetilde{\theta}_{\S^c} = [\widetilde{\theta}_{M_j}]_{\S^c}$. Suppose that Assumptions~\ref{A1Dep},~\ref{A2Inc},~\ref{A3Con}, and~\ref{A4} are satisfied. 
\begin{lemma}
	\label{lem11}
	Suppose that $\lambda_n \geq \frac{16 \max\{ n^{\kappa_2} \log \eta , \log^2 \eta ) \}}{n^{a}}$ for some $a \in \mathbb{R}$. Then,  
	\begin{equation*}
		P\left( \frac{\| W_j^n \|_\infty }{\lambda_n} \leq \frac{\alpha}{4(2- \alpha)} \right) 
		\geq 1 -2 \d \cdot \exp(-\frac{\alpha^2}{8 (2- \alpha)^2} \cdot n^{ 1 - 2a }) - M \cdot \eta^{-2}.
	\end{equation*}
\end{lemma}

\begin{lemma} 
	\label{lem12}
	Suppose that $\|W_j^n\|_{\infty} \leq \frac{\lambda_n}{4}$. For $\lambda_n \leq \frac{1}{40} \frac{ \rho_{\min}^2 }{  \rho_{\max} } \frac{1}{n^{\kappa_2} \d \log\eta }$,
	\begin{equation*}
		P \left( \| \widetilde{\theta}_S - \theta_S^* \|_2 \leq \frac{5}{ \lambda_{\min} } \sqrt{\d} \lambda_n \right) \geq 1 - 2 M \cdot \eta^{-2}.
	\end{equation*}
\end{lemma}

\begin{lemma} 
	\label{lem13}
	Suppose that $\|W_j^n\|_{\infty} \leq \frac{\lambda_n}{4}$. For $\lambda_n \leq \frac{\alpha}{400(2-\alpha)}\frac{ \rho_{\min}^2 }{ \rho_{\max} } \frac{1}{n^{\kappa_2} \d \log \eta }$, 
	\begin{equation*}
	P \left( \frac{\| R_j^n \|_\infty }{\lambda_n} \leq \frac{\alpha}{4(2 - \alpha)} \right) \geq 1 - 2 M \cdot \eta^{-2}.
	\end{equation*}
\end{lemma}

The rest of the proof is straightforward using Lemmas~\ref{lem11},~\ref{lem12}, and~\ref{lem13}. Consider the choice of regularization parameter $\lambda_n = \frac{16 \max\{ n^{\kappa_2} \log \eta , \log^2 \eta  \}}{n^{a}}$ for a constant $a \in (2 \kappa_2, 1/2)$ where $\kappa_2$ is determined by Assumption~\ref{A4}. Then, the condition for Lemma~\ref{lem11} is satisfied, and therefore $\|W_n\|_{\infty} \leq \frac{\lambda_n}{4}$. Moreover, the conditions for Lemmas~\ref{lem12} and~\ref{lem13} are satisfied for $n \geq C' \max\{ ( \d \log^2 \eta )^{ \frac{1}{  a - 2 \kappa_2 } }, ( \d \log^3 \eta )^{ \frac{1}{  a - \kappa_2 } } \}$ for some positive constants $C'$. Then,
\begin{equation}
	\| \widetilde{Z}_{\S^c}\|_\infty \leq ( 1- \alpha ) + ( 2- \alpha) \left[ \frac{ \| W_{j}^{n} \|_\infty }{ \lambda_n} + \frac{ \| R_{j}^n \|_\infty }{ \lambda_n} \right] \leq ( 1- \alpha ) + \frac{ \alpha }{4} + \frac{ \alpha }{4} < 1,
\end{equation}
with probability of at least $1 - C_1 \d \exp( - C_2 n^{ 1 - 2a } )- C_3 \eta^{-2}$ for positive constants $C_1, C_2$ and $C_3$.  

To prove sign consistency, it is sufficient to show that $\|\widehat{\theta}_{M_j} - \theta_{M_j}^* \|_{\infty} \leq \frac{ \|\theta_{M_j}^*\|_{\min} }{2}$. By Lemma~\ref{lem12}, we have $\|\widehat{\theta}_{M_j} - \theta_{M_j}^* \|_{\infty} \leq \|\widehat{\theta}_{M_j} - \theta_{M_j}^* \|_{2} \leq \frac{5}{\lambda_{\min}} \sqrt{\d}~\lambda_n \leq \frac{\|\theta_{M_j}^*\|_{\min} }{2}$ as long as $\|\theta_{M_j}^*\|_{\min} \geq \frac{10}{\lambda_{\min}} \sqrt{\d}~\lambda_n$.

Lemma~\ref{Lem:RestFaithfulness}(b) guarantees that $\ell_1$-penalized likelihood regression recovers the true neighborhood for each node with high probability. Because we have $p$ likelihood regression problems, if $n \geq C' (\d \log^2 \eta)^{ \frac{1}{a -2 \kappa_2} } )$, it follows that:
\begin{equation}
	 P( \widehat{G^m} = G^m ) \geq 1 - C_1 \d \cdot p \cdot \exp( - C_2 n^{1- 2 a} ) - C_3 \eta^{-1}. 
\end{equation}

\subsubsection{Proof for Proposition~\ref{Prop2}}

\begin{proof}
	Applying the union bound and the Chernoff bound,
	\begin{equation*}
	P (\xi_1^c ) \leq n.p.\max_{j \in V} \max_{i \in \{1,\cdots,n\}} P \left( | X_j^{(i)} | > 4 \log \eta \right)
	\leq \eta^{-2} \max_{i, j}  \E[ \exp( |X_j^{(i)}| ) ].
	\end{equation*}
	By Assumption~\ref{A3Con}, we obtain $\max_{i,j} \E( \exp(|X_j|^{(i)}) ) < M$, which completes the proof. 
\end{proof}

\subsubsection{Proof for Proposition~\ref{Prop1}}

\begin{proof}
	We exploit H\"{o}lder's inequality $\langle u, X \rangle \leq \|u \|_1 \max_{j \in V}| X_j |$. Therefore, we have 
	\begin{equation*}
	P( | \langle u, X \rangle)| \geq \delta \log \eta ) \leq P ( \max_{j \in V}| X_j | \geq \frac{ \delta }{ \|u \|_1} \log \eta ). 
	\end{equation*}
	
	Using the union bound, we have 
	\begin{equation*}
	P ( \max_{j \in V}| X_j | \geq \frac{ \delta }{ \|u \|_1} \log \eta ) \leq   p \cdot  \max_{j \in V} P ( | X_j | \geq \frac{ \delta }{ \|u \|_1} \log \eta ).
	\end{equation*}
	
	Applying the Chernoff bounding technique and Assumption~\ref{A3Con} $\max_{j} \E( \exp(|X_j| ) < M$, we obtain 
	\begin{equation*}
	p \cdot \max_{j \in V} P ( | X_j | \geq \frac{ \delta }{ \|u \|_1} \log \eta ) \leq M \cdot p \cdot \eta^{ -\frac{ \delta }{ \|u \|_1} }.
	\end{equation*}
	
	By the assumption $\|u\|_1 \leq c'$, we compete the proof. 
\end{proof}

\subsubsection{Proof for Proposition~\ref{prop: block}}

\begin{proof}
	Since $\widetilde{\theta}_{\S^c} = (0,0,...,0) \in \mathbb{R}^{|S^c|}$ in our primal-dual construction, we can re-state condition ~\eqref{eq:Contraint1} in block form as follows. For notational simplicity, $Q := Q^{M_j}$.
	\begin{eqnarray*}
		Q_{\S^c \S}[ \widetilde{\theta}_{\S} - \theta_{\S}] & = & W_{\S^c}^n - \lambda_n \widetilde{Z}_{\S^c} + R_{\S^c}^{n},. \nonumber \\
		Q_{\S \S}[ \widetilde{\theta}_S - \theta_S^*] & = & W_{\S}^n - \lambda_n \widetilde{Z}_{\S} + R_{\S}^n,
	\end{eqnarray*}
	where $W_{S}^{n}$ and $R_{S}^{n}$ are sub-vectors of $W_{j}^{n}$ and $R_{j}^n$ indexed by $S$, respectively.
		
Since the matrix $Q_{\S \S}$ is invertible, the above equations can be rewritten as
	\begin{equation*}
	Q_{\S^c \S} Q_{\S \S}^{-1} [ W_{\S}^n - \lambda_n \widetilde{Z}_{\S} - R_{\S}^n] = W_{\S^c}^n - \lambda_n \widetilde{Z}_{\S^c} - R_{\S^c}^n.
	\end{equation*}
	
Therefore
	\begin{equation*}
	[W_{\S^c}^n - R_{\S^c}^n ] - Q_{\S^c \S} Q_{\S \S}^{-1} [ W_{\S}^n  - R_{\S}^n] + \lambda_n Q_{\S^c \S} Q_{\S \S}^{-1} \widetilde{Z}_{\S} =  \lambda_n \widetilde{Z}_{\S^c}.
	\end{equation*}
	
Taking the $\ell_\infty$ norm of both sides yields
	\begin{eqnarray*}
		\| \widetilde{Z}_{\S^c}\|_\infty & \leq & |\| Q_{\S^c \S} Q_{\S \S}^{-1} \||_{\infty} \left[ \frac{ \| W_{\S}^n \|_\infty }{ \lambda_n} + \frac{ \| R_{\S}^n \|_\infty }{ \lambda_n} +1 \right] + \frac{ \| W_{\S^c}^n \|_\infty }{ \lambda_n} + \frac{ \| R_{\S^c}^n \|_\infty }{ \lambda_n}.
	\end{eqnarray*}
	
	Recalling Assumption~\eqref{A2Inc}, we obtain $ |\| Q_{\S^c \S} Q_{\S \S}^{-1} \||_{\infty} \leq (1 - \alpha)$, hence we have 
	\begin{eqnarray*} 
		\| \widetilde{Z}_{\S^c}\|_\infty & \leq & (1- \alpha) \left[ \frac{ \| W_{\S}^n \|_\infty }{ \lambda_n} + \frac{ \| R_{\S}^n \|_\infty }{ \lambda_n} +1 \right] + \frac{ \| W_{\S^c}^n \|_\infty }{ \lambda_n} + \frac{ \| R_{\S^c}^n \|_\infty }{ \lambda_n} \\
		&\leq & ( 1- \alpha ) + ( 2- \alpha) \left[ \frac{ \| W_{j}^{n} \|_\infty }{ \lambda_n} + \frac{ \| R^n \|_\infty }{ \lambda_n} \right].
	\end{eqnarray*}
	
	If $\| W_j^n \|_\infty$ and $\| R_j^n \|_\infty  \leq \frac{\lambda_n \alpha}{4(2- \alpha)}$ as assumed, 
	$$
	\| \widetilde{Z}_{\S^c}\|_\infty \leq ( 1- \alpha ) + \frac{\alpha}{2} \leq 1.
	$$
\end{proof}

\subsubsection{Proof for Lemma~\ref{lemma: uniq}}

\begin{proof}
	The main idea of the proof is the \emph{primal-dual-witness} method which asserts that there is a solution to the dual problem $\widetilde{\theta}_{M_j} = \widehat{\theta}_{M_j}$ if the following KKT conditions are satisfied: 
	\begin{itemize}
		\item[(a)] We define $\widetilde{\theta}_{M_j} \in \Theta_{M_j}$ where $\Theta_{M_j} = \{ \theta \in \mathbb{R}^{p-1} : \theta_{\S^c} = 0 \}$ as the solution to the following optimization problem. 
		\begin{equation}
		\label{eq:ResObjective}
		\widetilde{\theta}_{M_j} := \arg \min_{\theta \in \Theta_{M_j}} \mathcal{L}( \theta, \lambda_n ) = \arg \min_{\theta \in \Theta_{M_j} } \{ \ell_j(\theta ;X^{1:n}) + \lambda_n \| \theta \|_1 \}.
		\end{equation}
		\item[(b)] Define $\widetilde{Z}$ to be a sub-differential for the regularizer $\| \cdot \|_1$ evaluated at $\widetilde{\theta}_{M_j}$. For any $t \in \S$, $\widetilde{Z}_{t} = \mbox{sign}([\widetilde{\theta}_{M_j}]_{t})$. 
		\item[(c)] For any $t \notin \S$, $|\widetilde{Z}_{t}| < 1$.
	\end{itemize}
	
	If conditions (a), (b), and (c) are satisfied, $\widehat{\theta}_{M_j} = \widetilde{\theta}_{M_j}$, meaning that the solution of the unrestricted problem \eqref{eq:Objective} is the same as the solution of the restricted problem \eqref{eq:ResObjective}. Conditions (a), (b) and (c) suffice to obtain a pair $(\widetilde{\theta}_{M_j}, \widetilde{Z})$ that satisfies the optimality condition \eqref{eq:Contraint1}, but do not guarantee that $\widetilde{Z}$ is an element of the sub-differential $\| \widetilde{\theta}_{M_j} \|_1$ (see details in~\cite{Ravikumar2010, Ravikumar2011}). Since the sub-matrix of the Hessian $Q_{SS}^{M_j}$ is invertible, the restricted problem \eqref{eq:ResObjective} is strictly convex, $\widetilde{\theta}_{M_j}$ is unique.
\end{proof}

\subsubsection{Proof for Lemma~\ref{lem11}}

\label{SubSecProofLem11}

\begin{proof}

Each entry of the sample score function $W_{j}^{n}$~\eqref{eq:Wn} has the form $W_{jt}^{n} = \frac{1}{n} \sum_{i=1}^n W_{jt}^{(i)}$ for any $t \in \S$. In addition, $W_{jt}^{n} = 0$ for all $t \notin \S$ since $[\theta_{M_j}^*]_t = 0$ by the construction of $\theta_{M_j}^*$(~\eqref{ThetaM}). For any  $t \in S$ and $i \in \{ 1,2,\cdots,n\}$, 
$W_{jt}^{(i)} = X_{t}^{(i)} X_j^{(i)} - A_{j}'( \langle \theta_{\S}^*, X_{\S}^{(i)} \rangle) X_t^{(i)}$ 
are independent and have mean $0$.

Now, we show that $(|W_{jt}^{(i)}|)_{i = 1}^{n}$ are bounded with high probability given the following event $\xi_1$ using Hoeffding's inequality. Event $\xi_1$ is defined as follows: 
\begin{eqnarray*}
	\xi_1 & := & \left\{ \max_{j \in V} \max_{i \in \{1,\cdots,n\} } |X_j^{(i)}| < 4 \log \eta \right\}.
\end{eqnarray*}

Conditioning on $\xi_1$, it follows that $\langle \theta_{\S}^*, X_{\S}^{(i)} \rangle < 4 \log(\eta) \cdot \|\theta_{\S}^*\|_1 $, Assumption~\ref{A4} is satisfied. Hence $\max_{i} |A_{j}'( \langle \theta_{\S}^*, X_{\S}^{(i)} \rangle)| \leq n^{\kappa_2}$. Furthermore given $\xi_1$, $\max_{i} X_{t}^{(i)} X_j^{(i)} < 16 \log^2 \eta$. Therefore there exists a constant $C_{\max}(\eta, \kappa_2) := 16 \max\{ n^{\kappa_2} \log \eta, \log^2\eta \}$ such that $\max_{i,j,t} |W_{jt}^{(i)}| \leq C_{\max}(\eta, \kappa_2)$. 

Recall that $\d$ is the maximum degree of the moralized graph,  therefore $|\S| \leq d$. Applying the union bound, 
\begin{equation*}
	P( \| W_{j}^{n} \|_\infty > \delta, \xi_1 ) 
	\leq \d \cdot \max_{t \in \S} P( | W_{jt}^n | > \delta, \xi_1).
\end{equation*}

Using Hoeffding's inequality,
\begin{equation*}
	\d \cdot \max_{t \in \S} P( | W_{jt}^n | > \delta, \xi_1)
	\leq 2 \d \cdot \exp ( - \frac{ 2 n \delta^2}{ C_{\max}(\eta, \kappa_2)^2 } ).
\end{equation*}

Suppose that $\delta = \frac{\lambda_n \alpha}{4(2- \alpha)}$ and $\lambda_n \geq \frac{ C_{\max}(\eta, \kappa_2) }{ n^{a}  } $ for some $a \in [0, 1/2)$. Then
\begin{align}
	\label{eq:probWn}
	P( \frac{\| W_{j}^{n} \|_\infty }{\lambda_n} > \frac{\alpha}{4(2- \alpha)} , \xi_1) 
	& \leq 2 \d \cdot \exp \Big(- \frac{\alpha^2}{8 (2- \alpha)^2} \frac{ n \lambda_n^2}{ C_{\max}(\eta, \kappa_2)^2 } \Big) \nonumber \\
	& \leq  2 \d \cdot \exp \Big(-\frac{\alpha^2}{8 (2- \alpha)^2} n^{ 1-2a } \Big).
\end{align}

Since $P(A) = P(A \cap B) + P(A \cap B^c) \leq P(A \cap B) + P(B^c)$, 
\begin{eqnarray*}
	P( \frac{\| W_{j}^{n} \|_\infty }{\lambda_n} > \frac{\alpha}{4(2- \alpha)} ) & \leq &P( \frac{\| W_{j}^{n} \|_\infty }{\lambda_n} > \frac{\alpha}{4(2- \alpha)}, \xi_1 ) + P(\xi_1^c).
\end{eqnarray*}

Then, the probability bound in~\eqref{eq:probWn} and Proposition~\ref{Prop2} $P(\xi_1^c) \leq M \cdot \eta^{-2}$ directly implies that 
\begin{eqnarray*}
	P( \frac{\| W_{j}^{n} \|_\infty }{\lambda_n} > \frac{\alpha}{4(2- \alpha)} ) \leq 2 \d \cdot \exp \Big(-\frac{\alpha^2}{8 (2- \alpha)^2} n^{ 1-2a } \Big) + M \cdot \eta^{-2}.
\end{eqnarray*}

\end{proof}

\subsubsection{Proof for Lemma~\ref{lem12}}

\label{SubSecProofLem12}

\begin{proof}

In order to establish the error bound $\| \widetilde{\theta}_{\S} - \theta_{\S}^*\| \leq B$ for some radius $B$, several works~\cite{Yang2012, Ravikumar2010, Ravikumar2011} already proved that it suffices to show $F(u_{\S}) > 0 $ for all $u_{\S}:= \widetilde{\theta}_{\S} - \theta_{\S}^*$ such that $\| u_{\S} \|_2 = B$ where 
\begin{equation}
\label{eq:F}
F( a ) := \ell_j( \theta_{\S}^* + a; X^{1:n}) - \ell_j( \theta_{\S}^*; X^{1:n}) + \lambda_n( \| \theta_{\S}^* + a \|_1 - \| \theta_{\S}^* \|_1 ).
\end{equation}

More specifically, since $u_{\S} = \widetilde{\theta}_{\S} - \theta_{\S}^*$ is the minimizer of $F$ and $F(0) = 0$ by the construction of~\eqref{eq:F}, $F(u_{\S}) \leq 0$. Note that $F$ is convex, and therefore we have $F(u_{\S}) < 0$. Next we claim that $\| u_{\S} \|_2 \leq B$. In fact, if $u_{\S}$ lies outside the ball of radius $B$, then the convex combination $v \cdot u_{\S} + (1-v) \cdot 0$ would lie on the boundary of the ball, for an appropriately chosen $v \in (0,1)$. By convexity, 
\begin{equation}
	F(v \cdot u_{\S} + (1-v) \cdot 0 ) \leq v \cdot F(u_{\S}) + (1-v) \cdot 0 \leq 0
\end{equation}
contradicting the assumed strict positivity of $F$ on the boundary. 
	
Thus it suffices to establish strict positivity of $F$ on the boundary of the ball with radius $B := M_1 \lambda_n \sqrt{\d}$ where $M_1 > 0$ is a parameter to be chosen later in the proof. Let $u_{\S} \in \mathbb{R}^{|\S|}$ be an arbitrary vector with $\|u_{\S}\|_2 = B$. By the Taylor series expansion of $F$~\eqref{eq:F},
\begin{equation}
\label{eq:TaylorF}
	F(u_{\S}) = (W^n_{\S})^T u_S + u_{\S}^T [\bigtriangledown^2 \ell_j ( \theta_{\M}^* + v u_{\S} ; x ) ] u_S + \lambda_n( \| \theta_{\S}^* + u_{\S} \|_1 - \| \theta_{\S}^* \|_1 ), 
\end{equation}
for some $v \in [0,1]$. Since $\| W^n_{\S} \|_{\infty} \leq \frac{\lambda_n}{4}$ by assumption and $\| u_S \|_1 \leq \sqrt{d} \| u_S \|_2 \leq \sqrt{d} \cdot B$, the first term in Equation~\eqref{eq:TaylorF} has the following bound:
\begin{equation*}
	|(W^n_{\S})^T u_{\S}| \leq \| W^n_{\S}\|_{\infty} \| u_{\S} \|_1 \leq \|W^n_{\S} \|_{\infty} \sqrt{\d} \|u_{\S}\|_2 \leq (\lambda_n \sqrt{\d} )^2 \frac{M_1}{4}.
\end{equation*}

Applying the triangle inequality to the last part of Equation~\eqref{eq:TaylorF}, we have the following bound.
\begin{equation*}
	\lambda_n ( \| \theta_{\S}^* + u_{\S} \|_1 - \| \theta_{\S}^*\|_1 ) \geq - \lambda_n \|u_{\S}\|_1 \geq -\lambda_n \sqrt{\d} \|u_{\S}\|_2 = -M_1 ( \lambda_n \sqrt{\d} )^2.
\end{equation*}

Next we bound $\lambda_{\min}\left(\bigtriangledown^2  \ell_j( \theta_{\S}^* + v u_{\S})\right)$ where $\lambda_{\min}(\cdot)$ is the minimum eigenvalue of a matrix:
\begin{align}
	\label{eq:q}
	q^* &: = \lambda_{\min}\left(\bigtriangledown^2  \ell_j( \theta_{\S}^* + v u_{\S}) \right) \nonumber \\
	& \geq \min_{ v \in [0,1] } \lambda_{\min}\left( \bigtriangledown^2 \ell_j ( \theta_{\S}^* + v u_{\S}) \right) \nonumber  \\
	& \geq \lambda_{\min} \left( \bigtriangledown^2 \ell_j ( \theta_{\S}^* ) \right) 
	- \max_{v \in [0, 1]} \| \frac{1}{n} \sum_{i=1}^n A_{j}'''( \langle \theta_{\S}^* + v u_{\S}, X_{\S} \rangle ) u_{\S}^T X_{\S}^{(i)}  X_{\S}^{(i)} (X_{\S}^{(i)})^T \|_2  \nonumber  \\	
	&\geq \rho_{\min} - \max_{v \in [0, 1]} \max_{y: \|y\|_2 = 1} \frac{1}{n} \sum_{i=1}^n | A_{j}'''( \langle \theta_{\S}^* + v u_{\S}, X_{\S} \rangle ) | \cdot | u_{\S}^T X_{\S}^{(i)} | \cdot ( y^T X_{\S}^{(i)} )^2.
\end{align}

Next we define the event $\xi_2$ in order to bound $A_{j}'''( \langle \theta_{\S}^* + v u_{\S}, X_{\S} \rangle )$.
\begin{equation*}
	\xi_2 :=\{ \max_{i \in \{ 1,\cdots,n\} } \langle \theta_{\S}^* + v u_{\S}, X_{\S}^{(i)} \rangle < \kappa_1 \log \eta \}.
\end{equation*}

On $\xi_2$, Assumption~\ref{A4} is satisfied and
\begin{equation}
\label{Bound1}
	A_{j}'''( \langle \theta_{\S}^* + v u_{\S}, X_{\S} \rangle ) \leq n^{\kappa_2}.
\end{equation} 
	
In addition, we bound the second term in~\eqref{eq:q}. Recall that $\|X_{\S}^{(i)}\|_{\infty} \leq 4 \log \eta$ for all $i \in \{1,2,\cdots,n\}$ on $\xi_1$. Since $\|u_{\S}\|_1 \leq \sqrt{\d} \|u_{\S}\|_2 \leq \sqrt{d}\cdot B$,
\begin{equation}
	\label{Bound2}
	| u_{\S}^T X_{\S}^{(i)} | \leq 4 \log(\eta) \sqrt{\d} \|u_{\S}\|_2 \leq 4 \log (\eta) \cdot M_1 \lambda_n \d.
\end{equation}

Lastly, it is clear that  $\max_{y: \|y\|_2 = 1} ( y^T X_{\S}^{(i)} )^2 \leq  \rho_{\max}$ by the definition of the maximum eigenvalue and Assumption~\ref{A1Dep}. Together with the bounds of~\eqref{Bound1} and~\eqref{Bound2} on the events $\xi_1$ and $\xi_2$,
\begin{equation*}
	q^* \leq \rho_{\min} - 4 n^{\kappa_2}  \log (\eta) \cdot M_1 \lambda_n \d ~\rho_{\max}.
\end{equation*}

For $\lambda_n \leq \frac{ \rho_{\min} }{ 8 n^{\kappa_2}  \log(\eta) M_1 \d \rho_{\max} }$, we have $q^* \leq \frac{ \rho_{\min} }{2} $. Therefore, 
\begin{equation*}
	F(u) \geq (\lambda_n \sqrt{n})^2 \Big\{ -\frac{1}{4} M_1 + \frac{ \rho_{\min}}{2} M_1^2 - M_1 \Big\},
\end{equation*}
which is strictly positive for $M_1 = \frac{5}{ \rho_{\min} }$. Therefore for $\lambda_n \leq \frac{ \rho_{\min}^2 }{ 40 n^{\kappa_2} \log(\eta) \d \rho_{\max} }$ given $\xi_1$ and $\xi_2$,
\begin{equation*}
	\| \widetilde{\theta}_S - \theta_S^* \|_2 \leq \frac{5}{ \rho_{\min} } \sqrt{\d} \lambda_n.
\end{equation*}

Since $P(A) = P(A \cap B \cap C) +P(A \cap (B \cap C)^c) \leq P(A \cap B \cap C) + P(B^c) + P(C^c)$,
\begin{equation*}
P\left( \| \widetilde{\theta}_S - \theta_S^* \|_2 > \frac{5}{ \rho_{\min} } \sqrt{\d} \lambda_n \right) \leq 
P\left( \| \widetilde{\theta}_S - \theta_S^* \|_2 > \frac{5}{ \rho_{\min} } \sqrt{\d} \lambda_n, \xi_1, \xi_2 \right) + P(\xi_1^c) + P(\xi_2^c).
\end{equation*}

Here the probability of $\xi_2^c$ is upped bounded as follows. 
\begin{eqnarray*}
	P(\xi_2^c) 
	& \stackrel{(a)}{\leq} & n \max_{i} P( \langle \theta_{M_j}^* + v u_{\S}, X_{\S}^{(i)} \rangle > \kappa_1 \log \eta )  \\
	& \stackrel{(b)}{\leq} & n \cdot M \cdot \eta^{-\frac{\kappa_1}{2 \|\theta_{M_j}^*\|_1 } } \\
	& \stackrel{(c)}{\leq} & M \cdot \eta^{-2}.
\end{eqnarray*}
(a) follows from the union bound, and (b) follows from Proposition~\ref{Prop1}, and $\|u_S\|_1 \leq \sqrt{d} \|u_S\|_2 \leq d M_1 \lambda_{n} \leq \| \theta_{M_j}^* \|_1$ and $\min_{j \in V} \min_{t \in S} |[\theta_{M}^*]_t| \geq \frac{10}{\rho_{\min}} \sqrt{\d} \lambda_n$. Lastly (c) follows from Assumption~\ref{A4} that $\kappa_1 \geq 6 \|\theta_{M_j}^*\|_1$.

In addition the probability bound of $\xi_1^c$ is provided in Proposition~\ref{Prop2}. Therefore 	
\begin{equation*}
	P\left( \| \widetilde{\theta}_S - \theta_S^* \|_2 \leq \frac{5}{ \lambda_{\min} } \sqrt{\d} ~\lambda_n \right) \geq 1 - 2 M \cdot \eta^{-2}.
\end{equation*}

\end{proof}

\subsubsection{Proof for Lemma~\ref{lem13}}

\begin{proof}
According to ~\eqref{eq:Rn}, $R_{jt}^n$ for any $t \in \S$ can be expressed as
\begin{eqnarray*}
	R_{jt}^{n} 
	& = & \frac{1}{n} \sum_{i=1}^{n} [ \bigtriangledown^2 \ell_j(\theta_{M_j}^*; X^{1:n}) - \bigtriangledown^2 \ell_j(\bar{\theta}_{M_j}^{(t)}; X^{1:n})]_t^T (\widetilde{\theta} - \theta_{M_j}^*) \\
	& = & \frac{1}{n} \sum_{i=1}^{n} [A_{j}''( \langle \theta_{S}^*, X_{V \setminus j}^{(i)} \rangle ) - A_{j}''( \langle \bar{\theta}_{M_j}^{(t)}, X_{V \setminus j}^{(i)} \rangle) ] [ X_{V \setminus j}^{(i)}(X_{V \setminus j}^{(i)})^T ]_t^T (\widetilde{\theta} - \theta_{M_j}^*)
\end{eqnarray*}
for $\bar{\theta}_{M_j}^{(t)}$ which is some point in the line between $\widetilde{\theta}_{M_j}$ and $\theta_{M_j}^*$ (i.e., $\bar{\theta}_{M_j}^{(t)} = v \cdot \widetilde{\theta}_{M_j} + (1 - v) \cdot \theta_{M_j}^*$ for some $v \in [0,1]$). 

By the mean value theorem,
\begin{equation*}
	R_{jt}^{t}  
	= \frac{1}{n} \sum_{i=1}^n \Big\{ A_{j}'''( \langle \bar{ \bar{ \theta} }_{M_j}^{(t)}, X_{V \setminus j}^{(i)} \rangle)  X_{t}^{(i)} \Big\} \Big\{ v ( \widetilde{\theta}_{M_j} - \theta_{M_j}^* )^T X_{V \setminus j}^{(i)}  (X_{V \setminus j}^{(i)})^T ( \widetilde{\theta}_{M_j} - \theta_{M_j}^* ) \Big\}
\end{equation*}
for $\bar{ \bar{ \theta} }_{M_j}^{(t)}$ which is a point on the line between $\bar{\theta}_{M_j}^{(t)}$ and $\theta_{M_j}^*$. 

By Proposition~\ref{Prop2}, $\max_{i,j}|X_j^{(i)}| \leq 4 \log \eta$ given $\xi_1$. Furthermore in Section~\ref{SubSecProofLem12}, we showed that $A_{j}'''( \langle \bar{ \bar{ \theta} }_{M_j}^{(t)}, X_{M \setminus j} \rangle ) \leq n^{\kappa_2}$ given $\xi_2$ . Therefore, on $\xi_1$ and $\xi_2$, it follows that:
\begin{equation*}
	| R^n_{jt} |	\leq 4 n^{\kappa_2} \log(\eta) \rho_{\max} \|\widetilde{\theta} - \theta_{\M}^* \|_2^2.
\end{equation*}

In the proof of Lemma~\ref{lem12}, we showed that $\|\widetilde{\theta} - \theta_{\M}^* \|_2 \leq \frac{5}{ \rho_{\min} } \sqrt{\d} \lambda_n$ for $\lambda_n \leq \frac{\alpha}{400(2-\alpha)}\frac{ \rho_{\min}^2 }{ \rho_{\max} } \frac{1}{d n^{\kappa_2} \log(\eta) }$ given $\xi_1$ and $\xi_2$. Therefore
\begin{equation*}
	\| R^n \|_\infty \leq \frac{100 \rho_{\max} }{\rho_{\min}^2 } d~ n^{\kappa_2} \log(\eta) ~ \lambda_n^2 \leq \frac{\alpha \lambda_n}{4(2 - \alpha)}.
\end{equation*}

Since $P(A) = P(A \cap B \cap C) +P(A \cap (B \cap C)^c) \leq P(A \cap B \cap C) + P(B^c) + P(C^c)$ ,
\begin{equation*}
P\left( \| R^n \|_\infty > \frac{\alpha \lambda_n}{4(2 - \alpha)} \right) \leq 
P\left( \| R^n \|_\infty > \frac{\alpha \lambda_n}{4(2 - \alpha)} , \xi_1, \xi_2 \right) + P(\xi_1^c) + P(\xi_2^c).
\end{equation*}
Putting the probability bounds for $\xi_1^c$ and $\xi_2^c$ specified in Proposition~\ref{Prop2} and Section~\ref{SubSecProofLem12} together, we have
\begin{equation*}
	P \left( \| R_{j}^n \|_\infty \leq \frac{\alpha \lambda_n}{4(2 - \alpha)} \right) \geq 1 - 2 M \cdot \eta^{-2}.
\end{equation*}
	
\end{proof}

\subsection{Proof for Theorem~\ref{ThmCausalOrdering}}

\label{SecThmCausalOrderingProof}

\begin{proof}

Without loss of generality, assume that the true causal ordering is $\pi^* = (1,2,\cdots,p)$. Let $T_j(X_j) := \omega_j X_j$ where $\omega_j = ( \beta_0 + \beta_1 \E(X_j \mid X_{\pa(j)} ) )^{-1}$ (specified in Proposition~\ref{prop:a}). For any node $j \in V$ and $S \subset V \setminus \{j\}$, let $\mu_{j\mid S}$ and $\sigma_{j \mid S}^2$ represent $\E( T_j(X_j) \mid X_S)$ and $\var(T_j(X_j) \mid X_S)$ respectively. For realizations $x_S$, let $\mu_{j \mid S}(x_S)$ and $\sigma_{j \mid S}^2(x_S)$ denote $\E( T_j(X_j) \mid X_S = x_S)$ and $\var(T_j(X_j) \mid X_S = x_S)$, respectively. Let $n(x_S) = \sum_{i=1}^n \mathbf{1}( X_S^{(i)} = x_S )$ denote the total conditional sample size, and $n_S = \sum_{ x_S } n(x_S) \mathbf{1}( n(x_S) \geq c_0 \cdot n )$ for an arbitrary $c_0 \in (0,1)$ to denote the truncated conditional sample size. 

Let $E^m$ denote the set of undirected edges corresponding to the \emph{moralized} graph. Recall the definitions $\mathcal{N}(j) = \{k \in V : (j,k) \text{ or } (k,j) \in E^m \}$ denote the neighborhood set of node $j$ in the moralized graph, $K(j) = \{ k : k \in \mathcal{N}(j-1) \cap (V \setminus \{\pi_1,...,\pi_{j-1}\})$, and $C_{jk} = \mathcal{N}(k) \cap \{\pi_1,\pi_2,\cdots,\pi_{j-1}\}$. Since we assume the structure of the moralized graph is provided, $\widehat{K}(j) = K(j)$ and $\widehat{C}_{jk} = C_{jk}$. Hence $K(j)$ and $C_{jk}$ are used instead of estimated sets $\widehat{K}(j)$ and $\widehat{C}_{jk}$. 

The overdispersion score of node $k \in K(j)$ for the $j^{th}$ component of the causal ordering $\pi_j$ only depends on $\mathcal{X}(C_{jk}) = \{x \in \{X_{C_{jk}}^{(1)}
, X_{C_{jk}}^{(2)},\cdots,X_{C_{jk}}^{(n)} \} : n(x) \geq c_0 \cdot n \}$, so we only count up elements that occur sufficiently frequently. 

According to the generalized ODS algorithm, the truncated sample conditional mean and variance of $T_j(X_j)$ given $X_S = x_S$ are:  
\begin{eqnarray*}
	\widehat{\mu}_{j\mid S}(x_S) & := & \frac{1}{ n_S(x_S) } \sum_{i =1}^n T_j(X_j^{(i)}) \mathbf{1}( X_S^{(i)} = x_S ), \\ 
	\widehat{\sigma}_{j\mid S}^2(x_S) & := & \frac{1}{ n_S(x_S) - 1} \sum_{i =1}^n ( T_j(X_j^{(i)} ) - \widehat{\mu}_{j\mid S}(x_S) )^2 \mathbf{1}(X_S^{(i)} = x_S).
\end{eqnarray*}

Then, we can rewrite the overdispersion score~\eqref{EqnTruncScorej} of node $k \in K(j)$ for $\pi_j$ as follows:
\begin{eqnarray*}
	\widehat{\mathcal{S}}(1,k) & := & 
	\left[ \left( \frac{ \widehat{ \sigma }_k }{ \beta_0 + \beta_1 \widehat{ \mu }_{k} } \right)^2  - \frac{ \widehat{ \mu }_{k} }{ \beta_0 + \beta_1 \widehat{ \mu }_{k} } \right], \\
	\widehat{\mathcal{S}}(j,k) & := & 
	\sum_{x \in \mathcal{X}(C_{jk}) } \frac{ n(x) }{n_{C_{jk}}} \left[ \left( \frac{ \widehat{ \sigma }_{j \mid C_{jk}}(x) }{ \beta_0 + \beta_1 \widehat{ \mu }_{j \mid C_{jk}}(x) } \right)^2  - \frac{ \widehat{ \mu }_{j \mid C_{jk}}(x) }{ \beta_0 + \beta_1 \widehat{ \mu }_{j \mid C_{jk}}(x) }  \right].
\end{eqnarray*}

For notational convenience, let each entry of the overdispersion score $\widehat{\mathcal{S}}(j,k)$ for $x \in \mathcal{X}(C_{jk})$ be defined as:
\begin{eqnarray}
	\label{eq:generalods}
	\widehat{\mathcal{S}}(j,k)(x) & := & 
	\left( \frac{ \widehat{ \sigma }_{j \mid C_{jk}}(x) }{ \beta_0 + \beta_1 \widehat{ \mu }_{j \mid C_{jk}}(x) } \right)^2  - \frac{ \widehat{ \mu }_{j \mid C_{jk}}(x) }{ \beta_0 + \beta_1 \widehat{ \mu }_{j \mid C_{jk}}(x) }.
\end{eqnarray}	

The true overdispersion scores are: 
\begin{eqnarray*}
	\mathcal{S}^*(1,k) & := & 
	\left[ \left( \frac{ \sigma_j }{ \beta_0 + \beta_1 \mu_{j} } \right)^2  - \frac{ \mu_{j} }{ \beta_0 + \beta_1 \mu_{j} } \right], \\
	\mathcal{S}^*(j,k) & := & 
	\sum_{x \in \mathcal{X}(C_{jk}) } \frac{ n(x) }{n_{C_{jk}}} \left[ \left( \frac{ \sigma_{j \mid C_{jk}}(x) }{ \beta_0 + \beta_1 \mu_{j \mid C_{jk}}(x) } \right)^2  - \frac{ \mu_{j \mid C_{jk}}(x) }{ \beta_0 + \beta_1 \mu_{j \mid C_{jk}}(x) }  \right], \\
	\mathcal{S}^*(j,k)(x) & := & 
	\left( \frac{ \sigma_{j \mid C_{jk}}(x) }{ \beta_0 + \beta_1 \mu_{j \mid C_{jk}}(x) } \right)^2  - \frac{ \mu_{j \mid C_{jk}}(x) }{ \beta_0 + \beta_1 \mu_{j \mid C_{jk}}(x) } \text{~~~~~for $x \in \mathcal{X}(C_{jk})$}.
\end{eqnarray*}

Next we introduce Proposition~\ref{A11} which ensures the each component of the true overdispersion score $\mathcal{S}^*(j,k)(x)$ for $k \neq \pi_j$ is bounded away from $m_{\min} >0$.
\begin{proposition}
	\label{A11}
	For all $j \in V$, $K_j \subset \pa(j)$, $K_j \neq \emptyset$ and $S \subset \nd(j) \setminus K_j$, there exists $m_{\min} > 0$ such that 
	$$\var( T_j(X_j) \mid X_S) - \E( T_j(X_j) \mid X_{S})  > m_{\min}.$$
\end{proposition}

Now we define the following two events: For any $j \in V$ and $k \in K(j)$, 
\begin{eqnarray*}
\xi_1 & := & \{ \max_{j}  \max_{i \in \{ 1,2,\cdots,n\} } |X_j^{(i)}| < 4 \log \eta \} \\
\xi_3 & := &\{ \max_{j,k} | \widehat{\mathcal{S}}(j,k) - \mathcal{S}(j,k)^* | < \frac{m_{\min}}{2} \}.
\end{eqnarray*}

Then, 
\begin{eqnarray}
	\label{eqn31}
	P( \widehat{\pi} \neq \pi^* ) 
	&\stackrel{(a)}{\leq} & P( \widehat{\pi} \neq \pi^*, \xi_3 ) + P(\xi_3^c, \xi_1) + P(\xi_1^c) \nonumber \\
	&\stackrel{(b)}{\leq} & P ( \widehat{\pi}_1 \neq \pi_1^*, \xi_3 ) + P ( \widehat{\pi}_2 \neq \pi_2^*, \xi_3 \mid \widehat{\pi}_1 = \pi_1^* ) + \nonumber \\ 
	&~~~ & \cdots + P( \widehat{\pi}_{p} \neq \pi_{p}^*, \xi_3 \mid \widehat{\pi}_1 = \pi_1^*,\cdots,\widehat{\pi}_{p-1} = \pi_{p-1}^*) + P(\xi_3^c, \xi_1) + P(\xi_1^c).
\end{eqnarray}
(a) follows from $P(A) \leq P(A \cap B) + P (B^c)$, and (b) follows from the induction and the fact $P(A \cup B) = P(A) +  P(B \cap A^c) = P(A) + P(B \mid A^c) P(A^c) \leq P(A) + P (B \mid A^c)$. 

We prove the probability bound~\eqref{eqn31} by induction. For the first step ($m = 1$), overdispersion scores of $\pi_1$ in~\eqref{EqnTruncScore1} are used where a set of candidate element of $\pi_1$ is $K(1) = \{1,2,\cdots,p\}$. Then,
\begin{eqnarray*}
	P( \widehat{\pi}_1 \neq \pi_1^*, \xi_3 ) 
	& = & P \left( \exists k \in K(1)\setminus\{\pi_1^*\} \textrm{ such that } \widehat{\mathcal{S}}(1,\pi_1^*)  > \widehat{\mathcal{S}}(1,k) , \xi_3 \right) \\
	&\stackrel{(a)}{\leq} & (p-1) \max_{k \in K(1) \setminus \{\pi_1^*\}} P \left( \mathcal{S}^*(1,\pi_1^*) + \frac{m_{\min}}{2} > \mathcal{S}^*(1,k) - \frac{m_{\min}}{2}, \xi_3 \right) \\	
	& \stackrel{(b)}{=} & (p-1) \max_{k \in K(1) \setminus \{\pi_1^*\}} P \left( m_{\min} > \mathcal{S}^*(1,k), \xi_3 \right) \\
	& \stackrel{(c)}{=} &  0.
\end{eqnarray*}
(a) follows from the union bound and the definition of $\xi_3$. (b) follows from that $\mathcal{S}^*(1,\pi_1^*) = 0$ by the property of the transformation $T_j(\cdot)$ specified in Proposition~\ref{prop:a}, and (c) follows from Proposition~\ref{A11}.
 
For the $m= (j-1)^{st}$ step, assume that the first $j-1$ elements of the estimated causal ordering are correct $(\widehat{\pi}_1, \widehat{\pi}_2,\cdots,\widehat{\pi}_{j-1}) = (\pi_1^*,\cdots, \pi_{j-1}^*)$. Then for the $m=j^{th}$ step, we consider the probability of a false recovery of $\pi_j^*$ given $(\pi_1^*,\cdots, \pi_{j-1}^*)$. Using the same argument as the first step, the following result is straightforward. 
\begin{eqnarray*}
	P ( \widehat{\pi}_j \neq \pi_j^*, \xi_3 \mid \pi_1^*,\cdots,\pi_{j-1}^* )
	& = & P \left( \exists k \in K(j) \setminus \{\pi_j^*\} \textrm{ such that } \widehat{\mathcal{S}}(j,\pi_j^*) > \widehat{\mathcal{S}}(j,k), \xi_3 \right) \\
	& \stackrel{(a)}{\leq} & |K(j)| \max_{k \in K(j) \setminus \{\pi_j^*\}} P \left( \mathcal{S}^*(j,\pi_j^*) + \frac{m_{\min}}{2} > \mathcal{S}^*(j,k) - \frac{m_{\min}}{2}, \xi_3 \right) \\ 
	& \stackrel{(b)}{=} & |K(j)| \max_{k \in K(j) \setminus \{\pi_j^*\}} P \left( m_{\min} > \mathcal{S}^*(j,k), \xi_3 \right) \\
	& \stackrel{(c)}{=} & 0.
\end{eqnarray*}

Therefore, for any $j \in V$, 
\begin{equation*}
	\label{eq:Bound2}
	P( \widehat{\pi}_j \neq \pi_j^*, \xi_3 \mid \widehat{\pi}_1 = \pi_1^*,\cdots,\widehat{\pi}_{j-1} = \pi_{j-1}^* ) = 0.
\end{equation*}
Then, the probability bound~\eqref{eqn31} is reduced to $P( \widehat{\pi} \neq \pi^* ) \leq P(\xi_3^c, \xi_1) + P(\xi_1^c)$. Note that $P(\xi_1^c) \leq M \cdot \eta^{-2}$ by Proposition~\ref{Prop2}. The following lemma provides the upper bound of $P(\xi_3^c, \xi_1)$. 
\begin{lemma}
	\label{lem1}
	There exist positive constants $C_1$ and $C_2$ such that 
	\begin{align*}
	P( \xi_3^c, \xi_1 ) \leq C_1 p^2 c_0^{-1} \exp\left( -C_2 \frac{ c_0 \cdot n }{ \log^4 \eta }   \right).
	\end{align*}		
	where $c_0$ is the sample cut-off parameter.
\end{lemma}

Lastly, we define a condition on the sample cut-off parameter $c_0$. Intuitively if $c_0$ is too small, the estimated overdispersion scores may be biased due to the lack of samples. In contrast, if $c_0$ is too large, all components of the conditioning set $C_{jk}$ may not have enough samples size ($> c_0 \cdot n$), and therefore overdispersion scores cannot be calculated. The following proposition provides a maximum value of $c_0$ ensuring that overdispersion scores exist.
\begin{proposition}
	\label{prop1}
	On the event $\xi_1$, if $c_0 \leq (3\log(\eta))^{-d}$ then the conditioning set $C_{jk}$ has at least $c_0 \cdot n$ samples. 
\end{proposition}

The combination of Lemma~\ref{lem1} and Proposition~\ref{prop1} imply that for some $C_1$ and $C_2$
\begin{equation*}
P( \xi_3^c, \xi_1 ) \leq C_1 p^2 \log^d (\eta) \exp\left( -C_2 \frac{ n }{ ( \log(\eta) )^{4+d} } \right).
\end{equation*}	

Therefore, 
\begin{equation*}
	P( \widehat{\pi} \neq \pi^* ) \leq  C_1 p^2 \log^d (\eta) \exp\left( -C_2 \frac{ n }{ \log^{4+d} \eta } \right) + \frac{M}{\eta^2}.
\end{equation*}	

\end{proof}

\subsubsection{Proof for Proposition~\ref{A11} } 

\begin{proof}
	In the proof of the identifiability theorem in Appendix~\ref{SecSubThmIde}, we obtain
	$$
	\var( T_j(X_j) \mid X_S) - \E( T_j(X_j) \mid X_{S}) = \frac{ (1 + \beta_1) \var( \E(X_j \mid X_{pa(j)}) \mid X_{S} ) }{ (\beta_0 + \beta_1 \E(X_j \mid X_{S}) )^2 }.
	$$
	
	By Assumption~\ref{A1}, $\var( \E(X_j \mid X_{pa(j)}) \mid X_{S} ) > M_{\min}$ and $| \beta_{j0} + \beta_{j1} \E(X_j \mid X_{S} ) | > \omega_{\min}$. Then,
	$$
	\var( T_j(X_j) \mid X_S) - \E( T_j(X_j) \mid X_{S}) \geq \frac{ (1 + \beta_1) M_{\min} }{ \omega_{\min}^2 }.
	$$
	 
Since $\beta_1 > -1$, the proof is complete.
\end{proof}

\subsubsection{Proof for Proposition \ref{prop1} }

\begin{proof}
	Let $|X_S|$ denote the cardinality of a set $\{ X_S^{(1)}, X_S^{(2)},\cdots,X_S^{(n)} \}$ and $|\mathcal{X}(S)|$ denote the cardinality of the truncated set $\mathcal{X}(S) := \{x \in \{X_{S}^{(1)}, X_{S}^{(2)},\cdots,X_{S}^{(n)} \} : n(x) \geq c_0 \cdot n \}$.
	
If $|\mathcal{X}(S)| = 1$, for all $x \in \{ X_S^{(1)}, X_S^{(2)},\cdots,X_S^{(n)} \}$, $n_S(x) = c_0 \cdot n -1$ except for a single $z \in \mathcal{X}(S)$ where $n_S(z) \geq c_0.n$. In this case, the total sample size $n = n_S(z) + (|X_S|-1)( c_0 \cdot n -1)$. Hence
	\begin{equation*}
	n_S(z) = n - (|X_S| - 1)(c_0. n -1) = n - c_0 \cdot n \cdot |X_S| + c_0 \cdot n + |X_S| - 1.
	\end{equation*}
	
	Since $c_0 \cdot n \leq n_S(z)$,
	\begin{equation*}
	c_0 \leq \frac{ n +  |X_S| -1 }{ n \cdot |X_S| }.
	\end{equation*}
	
	Note that $\frac{1}{|X_S|} \leq \frac{ n + |X_S| -1 }{ n \cdot |X_S| }$ and  $|X_j^{(i)}| \leq 4 \log(\eta)$ for all $j \in V$ and $i \in \{1,2,\cdots, n\}$ given $\xi_1$. Then the maximum cardinality of $X_S$ is $(4 \log(\eta))^{|S|}$. Hence if $c_0 \leq (4 \log(\eta))^{-|S|}$ there exists a $z \in \mathcal{X}(S)$. 
	
	Recall that the size of a candidate parents set $C_{jk}$ is bounded by the maximum degree of the moralized graph $d$. Therefore if $c_0 \leq 4 \log(\eta)^{-d}$, there exists at least one $z \in \mathcal{X}(C_{jk})$.
\end{proof}

\subsubsection{Proof for Lemma \ref{lem1} }

\label{lem1proof}

\begin{proof}
For ease of notation, let $n_{jk} = n_{C_{jk}}$ and $n_{jk}(x) = n_{C_{jk}}(x)$ for $x \in \mathcal{X}(C_{jk})$. Using the union bound, for $j \in V$ and $k \in K(j)$
\begin{equation*}
	P( \xi_3^c, \xi_1 ) = P( \max_{j,k} |\widehat{\mathcal{S}}(j,k) - \mathcal{S}^*(j,k)| > \frac{m_{\min}}{2}, \xi_1 ) 
\leq p^2 \max_{j,k} P( |\widehat{\mathcal{S}}(j,k) - \mathcal{S}^*(j,k)| > \frac{m_{\min}}{2}, \xi_1).
\end{equation*}

Since overdispersion scores have an additive form,
\begin{equation*}
	P( |\widehat{\mathcal{S}}(j,k) - \mathcal{S}^*(j,k)| > \frac{m_{\min}}{2}, \xi_1) \leq P( \sum_{x \in \mathcal{X}(C_{jk} ) } \frac{n_{jk}(x)}{n_{jk}} |\widehat{\mathcal{S}}(j,k)(x) - \mathcal{S}^*(j,k)(x)| > \frac{m_{\min}}{2}, \xi_1).
\end{equation*}

Applying $P( \sum_i Y_i > \delta ) \leq \sum_i P( Y_i >  \omega_i \delta)$ for any $\delta \in \mathbb{R}$ and $\omega_i \in \mathbb{R}^{+}$ such that $\sum_i \omega_i = 1$, we have 
\begin{align*}
	& P( \sum_{x \in \mathcal{X}(C_{jk} ) } \frac{n_{jk}(x)}{n_{jk}} |\widehat{\mathcal{S}}(j,k)(x) - \mathcal{S}^*(j,k)(x)| > \frac{m_{\min}}{2}, \xi_1 ) \\	 
	& \hspace{5cm} \leq \sum_{x \in \mathcal{X}(C_{jk}) } P( |\widehat{\mathcal{S}}(j,k)(x) - \mathcal{S}^*(j,k)(x)| > \frac{m_{\min}}{2}, \xi_1).
\end{align*}

Applying the union bound,  
\begin{align*}	
& \sum_{x \in \mathcal{X}(C_{jk}) } P( |\widehat{\mathcal{S}}(j,k)(x) - \mathcal{S}^*(j,k)(x)| > \frac{m_{\min}}{2}, \xi_1) \\
& \hspace{5cm} \leq |\mathcal{X}(C_{jk})| \max_{x \in \mathcal{X}(C_{jk}) } P(  |\widehat{\mathcal{S}}(j,k)(x) - \mathcal{S}(j,k)^*(x)| > \frac{m_{\min}}{2}, \xi_1).
\end{align*} 

Since we only consider $x \in \mathcal{X}( C_{jk} )$, it follows that $n_{jk}(x) \geq c_0 \cdot n $. Further since the total truncated sample size is less than total sample size, $c_0 \cdot n \cdot |\mathcal{X}(C_{jk})| \leq n$, and therefore the cardinality of $C_{jk}$ is at most $c_0^{-1}$. Hence
\begin{align*}
	& |\mathcal{X}(C_{jk})| \max_{x \in \mathcal{X}(C_{jk}) } P(  |\widehat{\mathcal{S}}(j,k)(x) - \mathcal{S}^*(j,k)(x)| > \frac{m_{\min}}{2}, \xi_1) \\
	& \hspace{5cm} \leq c_0^{-1} \max_{x \in \mathcal{X}(C_{jk}) } P(  |\widehat{\mathcal{S}}(j,k)(x) - \mathcal{S}^*(j,k)(x)| > \frac{m_{\min}}{2}, \xi_1).
\end{align*}

Since the overdispersion score is the difference between the conditional mean and conditional variance, the remainder of the proof is reduced to finding the sample complexity for the sample conditional mean and variance. Suppose that $\epsilon := \widehat{\mu}_{ k \mid C_{jk} }(x) - \mu_{ k \mid C_{jk} }(x)$ and $\kappa \cdot \epsilon := \widehat{\sigma}_{ k \mid C_{jk} }^2(x) - \sigma_{ k \mid C_{jk} }^2(x)$ for some $\kappa \in \mathbb{R}$. By the definition of the overdispersion scores in~\eqref{eq:generalods}, we have 

\begin{align*}	
	& \{ \epsilon : |\widehat{\mathcal{S}}(j,k)(x) - \mathcal{S}^*(j,k)(x)| > \frac{m_{\min}}{2} \} \\
	& \subset \left\{ \epsilon : \left| \left( \frac{ \sigma_{j \mid C_{jk}}(x) + \kappa \epsilon }{ \beta_0 + \beta_1 \mu_{j \mid C_{jk}}(x) + \epsilon } \right)^2  - \frac{ \mu_{j \mid C_{jk}}(x) + \epsilon }{ \beta_0 + \beta_1 \mu_{j \mid C_{jk}}(x) + \epsilon } \right. \right. \\
	& \hspace{7cm} - \left. \left. \left( \frac{ \sigma_{j\mid C_{jk}}(x) }{ \beta_0 + \beta_1 \mu_{j\mid C_{jk}}(x) } \right)^2  - \frac{ \mu_{j\mid C_{jk}}(x) }{ \beta_0 + \beta_1 \mu_{j\mid C_{jk}}(x) } \right| > \frac{m_{\min}}{2} \right\} \\
	& = \left\{ \epsilon: \epsilon \in (\epsilon_1, \epsilon_2) \cup (\epsilon_3, \epsilon_4) \right\}.
\end{align*}
where $\epsilon_1, \epsilon_2, \epsilon_3, \epsilon_4$ are constants that depend on $\mu, \sigma^2, \beta_0, \beta_1, m_{\min}$, and $\kappa$ and are constructed as follows: 

\begin{align*}
	& \zeta_1(\mu,\sigma^2,\beta_0, \beta_1, m_{\min}, \kappa) := \beta_0^3 ( 1 + \beta_1 m_{\min} ) - \beta_1^4 m_{\min} \mu^3 + 2 \beta_1^2 \mu^2 \kappa \sigma^2 - 2 \beta_1^2 \mu \sigma^4  \\
	& \hspace{1.5cm}	+ \beta_0^2 ( -2 \beta_1 \mu - 3 \beta_1^2 m_{\min} \mu + 2 \kappa \sigma^2 )
	 - \beta_0 \beta_1 \big\{ \beta_1 \mu^2 + 3 \beta_1^2 m_{\min} \mu^2 + 2 \sigma^2 ( -2 \kappa \mu + \sigma^2 ) \big\}, \\
	& \zeta_2(\mu,\sigma^2,\beta_0, \beta_1, m_{\min}, \kappa) := (\beta_0+ \beta_1 \mu)^2 \Big[ \beta_0^4 (1 + 2 \kappa \mu) + 2 \beta_1^2( \kappa \mu - \sigma^2 )^2 (\beta_1^2 \mu^2 m_{\min} + 2 \sigma^4 ) \\
	& \hspace{1.5cm} + 4 \beta_0 \beta_1 (\kappa \mu - \sigma^2) \big\{ \beta_1^2 \mu m_{\min} (2 \kappa \mu - \sigma^2 ) + \beta_1 \mu \sigma^2 - 2 \kappa \sigma^2 \} \\
	& \hspace{1.5cm} + 2 \beta_0^3 \big\{ -2 \kappa \sigma^2 + \beta_1 (\mu + 4 m_{\min} \kappa^2 \mu - 2 m_{\min} \kappa \sigma^2 ) \big\} \\
	& \hspace{1.5cm}  + \beta_0^2 \big\{ 4 \kappa^2 \sigma^4 + 4 \beta_1 \sigma^2 ( -2 \kappa \mu + \sigma^2 ) + \beta_1^2 ( \mu^2 + 12 m_{\min} \kappa^2 \mu^2 - 12 m_{\min} \mu \kappa \sigma^2 + 2 m_{\min} \sigma^4 ) \big\} \Big],
	\\	
	& \zeta_3(\mu,\sigma^2,\beta_0, \beta_1, m_{\min}, \kappa) := \beta_0^2 ( -2 \kappa^2 +2 \beta_1 + \beta_1^2 m_{\min} ) + 2 \beta_0 \beta \mu ( \beta_1 + \beta_1^2 m_{\min} - \kappa^2 ) \\
	& \hspace{2cm} + \beta_1^2 ( \beta_1^2 m_{\min} \mu^2 + 2 \sigma^4 - 2 \kappa^2 \mu^2 ).
\end{align*}

Given $\zeta_1, \zeta_2, \zeta_3$,
\begin{align*}
	\epsilon_1' & = \frac{ \zeta_1(\mu_{j\mid C_{jk}}(x), \sigma_{j\mid C_{jk}}^2(x), \beta_0, \beta_1, m_{\min}, \kappa ) + \sqrt{ \zeta_2(\mu_{j\mid C_{jk}}(x), \sigma_{j\mid C_{jk}}^2(x), \beta_0, \beta_1, m_{\min}, \kappa ) } }{ \zeta_3(\mu_{j\mid C_{jk}}(x), \sigma_{j\mid C_{jk}}^2(x), \beta_0, \beta_1, m_{\min}, \kappa ) },\\
	\epsilon_2' & = \frac{ - \zeta_1(\mu_{j\mid C_{jk}}(x), \sigma_{j\mid C_{jk}}^2(x), \beta_0, \beta_1, m_{\min}, \kappa ) + \sqrt{ \zeta_2(\mu_{j\mid C_{jk}}(x), \sigma_{j\mid C_{jk}}^2(x), \beta_0, \beta_1, m_{\min}, \kappa ) } }{ \zeta_3(\mu_{j\mid C_{jk}}(x), \sigma_{j\mid C_{jk}}^2(x), \beta_0, \beta_1, m_{\min}, \kappa ) },\\	
	\epsilon_3' & = \frac{ \zeta_1(\mu_{j\mid C_{jk}}(x), \sigma_{j\mid C_{jk}}^2(x), \beta_0, \beta_1, -m_{\min}, \kappa ) + \sqrt{ \zeta_2(\mu_{j\mid C_{jk}}(x), \sigma_{j\mid C_{jk}}^2(x), \beta_0, \beta_1, -m_{\min}, \kappa ) } }{ \zeta_3(\mu_{j\mid C_{jk}}(x), \sigma_{j\mid C_{jk}}^2(x), \beta_0, \beta_1, -m_{\min}, \kappa ) }, \\
	\epsilon_4' & = \frac{- \zeta_1(\mu_{j\mid C_{jk}}(x), \sigma_{j\mid C_{jk}}^2(x), \beta_0, \beta_1, -m_{\min}, \kappa ) + \sqrt{ \zeta_2(\mu_{j\mid C_{jk}}(x), \sigma_{j\mid C_{jk}}^2(x), \beta_0, \beta_1, -m_{\min}, \kappa ) } }{ \zeta_3(\mu_{j\mid C_{jk}}(x), \sigma_{j\mid C_{jk}}^2(x), \beta_0, \beta_1, -m_{\min}, \kappa )}.	
\end{align*}

Let $(\epsilon_1, \epsilon_2, \epsilon_3, \epsilon_4)$ be the ordered values of $(\epsilon_1',\epsilon_2',\epsilon_3',\epsilon_4')$ from smallest to largest. Since $m_{\min} > 0$ it follows that $\epsilon_1, \epsilon_2 < 0$ and $\epsilon_3, \epsilon_4 > 0$. 

For ease of notation, $\epsilon_{\min} = \min\{ |\epsilon_2|, |\epsilon_3| \}$. Then, 
\begin{equation*}
	\{ \epsilon : |\widehat{\mathcal{S}}(j,k)(x) - \mathcal{S}^*(j,k)(x)| > \frac{m_{\min}}{2} \}
	\subset (-\infty, -\epsilon_{\min}) \cup (\epsilon_{\min}, \infty).
\end{equation*}

Hence 
\begin{eqnarray*}
& &P\{|\widehat{\mathcal{S}}(j,k)(x) - \mathcal{S}^*(j,k)(x)| > \frac{m_{\min}}{2} \} \\
&  & \hspace{2.5cm} \leq  P\left( |\widehat{\mu}_{ k \mid C_{jk} }(x) - \mu_{ k \mid C_{jk} }(x)| > \epsilon_{\min} \right) +  P\left( |\widehat{\sigma}_{ k \mid C_{jk} }^2(x) - \sigma_{ k \mid C_{jk} }^2(x)| > \kappa \epsilon_{\min} \right).
\end{eqnarray*}

On $\xi_1$, $\max_{i,j} |X_j^{(i)}| \leq 4 \log(\eta)$. Furthermore recall that $n_{jk}(x) \geq c_0 \cdot n$. By applying Hoeffding's inequality,
\begin{equation*}
	P( | \widehat{\mu}_{j \mid C_{jk}}(x) - \mu_{j \mid C_{jk}}(x) | > \epsilon_{\min}, \xi_1 )
	\leq 2 \exp\left( - \frac{ \epsilon_{\min}^2 c_0. n }{ 8 \log^2 \eta } \right).
\end{equation*}
	
Note that sample variance can be decomposed as follows:
\begin{equation*}
	\frac{1}{n-1} \left( \sum_i^n X_i^2 - \frac{1}{n} (\sum_i^n X_i)^2 \right) = \frac{1}{n} \sum_i^n X_i^2 - \frac{1}{n(n-1)} \sum_{i \neq j} X_i X_j.
\end{equation*}

Using Hoeffding's inequality for the decomposed sample variance,
\begin{equation*}
	P( | \widehat{\sigma}_{j \mid C_{jk}}^2(x) - \sigma_{j \mid C_{jk}}^2(x) | > |\kappa| \cdot \epsilon_{\min}, \xi_1 )
	\leq 2 \exp\left( - \frac{ \kappa^2 \epsilon_{\min}^2 c_0 \cdot n }{ 128 \log^4 \eta } \right)
	+ 2 \exp\left( - \frac{ \kappa^2 \epsilon_{\min}^2 c_0 \cdot n }{ 256 \log^4 \eta } \right).
\end{equation*}	

Therefore, 
\begin{eqnarray*}
& & P\{|\widehat{\mathcal{S}}(j,k)(x) - \mathcal{S}^*(j,k)(x)| > \frac{m_{\min}}{2}, \xi_1 \} \\
& & \hspace{2cm} \leq 2\left( \exp\left( - \frac{ \epsilon_{\min}^2 c_0. n }{ 8 \log^2 \eta } \right) + \exp\left( - \frac{ \kappa^2 \epsilon_{\min}^2 c_0 \cdot n }{ 128 \log^4 \eta } \right)+ \exp\left( - \frac{ \kappa^2 \epsilon_{\min}^2 c_0 \cdot n }{ 256 \log^4 \eta } \right) \right).
\end{eqnarray*}

This completes the proof since there exist constants $C_1$ and $C_2$ such that
\begin{equation*}
P( \xi_3^c, \xi_1 ) \leq C_1 p^2 c_0^{-1} \exp\left( -C_2 \frac{ c_0 \cdot n }{ \log^4 \eta } \right).
\end{equation*}	

\end{proof}

\subsection{Proof for Theorem~\ref{ThmDirectGraph}}

\label{SecThmStep3Proof}

\begin{proof}
Once again we use the \emph{primal-dual witness} method used in the the proof for Theorem~\ref{ThmMoralGraph}. The only difference is the conditioning set. In this proof, the conditioning  set is all elements of the ordering before node $j$ rather than $j$ is $ V \setminus \{j\}$. Without loss of generality, we assume the true causal ordering is $\pi^* = (1,2,\cdots,p)$. Then the conditioning set is $\{1,2,\cdots, j-1\}$.

For ease of notation, we define the  parameter $\theta \in \mathbb{R}^{j-1}$ since the node $j$ is not penalized in~\eqref{P1D}. Then, the conditional negative log-likelihood of a GLM~\eqref{Eq1D} for $X_j$ given $X_{1:j-1}$ is:
\begin{equation*}
	\ell_{j}^D( \theta; X^{1:n}) = \frac{1}{n} \sum_{i = 1}^{n} \left( -X_j^{(i)} \langle \theta, X_{1:j-1}^{(i)} \rangle + A_{j}( \langle \theta, X_{1:j-1}^{(i)} \rangle )  \right).
\end{equation*}

Recall that for any node $j \in V$:
\begin{equation*}
\widehat{\theta}_{D_j} := \arg \min_{\theta \in \mathbb{R}^{j-1}  } \mathcal{L}_{j}^{D}( \theta, \lambda_{n}^{D}) = \arg \min_{\theta \in \mathbb{R}^{j-1}  } \{ \ell_{j}^D (\theta ; X^{1:n}) + \lambda_{n}^{D} \| \theta \|_1 \}.
\end{equation*}

Using the \emph{sub-differential}, $\widehat{\theta}_{D_j}$ should satisfy the following condition. For notational simplicity, let $\T = \pa(j)$ for node $j \in V$. 
\begin{equation}
\label{eq:Contraint1D}
\bigtriangledown_\theta \mathcal{L}_{j}^{D}( \widehat{\theta}_{D_j}, \lambda_{n}^{D} ) = \bigtriangledown_\theta \ell_{j}^D( \widehat{\theta}_{D_j}; X^{1:n}) + \lambda_{n}^{D} \widehat{Z}  = 0
\end{equation}
where $\widehat{Z} \in \mathbb{R}^{j-1}$ and $\widehat{Z}_t = \mbox{sign}([\widehat{\theta}_{D_j}]_{t})$ if a node $t \in \T$, otherwise $|\widehat{Z}_{t}| < 1$. 

By Lemma~\ref{lemma: uniq}, it is sufficient the show that $|\widehat{Z}_{t}| < 1$ for all $t \in \T$. We note that the restricted solution is $(\widetilde{\theta}_{D_j}, \widetilde{Z})$. Equation~\eqref{eq:Contraint1D} with the dual solution $(\widetilde{\theta}_{D_j}, \widetilde{Z})$ can be represented as $\bigtriangledown^2 \ell_{j}^D( \theta_{D_j}^*; X^{1:n})( \widetilde{\theta}_{D_j} - \theta_{D_j}^* ) = -\lambda_{n}^{D} \widetilde{Z} - W_{Dj}^{n} + R_{Dj}^{n}$ by using the mean value theorem where: 
\begin{itemize}
	\item[(a)] $W_{Dj}^{n}$ is the sample score function, 
	\begin{equation}
	\label{eq:WnD}
	W_{Dj}^n := - \bigtriangledown \ell_{j}^D(\theta_{D_j}^*; X^{1:n}).
	\end{equation}
	\item[(b)] $R_{Dj}^{n} = (R_{Dj1}^n, R_{Dj2}^n,\cdots, R_{Dj j-1}^n)$ and $R_{Djk}^n$ is the remainder term by applying coordinate-wise mean value theorem,
	\begin{equation}
	\label{eq:RnD}
	R_{Djk}^n := [ \bigtriangledown^2 \ell_{j}^D(\theta_{D_j}^*; X^{1:n}) - \bigtriangledown^2 \ell_{j}^D(\bar{\theta}_{D_j}^{(k)}; X^{1:n})]_k^T (\widetilde{\theta}_{D_j}^{(k)} - \theta_{D_j}^*)
	\end{equation} 
	where $\bar{\theta}_{D_j}^{(j)}$ is a vector on the line between $\widetilde{\theta}_{D_j}$ and $\theta_{D_j}^*$ and $[\cdot]_k^T$ is the $k^{th}$ row of a matrix. 
\end{itemize}

Similar to Proposition~\ref{prop: block}, the following corollary provides a sufficient condition to control $\widetilde{Z}$.

\begin{corollary}
	\label{coro: block}
	Suppose that $\max(\| W_{Dj}^n \|_\infty, \| R_{Dj}^n \|_\infty) \leq \frac{\lambda_n \alpha}{4(2- \alpha)}$. Then, $| \widetilde{Z}_{t} | < 1$ for all $t \notin \pa(j)$.
\end{corollary}

Now we introduce the following three corollaries, to verify that the conditions in Proposition~\ref{coro: block} hold, and the deviation $\widetilde{\theta}_{M_j} - \theta_{D_j}^*$ is sufficiently small to conclude $\widehat{\pa}(j)= \mathcal{\pa}(j)$ with high probability. For ease of notation, let $\eta = \max\{n,p\}$ and 
For notational convenience, we use $\widetilde{\theta}_{\S} = [\widetilde{\theta}_{D_j}]_{\S}$ and $\widetilde{\theta}_{\S^c} = [\widetilde{\theta}_{D_j}]_{\S^c}$. Suppose that Assumptions~\ref{A1Dep},~\ref{A2Inc},~\ref{A3Con}, and~\ref{A4} are satisfied. 

\begin{corollary}
	\label{cor11}
	Suppose that $\lambda_{n}^{D} \geq \frac{16 \max\{ n^{\kappa_2} \log \eta , \log^2 \eta \}}{n^{a}}$ for some $a \in \mathbb{R}$. Then,  
	\begin{equation*}
	P( \frac{\| W_{Dj}^n \|_\infty }{\lambda_{n}^{D}} \leq \frac{\alpha}{4(2- \alpha)}) 
	\geq 1 -2 \d \cdot \exp(-\frac{\alpha^2}{8 (2- \alpha)^2} \cdot n^{ 1 - 2a }) - M \cdot \eta^{-2}.
	\end{equation*}
\end{corollary}

\begin{corollary} 
	\label{cor12}
	Suppose that $\|W_{Dj}^n\|_{\infty} \leq \frac{\lambda_{n}^{D}}{4}$. For $\lambda_{n}^{D} \leq \frac{ \rho_{\min}^2 }{ 40 \rho_{\max} } \frac{1}{n^{\kappa_2} \log \eta \d }$,
	\begin{equation*}
	P \left( \| \widetilde{\theta}_{\T} - \theta_S^* \|_2 \leq \frac{5}{ \lambda_{\min} } \sqrt{\d} \lambda_{n}^{D} \right) \geq 1 - 2 M \cdot \eta^{-2}.
	\end{equation*}
\end{corollary}

\begin{corollary} 
	\label{cor13}
	Suppose that $\|W_{Dj}^{n}\|_{\infty} \leq \frac{\lambda_{n}^{D}}{4}$. For $\lambda_{n}^{D} \leq \frac{\alpha}{400(2-\alpha)}\frac{ \rho_{\min}^2 }{ \rho_{\max} } \frac{1}{n^{\kappa_2} \d \log \eta }$, 
	\begin{equation*}
	P \left( \| R_{Dj}^n \|_\infty \leq \frac{\alpha \lambda_{n}^{D}}{4(2 - \alpha)} \right) \geq 1 - 2 M \cdot \eta^{-2}.
	\end{equation*}
\end{corollary}

Consider the choice of regularization parameter $\lambda_n^{D} = \frac{16 \max\{ n^{\kappa_2} \log \eta, \log^2 \eta \}}{n^{a}}$ where $a \in (2 \kappa_2, 1/2)$. Then, the condition for Corollary~\ref{cor11} is satisfied, and therefore $\|W_{Dj}^{n}\|_{\infty} \leq \frac{\lambda_n^D}{4}$. Moreover, the conditions for Corollaries~\ref{cor12} and~\ref{cor13} are satisfied for a sufficiently large sample size $n \geq D' \max\{ ( \d \log^2 \eta )^{ \frac{1}{  a - 2 \kappa_2 } }, ( \d \log^3 \eta )^{ \frac{1}{  a - \kappa_2 } } \}$ for a positive constant $D'$. Therefore, there exist some positive constants $D_1, D_2$ and $D_3$ such that
\begin{equation}
\| \widetilde{Z}_{\S^c}\|_\infty \leq ( 1- \alpha ) + ( 2- \alpha) \left[ \frac{ \| W_{Dj}^{n} \|_\infty }{ \lambda_{n}^{D}} + \frac{ \| R_{Dj}^n \|_\infty }{ \lambda_{n}^{D}} \right] \leq ( 1- \alpha ) + \frac{ \alpha }{4} + \frac{ \alpha }{4} < 1,
\end{equation}
with probability of at least $1 - D_1 \d \exp( - D_2 n^{ 1 - 2a } )- D_3 \eta^{-2}$.  

For sign consistency, it is sufficient to show that $\|\widehat{\theta}_{D_j}- \theta_{D_j}^* \|_{\infty} \leq \frac{ \|\theta_{D_j}^*\|_{\min} }{2}$. By Corollary~\ref{cor12}, we have $\|\widehat{\theta}_{D_j} - \theta_{D_j}^* \|_{\infty} \leq \|\widehat{\theta}_{D_j} - \theta_{D_j}^* \|_{2} \leq \frac{5}{\lambda_{\min}} \sqrt{\d}~\lambda_{n}^{D} \leq \frac{\|\theta_{D_j}^*\|_{\min} }{2}$ as long as $\|\theta_{D_j}^*\|_{\min} \geq \frac{10}{\lambda_{\min}} \sqrt{\d}~\lambda_{n}^{D}$. 

Lastly, Lemma~\ref{Lem:RestFaithfulness}(a) guarantees that $\ell_1$-penalized likelihood regression recovers the parent set for each node with high probability. Because we have $p$ regression problems if $n \geq D' \max\{ ( \d \log^2 \eta )^{ \frac{1}{  a - 2 \kappa_2 } }, ( \d \log^3 \eta )^{ \frac{1}{  a - \kappa_2 } } \}$, the full DAG model is recovered with high probability:
\begin{equation*}
P( \widehat{G} = G ) \geq 1 - D_1 \d \cdot p \cdot \exp( - D_2 n^{1- 2 a} ) - D_3 \eta^{-1}. 
\end{equation*}

\end{proof}

\end{document}